\documentclass{article}
\usepackage{arxiv,times}

\usepackage[table]{xcolor}

\usepackage{relsize}

\usepackage{mweights} 

\usepackage{array}

\usepackage[protrusion=false]{microtype} 

\usepackage{titletoc}
\newcommand\DoToC{%
  \startcontents
  \printcontents{}{1}{\vskip 1.5em\hrule\vskip .75em}
  \vskip .75em\hrule\vskip 2em
}

\makeatletter
\newcommand\myparagraph{\@startsection{paragraph}{4}{\z@}%
  {0ex plus 0ex minus 0ex}
  {-1em}
  {\normalsize\bfseries}}         
\makeatother

\usepackage{tikz}
\usetikzlibrary{positioning, shapes, decorations.pathreplacing,arrows.meta}

\definecolor{Bleu}{RGB}{0,0,175} 
\definecolor{DZO}{rgb}{1,.5,0} 
\definecolor{DZG}{rgb}{0.151,0.620,0.151} 
\definecolor{ZB}{rgb}{.255,.42,.882} 

\definecolor{myDarkGrey}{rgb}{.1,.1,.1}
\definecolor{myLessDarkGrey}{rgb}{.125,.125,.125}
\definecolor{slateGrey}{HTML}{708090}

\definecolor{CBmagenta}{HTML}{EE3377}
\definecolor{CBred}{HTML}{CC3311}
\definecolor{CBorange}{HTML}{EE7733}
\definecolor{CBteal}{HTML}{009988}
\definecolor{CBcyan}{HTML}{33BBEE}
\definecolor{CBblue}{HTML}{0077BB}
\definecolor{CBgrey}{HTML}{BBBBBB}

\definecolor{CB2blue}{HTML}{6699CC}
\definecolor{CB2yellow}{HTML}{EECC66}
\definecolor{CB2red}{HTML}{EE99AA}

\definecolor{CB3yellow}{HTML}{CCBB44}
\definecolor{CB3green}{HTML}{228833}
\definecolor{CB3purple}{HTML}{EE6677}
\definecolor{CB3red}{HTML}{AA3377}

\definecolor{CB4red}{HTML}{882255}
\definecolor{CB4green}{HTML}{117733}
\definecolor{CB4purple}{HTML}{AA4499}

\definecolor{CB5red}{HTML}{DA0000} 
\definecolor{CB5blue}{HTML}{0066FF} 
\definecolor{CB5purple}{HTML}{8000B3}  

\usepackage{listings}

\definecolor{codegreen}{rgb}{0,0.6,0}
\definecolor{codegray}{rgb}{0.5,0.5,0.5}
\definecolor{codepurple}{rgb}{0.58,0,0.82}
\definecolor{backcolour}{rgb}{0.97,0.97,0.95}

\usepackage{upquote} 

\lstdefinestyle{mystyle}{
    language=Python,
    escapeinside={(*}{*)},
    escapebegin=\color{codegreen},    backgroundcolor=\color{backcolour},   
    commentstyle=\color{codegreen},
    keywordstyle=\color{magenta},
    numberstyle=\tiny\color{codegray},
    stringstyle=\color{codepurple},
    basicstyle=\ttfamily\footnotesize,
    breakatwhitespace=false,         
    breaklines=true,                 
    captionpos=b,                    
    keepspaces=true,                 
    numbers=left,                    
    showspaces=false,                
    showstringspaces=false,
    showtabs=false,                  
    tabsize=2,
    moredelim=**[is][\color{slateGrey}]{@}{@},
    morestring=[b]' 
}

\usepackage{amsmath}
\usepackage{amsfonts}
\usepackage{amssymb}
\usepackage{amsthm}
\makeatletter 
\def\thm@space@setup{%
  \thm@preskip=\dimexpr\parskip+\smallskipamount\relax
  \thm@postskip=0pt
}
\makeatother

\usepackage[hypertexnames=false]{hyperref}

\usepackage[capitalize]{cleveref}

\crefname{section}{Sec.}{Secs.}

\crefname{theorem}{Thm.}{Thms.}
\crefname{lemma}{Lem.}{Lems.}
\crefname{corollary}{Cor.}{Cors.}
\crefname{proposition}{Prop.}{Props.}
\crefname{assumption}{Asm.}{Asms.}
\crefname{property}{Propt.}{Propts.}
\crefname{algorithm}{Alg.}{Algs.}
\crefname{appendix}{Appx.}{Appendices}

\crefname{figure}{Fig.}{Figs.}
\crefname{table}{Tab.}{Tabs.}

\usepackage{etoolbox}
\newcommand{\tablefont}{\fontsize{9.5}{11.4}\selectfont} 
\AtBeginEnvironment{tabular}{\tablefont} 

\DeclareFontFamily{U}{mathb}{}
\DeclareFontShape{U}{mathb}{m}{n}{
  <-5.5> mathb5
  <5.5-6.5> mathb6
  <6.5-7.5> mathb7
  <7.5-8.5> mathb8
  <8.5-9.5> mathb9
  <9.5-11.5> mathb10
  <11.5-> mathb12
}{}
\DeclareSymbolFont{mathb}{U}{mathb}{m}{n}
\DeclareMathSymbol{\olddrsh}{3}{mathb}{"EB} 

\usepackage{bm}
\usepackage{graphicx}
\usepackage{enumitem}
\usepackage{accents} 
\graphicspath{ {./figures/} }

\usepackage[title]{appendix}
\usepackage{booktabs}
\usepackage{multirow}

\usepackage{pifont} 

\newcommand*{\TakeFourierOrnament}[1]{{%
\fontencoding{U}\fontfamily{futs}\selectfont\char#1}}
\newcommand*{\danger}{\raisebox{.18em}{\small\TakeFourierOrnament{`1}}}

\usepackage{algorithm}
\usepackage[noend]{algpseudocode}

\algnewcommand{\SilentRequire}{\item[\phantom{\textbf{Require:}}]}

\makeatletter
\floatname{algorithm}{Algorithm}
\renewcommand{\thealgorithm}{\arabic{algorithm}}
\renewcommand{\fnum@algorithm}{\algorithmname\space\thealgorithm 
~(\mbox{\normalfont\@currentlabelname}):}

\newcommand{\algcaption}[2][]{%
  \caption[#1]{#2}
}
\makeatother

\usepackage{multirow}

\usepackage[mathscr]{euscript}
\usepackage{bbm}

\usepackage{thmtools} 
\renewcommand\thmcontinues[1]{continued}

\newtheorem{theorem}{Theorem}
\newtheorem{lemma}{Lemma}

\newtheorem{corollary}{Corollary}
\newtheorem{proposition}{Proposition}

\declaretheoremstyle[
  spaceabove=\dimexpr\parskip+\smallskipamount\relax,  
]{examplestyle}
\declaretheorem[style=examplestyle]{example}

\usepackage{chngcntr}
\newcounter{parentnumber}


\crefname{example}{Example}{Examples} 
\usepackage{mathtools}

\DeclareMathOperator{\KLop}{D_{\mathrm{KL}}}
\DeclarePairedDelimiterX{\KLinnerx}[2]{(}{)}{%
  #1\;\delimsize\|\;#2%
}
\newcommand{\KL}{\KLop\KLinnerx}

\DeclarePairedDelimiterX{\KLinnerxsqrt}[2]{(}{)}{%
  #1\;\delimsize\|\;#2%
}

\DeclareMathOperator{\chisqop}{\chi^2}
\newcommand{\chisq}{\chisqop\KLinnerx}

\DeclareMathOperator*{\logit}{logit}
\DeclareMathOperator*{\expit}{expit}
\DeclareMathOperator*{\argmin}{argmin}

\DeclareMathOperator*{\esssup}{ess\,sup}

\DeclareMathOperator*{\ReLU}{ReLU}
\DeclareMathOperator{\TV}{TV}
\DeclareMathOperator{\Law}{Law}

\newcommand{\mymid}{\,|\,}

\DeclareFontFamily{U}{stix2bb}{\skewchar\font127 }
\DeclareFontShape{U}{stix2bb}{m}{n} {<-> stix2-mathbb}{}
\DeclareMathAlphabet{\mathbbp}{U}{stix2bb}{m}{n}

\newcommand{\Cdivergence}{C_{\ref*{cond:divergence}}}
\newcommand{\Cpos}{C_{\ref*{cond:strongPositivity}}}

\newcommand{\Closs}{C_{\ref*{cond:bddLoss}}}

\newcommand{\Ccurv}{C_{\ref*{cond:curvature}}}
\newcommand{\Cmixedstrong}{C_{\ref*{cond:mixedLipschitzStrong}}}
\newcommand{\Cmixedweak}{C_{\ref*{cond:mixedLipschitzWeak}}}

\newcommand{\Centstrong}{C_{\ref*{cond:entropyStronger}}}
\newcommand{\CdiffDens}{C_{\ref*{cond:diffusionBddDensity}}}
\newcommand{\Cbesov}{C_{\ref*{cond:diffusionBesov}}}

\newcommand{\CdiffusionBoundary}{C_{\ref*{cond:diffusionBoundary}}}

\newlist{condenum}{enumerate}{1}
\crefformat{condenumi}{#2C#1#3}
\crefrangeformat{condenumi}{#3C#1#4--#5C#2#6}
\crefrangemultiformat{condenumi}{#3C#1#4--#5C#2#6}{ and~#3C#1#4--#5C#2#6}{, #3C#1#4--#5C#2#6}{, and #3C#1#4--#5C#2#6}
\crefmultiformat{condenumi}{#2C#1#3}{ and~#2C#1#3}{, #2C#1#3}{, and~#2C#1#3}

\newcommand{\indep}{\perp \!\!\! \perp}

\usepackage[width=\textwidth]{caption}

\usepackage{nccmath}

\usepackage{hyperref}
\usepackage{url}

\title{DoubleGen: Debiased Generative Modeling of Counterfactuals}

\author{Alex Luedtke \\
Harvard University\\
Boston, USA \\
\texttt{luedtke@hms.harvard.edu} \\
\And
Kenji Fukumizu \\
Institute of Statistical Mathematics \\
Tokyo, Japan \\
\texttt{fukumizu@ism.ac.jp}
}

\iclrfinalcopy
\begin{document}
\allowdisplaybreaks

\maketitle
\thispagestyle{plain}

\begin{abstract}
Generative models for counterfactual outcomes face two key sources of bias. Confounding bias arises when approaches fail to account for systematic differences between those who receive the intervention and those who do not. Misspecification bias arises when methods attempt to address confounding through estimation of an auxiliary model, but specify it incorrectly. We introduce DoubleGen, a doubly robust framework that modifies generative modeling training objectives to mitigate these biases. The new objectives rely on two auxiliaries---a propensity and outcome model---and successfully address confounding bias even if only one of them is correct. We provide finite-sample guarantees for this robustness property. We further establish conditions under which DoubleGen achieves oracle optimality---matching the convergence rates standard approaches would enjoy if interventional data were available---and minimax rate optimality. We illustrate DoubleGen with three examples: diffusion models, flow matching, and autoregressive language models.
\end{abstract}

\section{Introduction}

Generative models have achieved remarkable success at imitating real-world data. This capability underlies recent advances in image generation software \citep{ramesh2021zero,rombach2021highresolution} and large language models \citep{radford2018improving,touvron2023llama}. But sometimes, there is a need to go beyond the world as it is, and instead imitate counterfactual data---that which would have arisen had the world been intervened on in some way \citep{komanduri2023identifiable}.

For example, suppose a policymaker wants to see how medical records would look if a new treatment were universally adopted. The available records come from a time when resources were limited, and so the treatment was given preferentially to sicker patients. These patients also tend to have worse outcomes on treatment than others. A generative model trained on only treated patients' records would internalize this confounding, suggesting overly pessimistic outcomes. What the policymaker wants instead is to generate records as they \emph{would have appeared} had treatment been given to everyone.

\begin{table}[b]
    \caption{Selected attributes possessed by a larger percentage of smiling ($n$\;=\;78{,}080) than nonsmiling ($n$\;=\;84{,}690) CelebA faces. When trained on only smiling faces, {\color{CB5red}traditional generative models overrepresent these attributes}, failing to reflect \textbf{\color{CB5blue}how the population would look if everyone smiled}.}\label{tab:imbalance} \
    {\centering
    \begin{tabular}{l *{6}{>{\arraybackslash}b{4.6em}}}
        & Lipstick & Makeup & Female$^*$ & Earrings & No Beard & Blonde \\\hline
        \rule{0pt}{1.05em}{\color{CB5red}Smiling} & {\color{CB5red}56\%} & {\color{CB5red}47\%} & {\color{CB5red}65\%} & {\color{CB5red}26\%} & {\color{CB5red}88\%} & {\color{CB5red}18\%} \\
        Not Smiling & 38\% & 30\% & 52\% & 12\% & 79\% & 12\% \\\hline
        \rule{0pt}{1.05em}{\bf\color{CB5blue}Overall} & {\bf\color{CB5blue}47\%} & {\bf\color{CB5blue}38\%} & {\bf\color{CB5blue}58\%} & {\bf\color{CB5blue}19\%} & {\bf\color{CB5blue}83\%} & {\bf\color{CB5blue}15\%}
    \end{tabular}\vspace{.75em}
    }
    {\footnotesize$^*\,$\textit{Perceived binary sex, as labeled by human annotators \citep{liu2015faceattributes}.}}
\end{table}

Confounding can also arise in many other settings. More active internet users are more likely to be exposed to web campaigns \citep{chan2010evaluating}. Students with greater family resources are more likely to get extracurricular tutoring \citep{zhang2023effect}. Celebrities with certain attributes are more likely to smile---see \Cref{tab:imbalance} \citep{liu2015faceattributes}. In this work, we use this celebrity setting as an illustration, seeking to answer the question: what would celebrity photos look like if everyone smiled?

Formally, we aim to generate samples from the distribution $\mathbb{P}$ of a counterfactual outcome $Y^\star\in \mathcal{Y}$ under intervention $a^\star$. To accomplish this, we approximate a transport map $\phi_{\mathbb{P}}$ that takes as input simulated noise $U\sim \Pi$ and outputs $\phi_{\mathbb{P}}(U)\sim \mathbb{P}$---in other words, its pushforward satisfies $\phi_{\mathbb{P}\sharp}\Pi=\mathbb{P}$. If independent samples from $\mathbb{P}$ were available, a standard generative modeling approach such as variational autoencoders \citep{kingma2013auto}, diffusion models \citep{sohl2015deep}, or flow matching \citep{lipman2022flow} could be used. However, such samples are not available: $\mathbb{P}$ represents the distribution of outcomes if everyone had received $a^\star$, but in reality not everyone did.  Instead, a sample of independent copies of $Z=(X,A,Y)\sim P$ is available, and the intervention $A$ that someone received may depend on confounders $X$ of its effect on the outcome $Y$. Under a no-unmeasured confounders condition, $\mathbb{P}$ can be learned using data from $P$---see \Cref{sec:proposedApproach}.

\myparagraph{Contributions.} We introduce \nameref*{alg:double}, the first doubly robust framework for adapting standard generative modeling frameworks to generate counterfactuals (\Cref{sec:proposedApproach}). We illustrate how \nameref*{alg:double} can adapt three frameworks: flow matching, diffusion models, and autoregressive language models (\Cref{sec:examples}). We establish high-probability statistical divergence guarantees  (\Cref{sec:theoryOverview,sec:excRisk}) and develop a method to evaluate minimax rate optimality, which we use to give conditions for \nameref*{alg:double} diffusion models to be rate optimal (\Cref{sec:minimax}). Finally, we conduct experiments  (\Cref{sec:experiments} and \Cref{fig:faces}).

\begin{figure}[t]
    \centering
    \includegraphics[width=\textwidth]{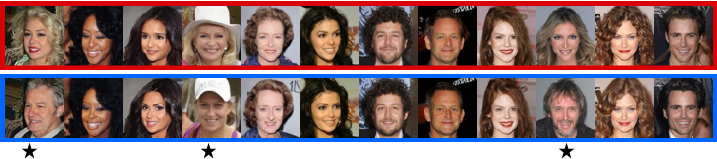}
    \caption{Counterfactual smiling celebrities generated by a \textbf{\color{CB5red}traditional diffusion model} trained~on only smiling faces (top) and a \textbf{\color{CB5blue}\nameref*{alg:double} diffusion model} (bottom). Columns contain coupled samples, with the random seed set to the same value before generation. The \raisebox{-0.1em}{\includegraphics[height=.85em]{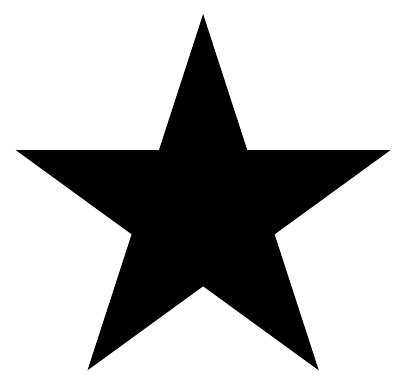}}'s mark the most qualitatively different pairs. 
    See \Cref{fig:moreFaces} for 200 random samples.}
    \label{fig:faces}
\end{figure}

\section{Related work and motivation}

Previous works have provided ways to learn a counterfactual transport map $\phi_{\mathbb{P}}$ from observational data. These methods use one of three generation strategies: iterative, joint, or direct.

Iterative approaches generate data according to a known causal ordering. In our setting, they first generate features $X$, then intervene to set $A=a^\star$, and finally generate an outcome. These strategies have been proposed using generative adversarial networks \citep{kocaoglu2017causalgan}, normalizing flows \citep{pawlowski2020deep}, variational autoencoders \citep{karimi2020algorithmic}, and diffusion models \citep{chaomodeling}. A disadvantage of these approaches is that they can suffer from error propagation, which is especially problematic if $X$ is high-dimensional (e.g., an image) \citep{javaloy2024causal}.

Joint generation approaches avoid error propagation by simultaneously learning the distribution of $(X,Y^\star)$. This has been done with normalizing flows \citep{khemakhem2021causal,javaloy2024causal}, variational graph autoencoders \citep{sanchez2021vaca}, and diffusion models \citep{sanchez2022diffusion}. However, these methods, too, solve a harder problem than is necessary: generating $(X,Y^\star)$ even when only $Y^\star$ is of interest. This can worsen sample efficiency and computation time.

Direct approaches generate only the counterfactual outcome $Y^\star$. \cite{wu2024counterfactual} describes how to use inverse probability weighting with (approximate) likelihood approaches---such as classifier-free guided diffusion models \citep{ho2022classifier} and conditional variational autoencoders \citep{sohn2015learning}. Our proposed approach---\nameref*{alg:double}---is also a direct approach. It applies to any loss-based generative modeling strategy, of which likelihood approaches are a special case. 

A limitation of all existing counterfactual generation approaches is that they are only singly robust. For iterative and joint generation approaches, this means they can only correctly generate samples of $Y^\star$ if they correctly model the distribution of $(X,Y^\star)$. For inverse probability weighted approaches, this means that they must correctly specify the propensity, $P\{A=a^\star\mymid X=\cdot\,\}$ \citep{wu2024counterfactual}. When these auxiliary models are incorrect, the resulting misspecification bias can lead to incorrect counterfactual distributions. Doubly robust estimators are less prone to this bias by remaining valid if either a propensity or outcome model is correct \citep{robins1994estimation}. These estimators have been used previously to estimate counterfactual distributions \citep{fawkes2022doubly,kennedy2023semiparametric,martinez2023efficient,luedtke2024one}, but not to generate samples from them.

Another limitation of existing works is that, with the exception of \citep{wu2024counterfactual}, they only show how to adapt individual generative modeling approaches, and so cannot be immediately applied when new ones are developed. 
For instance, despite the recent success of flow matching, this approach has not yet been adapted to generate counterfactuals.

From a theoretical perspective, existing counterfactual generation methods are not well understood. This is in contrast to non-counterfactual settings, where generalization upper bounds and minimax lower bounds have been developed \citep{lee2022convergence,de2022convergence,oko2023diffusion,chen2023score,lotfi2023non,fukumizu2024flow,holk2024statistical}. This theoretical gap is significant: without performance guarantees, practitioners cannot know when these methods will succeed or fail. 

\section{Proposed approach}\label{sec:proposedApproach}

Our proposed approach leverages observational data from $P$ to learn to generate samples from the counterfactual distribution $\mathbb{P}$. It does so under a no-unmeasured confounders assumption \citep{robins1986new,pearl2009causality}, which makes it possible to identify $\mathbb{P}$ with $P$ through the relation
\begin{align}
    \mathbb{P}\{Y^\star\in \mathcal{Y}'\}&= \medint\int P\{Y\in \mathcal{Y}'\mymid A=a^\star,X=x\}\, P_X(dx),\;\textnormal{ for all measurable sets $\mathcal{Y}'\subseteq \mathcal{Y}$,} \label{eq:ident}
\end{align}
with $P_X$ the marginal distribution of $X$ under $P$. 
Formally, \eqref{eq:ident} holds if $Z$ has a corresponding counterfactual outcome $Y^\star$ and the following hold: no unmeasured confounders ($Y^\star\indep A\mymid X$), consistency ($Y=Y^\star$ whenever $A=a^\star$), and positivity ($P\{A=a^\star\mymid X\}>0$ $P_X$-a.s.).

Our method builds upon existing generative modeling strategies---those that practitioners would use in idealized scenarios where draws from $\mathbb{P}$ are available. We restrict our focus to loss-based strategies expressible as in \Cref{alg:oracle} (\nameref*{alg:oracle}), which impose that the final transport map $\phi_n^\star$ is defined as a transformation of a hypothesis $\theta_n^\star$ selected by optimizing an empirical risk $R_n^\star$. This hypothesis may, for example, be a global empirical risk minimizer or a neural network whose weights were learned by stochastic gradient descent. As illustrated in the next subsection, flow matching, diffusion models, and autoregressive language models can all be expressed as in \nameref*{alg:oracle}.

\begin{algorithm}[b]
   \algcaption[OracleGen]{Oracle counterfactual generative modeling \\[.1em]
   {\color{CB5purple}\danger\, \textit{Cannot be used in practice}}\textit{ --- requires access to draws from the counterfactual distribution $\mathbb{P}$}}
   \label{alg:oracle}
   \linespread{1.05}\selectfont
\begin{algorithmic}[1]
    \Require $\bullet$ counterfactual data $Y_1^\star,Y_2^\star,\ldots,Y_n^\star\overset{\textnormal{iid}}{\sim} \mathbb{P}$
    \SilentRequire \begin{itemize}[leftmargin=*,itemindent=-3.1em]
        \item choice of generative modeling framework, defined by a hypothesis space $\Theta$, loss $\ell : \Theta\times\mathcal{Y}\rightarrow \mathbb{R}$, and sampling map transformation $\tau : \Theta\rightarrow \mathcal{Y}^{\mathcal{U}}$
    \end{itemize}
    \State \textbf{Risk minimization:} using any preferred approach, define $\theta_n^\star$ via $R_n^\star(\theta):= \frac{1}{n}\sum_{i=1}^n\ell(\theta,Y_i^\star)$
    \State \Return transport map $\phi_n^\star:=\tau(\theta_n^\star)$
\end{algorithmic}
\end{algorithm}

\begin{algorithm}[tb]
    \algcaption[DoubleGen]{Doubly robust counterfactual generative modeling\\
    {\color{CB5blue}\makebox[0pt][l]{$\square$}\raisebox{.15ex}{\hspace{0.1em}$\checkmark$}\, \textit{Can be used in practice}}\textit{ --- only requires draws from the factual distribution $P$}}
   \label{alg:double}
   \linespread{1.05}\selectfont
\begin{algorithmic}[1]\vspace{-.075em}
    \Require $\bullet$ data $Z_1,Z_2,\ldots,Z_n\overset{\textnormal{iid}}{\sim} P$, partitioned into multisets $\mathcal{Z}_n^1,\mathcal{Z}_n^2$ of sizes $\lfloor n/2\rfloor$ and $\lceil n/2\rceil$
    \SilentRequire $\bullet$ choice of generative modeling framework, as in \nameref*{alg:oracle}
    \State \textbf{Nuisance estimation:} for $j\in \{1,2\}$, use observations in $\mathcal{Z}_n^j$ to obtain estimates $\psi_n^j$ and $\alpha_n^j$ of
    \begin{enumerate}[label=(\roman*)]
        \item \textbf{Outcome model:} a conditional transport map with $\psi_P(\,\cdot\mymid x)_{\sharp}\Pi= P_{Y|A=a^\star,X=x}$ $P_X$-a.s.
        \item \textbf{Inverse propensity:} $\alpha_P(x):= 1/P(A=a^\star\mymid X=x)$
    \end{enumerate}
    \State\label{ln:riskMin} \textbf{Risk minimization:} using any preferred approach, define $\theta_n$ based on
    \[
        R_n(\theta):= \frac{1}{n}\medop\sum_{j=1}^2\medop\sum_{z\in\mathcal{Z}_n^{3-j}}\medint\int \big[1(a=a^\star)\alpha_n^j(x)\big\{\ell(\theta,y)-\ell(\theta,\psi_n^j(u|x))\big\} + \ell(\theta,\psi_n^j(u|x))\big] \Pi(du)
    \]
    \Statex $\triangleright$ \textit{unbiased gradients can be obtained in the usual way, via random sampling of $(j,z,u)$}
    \State \Return $\phi_n:=\tau(\theta_n)$
\end{algorithmic}
\end{algorithm}

\Cref{alg:double} displays our proposed approach, \nameref*{alg:double}. In it, the oracle risk $R_n^\star$ is replaced by a risk $R_n$ that can be computed using data drawn from $P$. Evaluating this risk requires estimating two auxiliary nuisances: an inverse propensity $1/P(A=a^\star\mymid X=\cdot\,)$ and a conditional transport map $\psi_P$ that returns draws of $Y\mymid A=1,X$. These nuisances can be constructed using any preferred approach, such as Riesz regression for the inverse propensity \citep{chernozhukov2021automatic} or a diffusion model for the outcome model \citep{batzolis2021conditional}. The nuisance estimates are used to construct an augmented inverse probability weighted estimator of the risk \citep{robins1994estimation}. The resulting estimator, $R_n$, is doubly robust \citep{scharfstein1999adjusting}, in the sense that, for each $\theta\in\Theta$, $R_n(\theta)$ is a consistent estimator of the counterfactual population risk $E_{\mathbb{P}}[\ell(\theta,Y^\star)]$ if either nuisance is estimated consistently \citep{van2003unified,bang2005doubly}. 
Such double robustness has previously been leveraged when deriving causal machine learning guarantees \citep[e.g.,][]{rotnitzky2006doubly,laan2006cross,rubin2007doubly,luedtke2017sequential,rotnitzky2017multiply,nie2021quasi,kennedy2023towards,morzywolek2023general,foster2023orthogonal,van2024combining}.

\section{Examples of frameworks expressible as in \nameref*{alg:oracle}}\label{sec:examples}

\begin{example}[label=ex:flow,name=Flow matching]
    Flow matching estimates a transport map $\phi_{\mathbb{P}}$ that solves an ordinary differential equation \citep{lipman2022flow}. Taking $\mathcal{U}=\mathcal{Y}=\mathbb{R}^d$, a  simple version of this approach takes $\phi_{\mathbb{P}}(u):=y_1$, with the value of $y_1$ implied by the differential equation on $[0,1]$ satisfying initial condition $y_0=u$ with $y_t':=dy_t/dt$ satisfying
    \begin{align*}
        y_t'&= E\left[Y^\star-U\mymid (1-t)U+tY^\star=y_t\right],
    \end{align*}
    where $(U,Y^\star)\sim \Pi\times \mathbb{P}$ with $\Pi=N(0_d,I_d)$ \citep{liu2022flow}. This can be estimated as in \nameref*{alg:oracle} by letting $\Theta$ denote an appropriately restricted (e.g., bounded, smooth) set of $\mathcal{Y}\times [0,1]\rightarrow\mathcal{Y}$ functions, taking $\ell(\theta,y)=\int_0^1 E_\Pi[\|y-U-\theta([1-t]U+ty,t)\|^2] dt$, and letting $\tau(\theta)(u)$ denote the value of $y_1$ implied by the differential equation with initial condition $y_0=u$ and $dy_t/dt=\theta(y_t,t)$.
\end{example}

\begin{example}[label=ex:diffusion,name=Diffusion model] 
Diffusion models seek to reverse an iterative process that gradually converts data from $\mathbb{P}$ into Gaussian noise \citep{song2020score}. The noising process is described via a stochastic differential equation (SDE) on $\mathbb{R}^d\supseteq \mathcal{Y}$. A common framing takes $Y_0\sim \mathbb{P}$ and then, for $W$ a $d$-dimensional Wiener process on $[0,\overline{t}]$, evolves as $dY_t= -\beta_t Y_t\, dt + \sqrt{2\beta_t}\,dW_t$, with $\overline{t}<\infty$ a truncation time at which the noising process is terminated and $t\mapsto \beta_t$ a smooth map from $\mathbb{R}$ into $[\underline{\beta},\overline{\beta}]\subset (0,\infty)$. A strong solution exists provided $E_{\mathbb{P}}[\|Y_0\|^2]<\infty$ \citep[][Thm.~5.2.1]{oksendal2013stochastic}.

An approximate transport map $\phi_{\mathbb{P}}$ from $\Law(Y_{\overline{t}})\times \Law(\widetilde{W})$ to $\mathbb{P}$ can be obtained by reversing this SDE \citep{anderson1982reverse}, where $\widetilde{W}$ denotes a $d$-dimensional Wiener process with time flowing backward from $\overline{t}$ to $\underline{t}$. The truncation time $\underline{t}$ can be chosen to be slightly larger than $0$ to avoid instability \citep{oko2023diffusion}. To generate a sample using $\phi_{\mathbb{P}}$, first independently draw $\widetilde{Y}_{\overline{t}}\sim \Law(Y_{\overline{t}})$ and $\widetilde{W}$, and then let $\phi_{\mathbb{P}}(\widetilde{Y}_{\overline{t}},\widetilde{W})$ be a strong solution $\widetilde{Y}_{\underline{t}}$ to the reverse-time SDE
\begin{align}
    d\widetilde{Y}_t&= -\beta_t\big[\widetilde{Y}_t + 2\theta_{\mathbb{P}}(\tilde{Y}_t,t)\big]dt + \sqrt{2\beta_t}\, d\widetilde{W}_t, \label{eq:revTimeSDE}
\end{align}
where $\theta_{\mathbb{P}}(y,t):= \nabla_{\!y} \log \mathbbm{p}_t(y)$ is the score of $\mathbbm{p}_t$, the density of $Y_t$.
This score rewrites as $\theta_{\mathbb{P}}(y,t)=\{\mu_t E\left[Y_0\mymid Y_t=y\right]-y\}/\sigma_t^2$ with $\mu_t:=\exp(-\int_0^t \beta_{v}\, dv)$ and $\sigma_t^2:=1-\exp(-2\int_0^t \beta_{v}\, dv)$, and belongs to the set $\Theta$ of maps from $\mathbb{R}^d\times [\underline{t},\overline{t}]$ to $\mathbb{R}^d$.

The approximate transport map $\phi_{\mathbb{P}}$ cannot be used to generate samples in practice, because using it relies on knowing two $\mathbb{P}$-dependent quantities: $\Law(Y_{\overline{t}})$ and $\theta_{\mathbb{P}}$. Diffusion modeling replaces both by approximations. For the first, it uses that $\Law(Y_{\overline{t}})$ is approximately Gaussian noise, $N(0_d,I_d)$, provided the hyperparameter $\beta$ is chosen so $(\mu_{\overline{t}},\sigma_{\overline{t}}^2)\approx (0,1)$; this holds since $Y_t\mymid Y_0\sim N(\mu_t  Y_0,\sigma_t^2 I_d)$. Second, it uses that $\theta_{\mathbb{P}}$ can be estimated using the denoising score matching loss \citep{vincent2011connection}
\[
    \ell(\theta,y)={\medint\int_{\underline{t}}^{\overline{t}}} \,E\big[\big\|(\mu_t Y_0-Y_t)/\sigma_t^2-\theta(Y_t,t)\big\|^2\,\big|\, Y_0=y\big] dt.
\]
The final map $\tau(\theta)$ is defined a.s. over $(\widetilde{Y}_{\overline{t}},\widetilde{W})\sim N(0_d,I_d)\times \Law(\widetilde{W})=:\Pi$ so that $\tau(\theta)(\widetilde{Y}_{\overline{t}},\widetilde{W})$ is the value of $\widetilde{Y}_{\underline{t}}$ under a strong solution of an SDE that evolves in as \eqref{eq:revTimeSDE}, but with $\theta_{\mathbb{P}}$ replaced by $\theta$.
\end{example}

\begin{example}[label=ex:autoreg,name=Unsupervised pretraining of autoregressive language model] 
Autoregressive language models generate sequences $Y(1),Y(2),\ldots,Y(d)$ whose elements belong to a token dictionary $[k]$ \citep{graves2013generating,radford2019language,touvron2023llama}. Token $k$ marks the end of content and token $1$ is used for padding thereafter, meaning $Y(j)=k$ implies $Y(i)=1$ for all $i>j$; removing these tokens in post-processing allows generation of variable-length outputs.

Samples from a token-sequence distribution $\mathbb{P}$ can be obtained via ancestral sampling, which recursively draws the next token from a categorical distribution conditional on previous ones \citep{bishop2006pattern}. 
Expressing this procedure via inverse transform sampling yields a transport map $\phi_{\mathbb{P}}(\cdot)=(\phi_{\mathbb{P},j}(\cdot))_{j=1}^d$ from $\Pi=\mathrm{Unif}[0,1]^{d}$ to $\mathbb{P}$. This map is defined recursively through the conditional quantile function of $Y^\star(j)$ as $\phi_{\mathbb{P},j}(u)= Q_{\mathbb{P},j}\left(u_j\mymid Y^\star(i)=\phi_{\mathbb{P},i}(u) : i<j\right)$, $j=1,2,\ldots,d$. 

The transport map $\phi_{\mathbb{P}}$ can be estimated as in \nameref*{alg:oracle}. To do this, $\Theta$ is taken to be the set of functions mapping from $[k]^{d-1}$ to the ($k-1$)-simplex. The cross-entropy loss
\begin{align*}
    \ell(\theta,y)&= -{\medop\sum_{j : y(j)\not=1}} \log \theta_{y(j)}(1,1,\ldots,1,y(1),y(2),\ldots,y(j-1))
\end{align*}
is used \citep{graves2013generating}, with the input to each $\theta_{y(j)}$ left-padded with $1$s to make them ($d-1$)-dimensional. The risk minimization step of \nameref*{alg:oracle} may, for example, train a transformer model that masks the padded tokens \citep{vaswani2017attention}. The inverse transform sampling map $\tau(\theta)(\cdot)=(\tau_j(\theta)(\cdot))_{j=1}^d$ is defined recursively as $\tau_j(\theta)(u)=Q_{\theta,j}(u_j\mymid \tau_i(\theta)(u) : i<j)$, with the conditional quantile function $Q_{\theta,j}$ defined as the left-continuous generalized inverse of the distribution function
$$F_{\theta,j}(\,\cdot\,\mymid \tau_i(\theta)(u) : i<j)= {\medop\sum_{m=1}^{k}} 1\{m\le\cdot\,\}\,\theta_{m}\big(1,1,\ldots,1,\tau_1(\theta)(u),\tau_2(\theta)(u),\ldots,\tau_{j-1}(\theta)(u)\big).$$
The transport map $\phi_{\mathbb{P}}$ defined earlier equals $\tau(\theta_{\mathbb{P}})$ with $\theta_{\mathbb{P}}\in \argmin_{\theta\in \Theta}E_{\mathbb{P}}[\ell(\theta,Y^\star)]$.
\end{example}

\begin{table}[tb]\centering
    \caption{Examples of generative modeling paradigms that can be expressed as in \nameref*{alg:oracle}.}
    \begin{tabular}{@{\hspace{0.3em}}l l l l l@{\hspace{0.3em}}}
        & Outcome type ($Y^\star$) & Hypothesis ($\theta_{\mathbb{P}}$) & Loss ($\ell$) & Sampler ($\tau$) \\\midrule
        Flow matching & $\mathbb{R}^d$ (e.g., image) & vector field  & velocity matching & ODE solver \\
        Diffusion model & $\mathbb{R}^d$ (e.g., image) & score & denoising score matching & SDE solver \\
       Autoreg. model & $[k]^d$ (token seq.) & next-token prob.  & cross-entropy & ancestral
    \end{tabular}
\end{table}

\section{Theoretical guarantees}\label{sec:guarantees}

\subsection{Overview}\label{sec:theoryOverview}

Our guarantees for the transport map $\phi_n$ are measured in terms of a divergence $D$, defined as a nonnegative function of two distributions satisfying $D(\mathbb{P}_1,\mathbb{P}_2)=0$ if and only if $\mathbb{P}_1=\mathbb{P}_2$. Giving these guarantees requires relating the divergence of a pushforward $\tau(\theta)_{\sharp}\Pi$ from $\mathbb{P}$ to the generalization error $\mathcal{G}_{\mathbb{P}}(\theta):=\inf_{\theta^\star\in\Theta}\int [\ell(\theta,y)-\ell(\theta^\star,y)]\,\mathbb{P}(dy)$ and ensuring that the output of \nameref*{alg:double} makes this generalization error small with high probability.

\begin{condenum}[label={\bf C\arabic*)},ref=\arabic*,topsep=.5\topsep]
    \item\label{cond:divergence} \textit{Divergence dominated by generalization error:} There exist $b,\Cdivergence,\epsilon>0$ such that
    \begin{align}
        D(\mathbb{P},\tau(\theta)_{\sharp}\Pi)\le \Cdivergence\mathcal{G}_{\mathbb{P}}(\theta)^b + \epsilon\hspace{1em} \textnormal{ for all $\theta\in\Theta$}. \label{eq:regBound}
    \end{align}
    \item\label{cond:genError} \textit{Generalization bound:} For $s,r>0$, $\mathcal{G}_{\mathbb{P}}(\theta_n)\le r$ with probability (w.p.) at least $1-e^{-s}$.
\end{condenum}

\begin{proposition}[Divergence bound]\label{prop:divGuarantee}
    Under \cref{cond:divergence,cond:genError}, $D(\mathbb{P},\phi_{n\sharp}\Pi)\le \Cdivergence r^b+\epsilon$ w.p. at least $1-e^{-s}$.
\end{proposition}
The proof is immediate. 
This result is important because it yields exactly the sort of guarantee desired for a generative modeling algorithm, provided its conditions are satisfied. The first has already been established for all our examples (\Cref{app:divergence}). 
We provide a means to establish the second in \Cref{sec:excRisk}.

The minimax optimality of $\phi_n$ follows under conditions, as we show in \Cref{sec:minimax}. The main insight is simple: because \nameref*{alg:double} must learn from factual data alone while \nameref*{alg:oracle} has access to counterfactuals, the oracle problem is no harder. We formalize this in \Cref{thm:minimax} by showing that minimax lower bounds for \nameref*{alg:oracle} also apply to \nameref*{alg:double}. Hence, optimality can be assessed by importing existing lower bounds from the non-counterfactual generative modeling literature.

\subsection{Generalization bounds for \nameref*{alg:double}}\label{sec:excRisk}

We now give a generalization bound for \nameref*{alg:double} when implemented via an empirical risk minimizer over $\underline{\Theta}\subseteq\Theta$. While we focus on empirical risk minimizers here, generalization bounds can also be derived for other frameworks---see \Cref{app:beyondERM}. Our guarantee will rely on several conditions.
\begin{condenum}[resume*]
    \item\label{cond:RMexists} \textit{Existence of risk minimizer:} there exists $\theta_{\mathbb{P}}\in \argmin_{\theta\in\Theta}\int \ell(\theta,y)\, \mathbb{P}(dy)$.
\end{condenum}
The remaining conditions are assumed to hold for a common choice of risk minimizer $\theta_{\mathbb{P}}$. We also define the centered losses $\ell_{\mathbb{P}}(\theta)(y):=\ell(\theta,y)-\ell(\theta_{\mathbb{P}},y)$ and loss class $\ell_{\mathbb{P}}(\underline{\Theta}):=\{\ell_{\mathbb{P}}(\theta) : \theta\in\underline{\Theta}\}$. 
\begin{condenum}[resume*]
    \item\label{cond:bddLoss} \textit{Bounded loss:} there exists $\Closs<\infty$ such that $\sup_{\theta\in\underline{\Theta}}\|\ell_{\mathbb{P}}(\theta)\|_{L^\infty(\mathbb{P})}\le \Closs$.
    \item\label{cond:curvature} \textit{Curvature of risk:} there exists $\Ccurv<\infty$ such that, for all $\theta\in \underline{\Theta}$, $\|\ell_{\mathbb{P}}(\theta)\|_{L^2(\mathbb{P})}^2\le \Ccurv \mathcal{G}_{\mathbb{P}}(\theta)$.
    \item\label{cond:entropyStronger} \textit{Uniform entropy integral bound (see \Cref{app:covNumReview}):} $J(\delta,\ell_{\mathbb{P}}(\underline{\Theta}))<\infty$ for some $\delta>0$.
\end{condenum}
The above conditions yield a generalization bound for \nameref*{alg:oracle} (\Cref{app:oracleGenBenchmark}). Similar conditions are commonly used to derive rate guarantees in statistical learning \citep{bartlett2006local}.

The remaining conditions account for \nameref*{alg:double}'s reliance on nuisances. The first imposes smoothness of the loss. For collections of conditional transport maps $\Psi:=\mathcal{Y}^{\mathcal{U}\times\mathcal{X}}$ and $\Psi_P:= \{\psi : \psi(\,\cdot\mymid x)_{\sharp}\Pi= P_{Y|A=a^\star,X=x}\ P_X\textnormal{-a.s.}\}$, this condition involves a function $d_\Psi(\,\cdot\,,\Psi_P) : \Psi\rightarrow [0,\infty)$. Our bound will be tightest when $d_\Psi(\psi_n^j,\Psi_P)$ is small. The measure $d_\Psi$ need not arise from a metric, but it must satisfy $d_\Psi(\psi_P,\Psi_P)=0$ for all $\psi_P\in\Psi_P$---see \Cref{lem:mixedLipSuff} for one possible choice of $d_\Psi$.
\begin{condenum}[resume*]
    \item\label{cond:mixedLipschitzWeak} \textit{Mixed-Lipschitz loss:} there exists $\Cmixedweak<\infty$ such that, for all $\theta\in\underline{\Theta}$, $\psi\in\Psi$, and $\psi_P\in\Psi_P$,
    \begin{align*}
       &\medint\int\left\{\medint\int \left[ \ell_{\mathbb{P}}(\theta)(\psi(u|x)) - \ell_{\mathbb{P}}(\theta)(\psi_P(u|x))\right]\,\Pi(du)\right\}^2 P_X(dx)\le \Cmixedweak \mathcal{G}_{\mathbb{P}}(\theta)\,d_{\Psi}^2(\psi,\Psi_P).
    \end{align*}
    \item\label{cond:strongPositivity} \textit{Strong positivity:} there exists $\Cpos>0$ and a version of $\alpha_P$ such that $\alpha_P\in [1,\Cpos]^{\mathcal{X}}$. Moreover, the estimates of $\alpha_P$ respect this bound, in that $\alpha_n^j\in [1,\Cpos]^{\mathcal{X}}$ for all $j\in [2]$.
\end{condenum}
\begin{theorem}[\nameref*{alg:double} generalization bound, informal statement]\label{thm:excRiskOSL}
    Suppose there exists $\theta_n\in\argmin_{\theta\in\underline{\Theta}} R_n(\theta)$ and \cref{cond:strongPositivity,cond:RMexists,cond:bddLoss,cond:mixedLipschitzWeak,cond:curvature,cond:entropyStronger}. If $s>0$ and $\delta_n$ satisfies $J(\delta_n,\ell_{\mathbb{P}}(\underline{\Theta}))\le n^{1/2} \delta_n^2$, then
    \begin{align}
        \mathcal{G}_{\mathbb{P}}(\theta_n)&\lesssim \inf_{\theta\in\underline{\Theta}}\mathcal{G}_{\mathbb{P}}(\theta) + \delta_n^2 + s/n + {\color{CBteal}\bm{\max_{j\in [2]}\|\alpha_n^j-\alpha_P\|_{L^2(P_X)}^2d_\Psi^2(\psi_n^j,\Psi_P)}}   \label{eq:excRiskBdOSLinformal}   
    \end{align}
    w.p. at least $1-e^{-s}$. The final term is \textbf{\color{CBteal}doubly robust}, vanishing if $\alpha_n^j=\alpha_P$ or $\psi_n^j\in\Psi_P$ for $j\in [2]$.
\end{theorem}
Above, `$\lesssim$' hides a multiplicative constant that does not depend on $n$. The bound reveals a tradeoff: increasing the size of $\underline{\Theta}$ will decrease the approximation error, $\inf_{\underline{\theta}\in\underline{\Theta}}\mathcal{G}_{\mathbb{P}}(\underline{\theta})$, but increase the complexity term. This informal statement captures these main dependencies and is valid when $\max_j d_\Psi^2(\psi_n^j,\Psi_P)$ is almost surely bounded. See \Cref{thm:excRiskOSLFormal} in \Cref{app:formalTheorem} for the formal statement, which gives explicit constants and also holds when $\max_j d_\Psi^2(\psi_n^j,\Psi_P)$ is unbounded.

Appendix \ref{app:genBoundDiscussion} interprets and discusses the above generalization bound: the complexity term $\delta_n^2$, conditions under which the doubly robust term will be negligible, a localized version of \Cref{thm:excRiskOSL} that can be used to establish oracle-optimal rates of convergence, and connections to existing approaches.

\subsection{Minimax optimality}\label{sec:minimax}

\myparagraph{Upper bound on worst-case divergence of an empirical risk minimizer.} 
When paired with \Cref{prop:divGuarantee}, \Cref{thm:excRiskOSL} provides a means to provide divergence guarantees for the transport map $\phi_n=\tau(\theta_n)$. In particular, if \cref{cond:divergence} and the conditions of \Cref{thm:excRiskOSL} hold, the doubly robust term is no more than $\delta_n^2$ with sufficient probability, and $\epsilon\le \delta_n^2$, then \Cref{prop:divGuarantee} shows that $D(\mathbb{P},\phi_{n\sharp}\Pi)\lesssim [\inf_{\theta\in\underline{\Theta}}\mathcal{G}_{\mathbb{P}}(\theta) + \delta_n^2 + s/n]^b$ 
w.p. at least $1-e^{-s}$, where `$\lesssim$' denotes inequality up to constants that only depend on the $C_j$ constants indexing the conditions for \Cref{prop:divGuarantee,thm:excRiskOSL}. If $\delta_n=\Omega(\log(n)/\sqrt{n})$---as is typical for nonparametric hypothesis classes $\underline{\Theta}$---then Thm.~1 in \cite{mey2020note} yields the in-expectation bound $E_{P^n}[D(\mathbb{P},\phi_{n\sharp}\Pi)]\lesssim [\inf_{\theta\in\underline{\Theta}}\mathcal{G}_{\mathbb{P}}(\theta) + \delta_n^2]^b$. 
This upper bound only depends on $\mathbb{P}$ and $P$ through the constants indexing the conditions for \Cref{prop:divGuarantee,thm:excRiskOSL}. Hence, it holds uniformly over all $(\mathbb{P},P)$ in any collection $\mathcal{M}$ over which these constants are uniformly bounded and $\mathbb{P}$ is identified through \eqref{eq:ident}. This in turn provides a way to upper bound \nameref*{alg:double}'s worst-case expected divergence over $\mathcal{M}$.

Interesting models $\mathcal{M}$ arise by defining a model $\mathcal{P}^\star$ for the counterfactual distribution $\mathbb{P}$ and then, for each $\mathbb{P}$, letting $\mathcal{P}(\mathbb{P})$ denote a collection of $P$ satisfying \eqref{eq:ident}. 
Both $\mathcal{P}^\star$ and $\mathcal{P}(\mathbb{P})$ may be subject to local or global smoothness constraints so that the relevant constants from \cref{cond:divergence,cond:genError,cond:strongPositivity,cond:RMexists,cond:bddLoss,cond:mixedLipschitzWeak,cond:curvature,cond:entropyStronger} are uniformly bounded over $(\mathbb{P},P)$ in $\mathcal{M}:=\{(\mathbb{P},P) : \mathbb{P}\in\mathcal{P}^\star,P\in\mathcal{P}(\mathbb{P})\}$. We then have the following bound on the worst-case performance of the transport map $\phi_n$ in terms of expected divergence:
\begin{align}
    \sup_{\mathbb{P}\in\mathcal{P}^\star}\sup_{P\in\mathcal{P}(\mathbb{P})} E_{P^n}[D(\mathbb{P},\phi_{n\sharp}\Pi)]\lesssim \big[\inf_{\theta\in\underline{\Theta}}\mathcal{G}_{\mathbb{P}}(\theta) + \delta_n^2\big]^b. \label{eq:maximalRisk}
\end{align}

\myparagraph{Minimax lower bound for any generative modeling algorithm.} We now show that any minimax lower bound for \nameref*{alg:oracle} is also a lower bound for \nameref*{alg:double}. This provides a simple path for establishing the minimax optimality of $\phi_n$: show that an existing lower bound from the non-causal generative modeling literature matches the rate of decay of \eqref{eq:maximalRisk} in $n$.

In the following result,  $\mathcal{T}^\star$ and $\mathcal{T}$ denote unrestricted sets of oracle and non-oracle generative modeling procedures returning transport maps in $\mathcal{Y}^{\mathcal{U}}$. 
Procedures in $\mathcal{T}^\star$ take as input counterfactuals $Y_{[n]}^\star:=\{Y_i^\star\}_{i\in[n]}$, while those in $\mathcal{T}$ take factual data $Z_{[n]}:=\{Z_i\}_{i\in [n]}$. They also take in exogeneous noise $V\sim \nu:=\mathrm{Unif}[0,1]$; this allows, for example, the procedures to train a neural network using stochastic gradient descent or use a stopping criterion based on a training-validation split. 
\begin{theorem}\label{thm:minimax}
    Let $\mathbb{E}$ and $E$ being expectations under sampling from $\mathbb{P}^n\times\nu$ and $P^n\times \nu$. It holds that
    \begin{align*}
        {\color{CB5purple}\bm{\inf_{T^\star\in\mathcal{T}^\star}\sup_{\mathbb{P}\in\mathcal{P}^\star}\mathbb{E}\Big[D\Big(\mathbb{P},T^\star(Y_{[n]}^\star,V)_{\sharp}\Pi\Big)\Big]}}\le {\color{CB5blue}\bm{\inf_{T\in\mathcal{T}}\sup_{\mathbb{P}\in\mathcal{P}^\star}\sup_{P\in\mathcal{P}(\mathbb{P})}E\Big[D\Big(\mathbb{P},T(Z_{[n]},V)_{\sharp}\Pi\Big)\Big]}} ,
    \end{align*}
    Hence, any~lower bound on the \textbf{\color{CB5purple}oracle minimax risk} also lower bounds the \textbf{\color{CB5blue}factual minimax risk}.
\end{theorem}

\myparagraph{Example~\ref{ex:diffusion} (cont.): rate optimality of DoubleGen diffusion modeling.} 
We build on a recent work that showed non-counterfactual diffusion modeling nearly achieves the minimax rate for estimating a Besov-smooth density with respect to the total variation distance \citep{oko2023diffusion}. We show \nameref*{alg:double} achieves the same guarantee provided the nuisances are estimated well enough.

We begin by deriving a minimax lower bound for our counterfactual generative modeling problem. When doing this, we let $\mathcal{P}^\star$ denote the set of distributions on $[-1,1]^d$ whose densities belong to a fixed-radius ball in the Besov space $B_{p,q}^s([-1,1]^d)$, with $p,q\in [1,\infty)$ and $s>0\vee d(1/p-1/2)$ \citep[][Sec.~2.7.2]{vaart2023empirical}. Proposition~D.4 of \cite{oko2023diffusion} gives a minimax lower bound for any deterministic estimator of a distribution in $\mathcal{P}^\star$, and the convexity of the total variation distance ($\TV$) implies that any randomized estimator is dominated by a deterministic one. Hence, for a constant $C$ not depending on $n$, the oracle minimax risk with $D=\TV$ lower bounds as
\begin{align*}
    \inf_{T^\star\in\mathcal{T}^\star}\sup_{\mathbb{P}\in\mathcal{P}^\star}E_{\,\mathbb{P}^n\times \nu}\left[\TV\big(\mathbb{P},T^\star(Y_{[n]}^\star,V)_{\sharp}\Pi\big)\right]&\ge {\color{CB5purple}\bm{C\, n^{-\frac{s}{2s+d}}}}.
\end{align*}
By  \Cref{thm:minimax}, the same \textbf{\color{CB5purple}minimax lower bound} is valid for our counterfactual generation problem.

The instance of \nameref*{alg:double} diffusion we study is nearly the same as the one used in \cite{oko2023diffusion}, differing only in the choice of loss function and the need to estimate the nuisances used to define it. An empirical risk minimizer is run over a deep, sparse neural network class $\underline{\Theta}$ that grows with $n$. This class is rich enough to approximate the true score well (\Cref{lem:neuralNet}), while also being small enough so that the entropy of $\ell_{\mathbb{P}}(\underline{\Theta})$ is not too large (\Cref{lem:neuralNetEntropy}). 

The conditions of our generalization bound are satisfied when $\underline{\Theta}$ is specified in this way (\Cref{app:diffusionConditions}). Combining this with \Cref{prop:divGuarantee} gives a TV bound, which we state below. In this result, `$\lesssim$' denotes an inequality up to a multiplicative constant that may depend on the constants from the conditions but may not depend on $n$ or the particular instance of $(\mathbb{P},P)$. 
\begin{theorem}[TV bound for \nameref*{alg:double} diffusion]\label{thm:diffusionMinimax}
    Suppose $\phi_n:=\tau(\theta_n)$ for $\theta_n\in \argmin_{\theta\in\underline{\Theta}} R_n(\theta)$, with $\underline{\Theta}$ the neural network class from \Cref{lem:neuralNet}. Let $d_{\Psi}$ be as in \Cref{lem:mixedLipSuff} and suppose  $\max_j d_\Psi(\psi_n^j,\Psi_P)$ is a.s. bounded uniformly in $n$. 
    If \cref{cond:strongPositivity,cond:diffusionBddDensity,cond:diffusionBesov,cond:diffusionSupport,cond:diffusionTruncation,cond:diffusionBoundary}, $r>0$, and $n$ is large enough, then, w.p. at least $1-e^{-r}$,
    \[
        \TV(\mathbb{P},\phi_{n\sharp}\Pi)\lesssim \log^{\frac{17}{2}}(n)\,{\color{CB5blue}\bm{n^{\frac{-s}{2s+d}}}} + \sqrt{r/n} + \max_{j\in [2]}\|\alpha_n^j-\alpha_P\|_{L^2(P_X)}d_{\Psi}(\psi_n^j,\Psi_P).
    \]
\end{theorem}
If the first term on the right dominates with high probability uniformly over $P$, then this shows that \nameref*{alg:double} diffusion achieves the \textbf{\color{CB5blue}minimax rate} over a Besov smoothness class, up to polylogarithmic factors. This is the same guarantee as \cite{oko2023diffusion} established for traditional diffusion models. 

Both our analysis and that of \citet{oko2023diffusion} invoke two additional regularity assumptions beyond Besov smoothness---boundedness away from zero and smoothness on the boundary (\cref{cond:diffusionBddDensity,cond:diffusionBoundary}). Ours further assumes that the nuisances can be estimated at a suitable rate. Because the available minimax lower bound is stated for the larger Besov class, a matching lower bound for the slightly more restrictive class considered here (and in \citeauthor{oko2023diffusion}) is still needed to claim full minimax optimality.  Closing this gap is an interesting direction for future work.

In \Cref{app:flow,app:autoreg}, we similarly combine \Cref{prop:divGuarantee,thm:excRiskOSLFormal} to establish divergence bounds for our other two examples. In future work, it would be interesting to show they are are rate optimal.

\section{Numerical experiments}\label{sec:experiments}

\subsection{Generating counterfactual faces}\label{sec:faces}

We evaluated \nameref*{alg:double}'s performance on generating synthetic celebrity faces, under the intervention that they are all smiling. We trained diffusion models using 162,770 images from CelebA \citep{liu2015faceattributes}, withholding the other 39,829 for evaluation. Each image $Y$ is accompanied by a binary smiling indicator $A$ and 31 potential baseline confounders, such as those in \Cref{tab:imbalance}.

We compare to three baselines. The first ignores potential confounding by fitting a traditional diffusion model with only smiling instances; \Cref{tab:imbalance} suggests this is inadvisable, and so we call it the `na\"{i}ve' approach. The second uses inverse probability weighting (IPW), equivalent to \Cref{alg:double} with $\ell(\theta,\psi_n^j)$ replaced by 0. The third uses plug-in estimation, equivalent to \Cref{alg:double} with $\alpha_n^j$ replaced by 0.

To ensure fair comparison, the same nuisance estimates were supplied to all methods. We estimated each nuisance twice: once in a well-specified setting using all available training data, and again in a misspecified setting where nuisance models were trained to overrepresent dark-haired instances. Further details on our experiment can be found in \Cref{app:moreExperimentalSet}.

The left half of \Cref{tab:facesMetricsArcFace} displays the results. When both nuisances were well-specified, \nameref*{alg:double} achieved better Fr\'{e}chet and kernel distances than the na\"{i}ve approach. Its precision was slightly lower, suggesting the na\"{i}vely-generated faces more often resembled real faces. However, its recall was considerably higher, suggesting counterfactual smiling faces are more likely to be represented among \nameref*{alg:double}'s samples. Overall, \nameref*{alg:double} performed comparably to plug-in and IPW estimation when both nuisances were well-specified. When nuisances were misspecified, \nameref*{alg:double} was more robust. \Cref{tab:facesMetricsInception} in the appendix shows similar results with respect to other metrics.

\subsection{Generating counterfactual product reviews}\label{sec:reviews}

We next conducted a semi-synthetic experiment using the Amazon Reviews 2023 dataset \citep{hou2024bridging}, which consists of roughly 570 million reviews about 48 million products. We used real baseline features and review texts and a synthetic intervention sampled from a known propensity $\pi$. This semi-synthetic setup provides two key advantages: it gives us access to ground truth for evaluation, and it ensures that the identifiability condition in \eqref{eq:ident} is satisfied by construction.

We used low-rank adaptation (LoRA) \citep{hu2022lora} to finetune Llama-3.2-1B \citep{dubey2024llama}. We compare performance using the \nameref*{alg:double} loss to the same three baselines as in the previous experiment, under the same set of misspecification scenarios.  
Further details are in \Cref{app:amazonReviewsExperimentalSetup}.

The right half of \Cref{tab:facesMetricsArcFace} displays the results. When at least one nuisance was well-specified, \nameref*{alg:double} outperformed the na\"{i}ve approach across all metrics. \nameref*{alg:double} also performed similarly to or better than other methods when both nuisances were correct. Under propensity misspecification, \nameref*{alg:double} outperformed IPW. Under outcome model misspecification, both \nameref*{alg:double} and plug-in estimation maintained similar performance to what they achieved under correct specification, suggesting outcome model misspecification was mild. \Cref{fig:reviews} in the appendix provides examples of reviews generated by \nameref*{alg:double}, while \Cref{fig:reviewsBooks} provides generated samples illustrating how the na\"{i}ve model severely underrepresents book-related content, whereas \nameref*{alg:double} does not. 

\begin{table}[tb]\centering
\caption{Performance under different misspecification settings. $\downarrow$ = lower better, $\uparrow$ = higher better.}\label{tab:facesMetricsArcFace}
{\centering
\begin{tabular}{@{\hskip .25em}l@{\hskip .75em}l@{\hskip 1.5em}p{2.5em}p{2.5em}p{2.5em}p{2.5em}@{\hskip 1.5em}p{2.5em}p{2.5em}p{2.5em}l@{\hskip .25em}}
     & & \multicolumn{4}{c}{Diffusion model (\Cref*{sec:faces})$^1$} & \multicolumn{4}{c}{Language model (\Cref*{sec:reviews})$^2$} \\[.25em]
     & & FAD~$\downarrow$ & KAD~$\downarrow$ & Prec $\uparrow$ & Rec $\uparrow$ & PPL $\downarrow$ & Mve $\uparrow$ & FI $\downarrow$ & $W_1$ $\downarrow$ \\\midrule\midrule
&  Na\"{i}ve &                 1.00 &                 1.00 &                 0.38 &                 0.65 & 1.99 & 0.72 & 1.22 & 0.36  \\\midrule
\multirow{3}{4em}{\textit{Both right}} &  Plug-in & 0.87 & 0.68 & 0.34 & 0.76 & 1.98 & 0.81 & 0.91 & 0.17  \\
&  IPW & 0.88 & 0.71 & 0.36 & 0.73 & 1.98 & 0.82 & 0.90 & 0.05  \\
&  \nameref*{alg:double} & \textcolor{CB5blue}{\textbf{0.86}} & \textcolor{CB5blue}{\textbf{0.68}} & 0.35 & \textcolor{CB5blue}{0.74} & \textcolor{CB5blue}{\textbf{1.98}} & \textcolor{CB5blue}{\textbf{0.83}} & \textcolor{CB5blue}{\textbf{0.87}} & \textcolor{CB5blue}{0.17}  \\\midrule
\multirow{2}{4em}{\textit{Outcome wrong}} &  Plug-in & 1.90 & 2.17 & 0.34 & 0.34 & 1.98 & 0.83 & 0.87 & 0.17  \\
&  \nameref*{alg:double} & \textcolor{CB5blue}{\textbf{0.86}} & \textcolor{CB5blue}{\textbf{0.68}} & \textbf{0.35} & \textcolor{CB5blue}{\textbf{0.73}} & \textcolor{CB5blue}{\textbf{1.98}} & \textcolor{CB5blue}{0.82} & \textcolor{CB5blue}{0.88} & \textcolor{CB5blue}{\textbf{0.17}}  \\\midrule
\multirow{2}{4em}{\textit{Propensity wrong}} &  IPW & 0.93 & 0.71 & 0.34 & 0.74 & 1.98 & 0.80 & 0.96 & 0.17  \\
&  \nameref*{alg:double} & \textcolor{CB5blue}{\textbf{0.85}} & \textcolor{CB5blue}{\textbf{0.56}} & 0.32 & \textcolor{CB5blue}{\textbf{0.79}} & \textcolor{CB5blue}{\textbf{1.98}} & \textcolor{CB5blue}{\textbf{0.82}} & \textcolor{CB5blue}{\textbf{0.89}} & \textcolor{CB5blue}{\textbf{0.12}}  \\\midrule
\multirow{1}{5.5em}{\textit{Both wrong}} &  \nameref*{alg:double} & 1.01 & \textcolor{CB5blue}{0.79} & 0.32 & \textcolor{CB5blue}{0.77} & 1.99 & \textcolor{CB5blue}{0.82} & \textcolor{CB5blue}{0.88} & \textcolor{CB5blue}{0.17}  \\
\end{tabular}
}
\vspace{.25em}

\begin{flushleft}
Blue: \nameref*{alg:double}  \textcolor{CB5blue}{better than Na\"{i}ve}. Bold: \textbf{at least as good as best baseline} (in misspec. category).\\[.25em]
$^1\,$FAD/KAD=Fr\'{e}chet/kernel ArcFace distance (rescaled so Na\"{i}ve$\,=\,$1), Prec=precision, Rec=recall.\\[.25em]
$^2\,$PPL=perplexity$\div 10$, Mve=MAUVE, FI=frontier int$\div 10$, $W_1$=rating distribution Wass. error$\times 10$.
\end{flushleft}\vspace{-.5em}
\end{table}

\section{Discussion of implications and extensions}\label{sec:discussion}

\myparagraph{Equivalence between causal and missing data problems.} Generative modeling problems with outcomes missing at random can be tackled using \nameref*{alg:double}. To do this, $A$ can be taken to indicate whether the outcome is observed ($A=a^\star$) or missing \citep{ding2018causal}. The cross-fitted AIPW risk estimator \nameref*{alg:double} uses allows predictions of missing outcomes to be obtained from any algorithm, including a pretrained foundation model \citep{van2011cross,chernozhukov2018double}.

\myparagraph{Minimax optimality when nuisances are hard to estimate.} We showed \nameref*{alg:double} diffusion achieves the optimal rate under conditions, including that the nuisances are estimated well enough. When they cannot be estimated well---for example, because they are nonsmooth---the order $n^{-s/(2s+d)}$ minimax lower bound we gave may be loose.  It would be interesting to derive a sharper bound that reflects this difficulty, as was done by \cite{kennedy2024minimax} in a different problem.

\myparagraph{Reduced-entropy sampling for language models.} Language models often generate qualitatively better text when they oversample high-probability tokens \citep{holtzman2019curious}---e.g., via temperature scaling \citep{caccia2018language}, top-$k$ sampling \citep{fan2018hierarchical}, or nucleus sampling \citep{holtzman2019curious}. Each of these approaches redefines $\tau$ to yield a lower-entropy sampling scheme than the one from \Cref{ex:autoreg} \citep{nadeem2020systematic}. While reducing entropy can improve sample quality, it also makes the resulting $\tau(\theta_{\mathbb{P}})$ a transport map to an entropy-reduced variant of $\mathbb{P}$, rather than $\mathbb{P}$ itself. Regardless, \nameref*{alg:double} can be applied to estimate these transport maps by simply redefining $\tau$.

\myparagraph{Extensions to joint and conditional sampling.} \nameref*{alg:double} can be naturally extended to generate counterfactuals jointly with or conditionally on a subvector $V$ of the features $X$. Joint generation is straightforward: run \nameref*{alg:double} with the modified outcome $Y'=(V,Y)$. Conditional generation requires allowing the oracle loss $\ell$ to depend on $y'=(v,y)$, rather than just $y$. For example, some image restoration and text-to-image diffusion models use  $\ell(\theta,y')=\int 1\{t\in [\underline{t},\overline{t}]\} E[\|(\mu_t Y_0-Y_t)/\sigma_t^2-\theta((v,Y_t),t)\|^2\mymid Y_0=y]\, dt$  \citep{saharia2022palette,rombach2022high}. With such losses, \nameref*{alg:double} proceeds as in \Cref{alg:double}, but with $y$ and $\psi_n^j(u|x)$ on line~\ref{ln:riskMin} replaced by $y'$ and $(v,\psi_n^j(u|x))$. The analysis is nearly identical, yielding a generalization bound like \Cref{thm:excRiskOSL}.

\subsection*{Acknowledgments}
This work was completed during AL's sabbatical at Japan's Institute of Statistical Mathematics. AL is grateful to KF for his generous hospitality and valuable discussions during the visit, and to the University of Washington for granting the sabbatical that made this collaboration possible.

{
\setlength{\bibsep}{6pt plus 3pt minus 1.5pt}
\bibliographystyle{arxiv}
\bibliography{References}

\begin{thebibliography}{110}
\providecommand{\natexlab}[1]{#1}
\providecommand{\url}[1]{\texttt{#1}}
\expandafter\ifx\csname urlstyle\endcsname\relax
  \providecommand{\doi}[1]{doi: #1}\else
  \providecommand{\doi}{doi: \begingroup \urlstyle{rm}\Url}\fi

\bibitem[Akiba et~al.(2019)Akiba, Sano, Yanase, Ohta, and Koyama]{akiba2019optuna}
Takuya Akiba, Shotaro Sano, Toshihiko Yanase, Takeru Ohta, and Masanori Koyama.
\newblock Optuna: A next-generation hyperparameter optimization framework.
\newblock In \emph{Proceedings of the 25th ACM SIGKDD international conference on knowledge discovery \& data mining}, pp.\  2623--2631, 2019.

\bibitem[Anderson(1982)]{anderson1982reverse}
Brian~DO Anderson.
\newblock Reverse-time diffusion equation models.
\newblock \emph{Stochastic Processes and their Applications}, 12\penalty0 (3):\penalty0 313--326, 1982.

\bibitem[Bang \& Robins(2005)Bang and Robins]{bang2005doubly}
Heejung Bang and James~M Robins.
\newblock Doubly robust estimation in missing data and causal inference models.
\newblock \emph{Biometrics}, 61\penalty0 (4):\penalty0 962--973, 2005.

\bibitem[Bartlett \& Mendelson(2006)Bartlett and Mendelson]{bartlett2006local}
Peter~L Bartlett and Shahar Mendelson.
\newblock Local rademacher complexities and empirical minimization.
\newblock \emph{Annals of Statistics}, 34, 2006.

\bibitem[Batzolis et~al.(2021)Batzolis, Stanczuk, Sch{\"o}nlieb, and Etmann]{batzolis2021conditional}
Georgios Batzolis, Jan Stanczuk, Carola-Bibiane Sch{\"o}nlieb, and Christian Etmann.
\newblock Conditional image generation with score-based diffusion models.
\newblock \emph{arXiv preprint arXiv:2111.13606}, 2021.

\bibitem[Benton et~al.(2023)Benton, Deligiannidis, and Doucet]{benton2023error}
Joe Benton, George Deligiannidis, and Arnaud Doucet.
\newblock Error bounds for flow matching methods.
\newblock \emph{arXiv preprint arXiv:2305.16860}, 2023.

\bibitem[Bi{\'n}kowski et~al.(2018)Bi{\'n}kowski, Sutherland, Arbel, and Gretton]{binkowski2018demystifying}
Miko{\l}aj Bi{\'n}kowski, Danica~J Sutherland, Michael Arbel, and Arthur Gretton.
\newblock Demystifying {MMD} {GAN}s.
\newblock \emph{arXiv preprint arXiv:1801.01401}, 2018.

\bibitem[Bishop \& Nasrabadi(2006)Bishop and Nasrabadi]{bishop2006pattern}
Christopher~M Bishop and Nasser~M Nasrabadi.
\newblock \emph{Pattern recognition and machine learning}, volume~4.
\newblock Springer, 2006.

\bibitem[Bonvini \& Kennedy(2022)Bonvini and Kennedy]{bonvini2022fast}
Matteo Bonvini and Edward~H Kennedy.
\newblock Fast convergence rates for dose-response estimation.
\newblock \emph{arXiv preprint arXiv:2207.11825}, 2022.

\bibitem[B{\"u}hlmann \& Van De~Geer(2011)B{\"u}hlmann and Van De~Geer]{buhlmann2011statistics}
Peter B{\"u}hlmann and Sara Van De~Geer.
\newblock \emph{Statistics for high-dimensional data: methods, theory and applications}.
\newblock Springer Science \& Business Media, 2011.

\bibitem[Caccia et~al.(2018)Caccia, Caccia, Fedus, Larochelle, Pineau, and Charlin]{caccia2018language}
Massimo Caccia, Lucas Caccia, William Fedus, Hugo Larochelle, Joelle Pineau, and Laurent Charlin.
\newblock Language {GAN}s falling short.
\newblock \emph{arXiv preprint arXiv:1811.02549}, 2018.

\bibitem[Chan et~al.(2010)Chan, Ge, Gershony, Hesterberg, and Lambert]{chan2010evaluating}
David Chan, Rong Ge, Ori Gershony, Tim Hesterberg, and Diane Lambert.
\newblock Evaluating online ad campaigns in a pipeline: causal models at scale.
\newblock In \emph{Proceedings of the 16th ACM SIGKDD international conference on Knowledge discovery and data mining}, pp.\  7--16, 2010.

\bibitem[Chao et~al.()Chao, Bl{\"o}baum, Patel, and Kasiviswanathan]{chaomodeling}
Patrick Chao, Patrick Bl{\"o}baum, Sapan~Kirit Patel, and Shiva Kasiviswanathan.
\newblock Modeling causal mechanisms with diffusion models for interventional and counterfactual queries.
\newblock \emph{Transactions on Machine Learning Research}.

\bibitem[Chen et~al.(2023)Chen, Huang, Zhao, and Wang]{chen2023score}
Minshuo Chen, Kaixuan Huang, Tuo Zhao, and Mengdi Wang.
\newblock Score approximation, estimation and distribution recovery of diffusion models on low-dimensional data.
\newblock In \emph{International Conference on Machine Learning}, pp.\  4672--4712. {PMLR}, 2023.

\bibitem[Chernozhukov et~al.(2018)Chernozhukov, Chetverikov, Demirer, Duflo, Hansen, Newey, and Robins]{chernozhukov2018double}
Victor Chernozhukov, Denis Chetverikov, Mert Demirer, Esther Duflo, Christian Hansen, Whitney Newey, and James Robins.
\newblock Double/debiased machine learning for treatment and structural parameters, 2018.

\bibitem[Chernozhukov et~al.(2021)Chernozhukov, Newey, Quintas-Martinez, and Syrgkanis]{chernozhukov2021automatic}
Victor Chernozhukov, Whitney~K Newey, Victor Quintas-Martinez, and Vasilis Syrgkanis.
\newblock Automatic debiased machine learning via riesz regression.
\newblock \emph{arXiv preprint arXiv:2104.14737}, 2021.

\bibitem[Cover(1999)]{cover1999elements}
Thomas~M Cover.
\newblock \emph{Elements of information theory}.
\newblock John Wiley \& Sons, 1999.

\bibitem[De~Bortoli(2022)]{de2022convergence}
Valentin De~Bortoli.
\newblock Convergence of denoising diffusion models under the manifold hypothesis.
\newblock \emph{arXiv preprint arXiv:2208.05314}, 2022.

\bibitem[Deng et~al.(2019)Deng, Guo, Xue, and Zafeiriou]{deng2019arcface}
Jiankang Deng, Jia Guo, Niannan Xue, and Stefanos Zafeiriou.
\newblock Arcface: Additive angular margin loss for deep face recognition.
\newblock In \emph{Proceedings of the IEEE/CVF conference on computer vision and pattern recognition}, pp.\  4690--4699, 2019.

\bibitem[Ding \& Li(2018)Ding and Li]{ding2018causal}
Peng Ding and Fan Li.
\newblock {Causal Inference: A Missing Data Perspective}.
\newblock \emph{Statistical Science}, 33\penalty0 (2):\penalty0 214 -- 237, 2018.
\newblock \doi{10.1214/18-STS645}.
\newblock URL \url{https://doi.org/10.1214/18-STS645}.

\bibitem[Dubey et~al.(2024)Dubey, Jauhri, Pandey, Kadian, Al-Dahle, Letman, Mathur, Schelten, Yang, Fan, et~al.]{dubey2024llama}
Abhimanyu Dubey, Abhinav Jauhri, Abhinav Pandey, Abhishek Kadian, Ahmad Al-Dahle, Aiesha Letman, Akhil Mathur, Alan Schelten, Amy Yang, Angela Fan, et~al.
\newblock The llama 3 herd of models.
\newblock \emph{arXiv e-prints}, pp.\  arXiv--2407, 2024.

\bibitem[Durrett(2019)]{durrett2019probability}
Rick Durrett.
\newblock \emph{Probability: theory and examples}, volume~49.
\newblock Cambridge university press, 2019.

\bibitem[Fan et~al.(2018)Fan, Lewis, and Dauphin]{fan2018hierarchical}
Angela Fan, Mike Lewis, and Yann Dauphin.
\newblock Hierarchical neural story generation.
\newblock \emph{arXiv preprint arXiv:1805.04833}, 2018.

\bibitem[Fawkes et~al.(2022)Fawkes, Hu, Evans, and Sejdinovic]{fawkes2022doubly}
Jake Fawkes, Robert Hu, Robin~J Evans, and Dino Sejdinovic.
\newblock Doubly robust kernel statistics for testing distributional treatment effects.
\newblock \emph{arXiv preprint arXiv:2212.04922}, 2022.

\bibitem[Foster \& Syrgkanis(2023)Foster and Syrgkanis]{foster2023orthogonal}
Dylan~J Foster and Vasilis Syrgkanis.
\newblock Orthogonal statistical learning.
\newblock \emph{The Annals of Statistics}, 51\penalty0 (3):\penalty0 879--908, 2023.

\bibitem[Fukumizu et~al.(2024)Fukumizu, Suzuki, Isobe, Oko, and Koyama]{fukumizu2024flow}
Kenji Fukumizu, Taiji Suzuki, Noboru Isobe, Kazusato Oko, and Masanori Koyama.
\newblock Flow matching achieves minimax optimal convergence.
\newblock \emph{arXiv preprint arXiv:2405.20879}, 2024.

\bibitem[Graves(2013)]{graves2013generating}
Alex Graves.
\newblock Generating sequences with recurrent neural networks.
\newblock \emph{arXiv preprint arXiv:1308.0850}, 2013.

\bibitem[Hamilton et~al.(2018)Hamilton, Raghunathan, Matiach, Schonhoffer, Raman, Barzilay, Rajendran, Banda, Hong, Knoertzer, et~al.]{hamilton2018mmlspark}
Mark Hamilton, Sudarshan Raghunathan, Ilya Matiach, Andrew Schonhoffer, Anand Raman, Eli Barzilay, Karthik Rajendran, Dalitso Banda, Casey~Jisoo Hong, Manon Knoertzer, et~al.
\newblock Mmlspark: Unifying machine learning ecosystems at massive scales.
\newblock \emph{arXiv preprint arXiv:1810.08744}, 2018.

\bibitem[Hern\'{a}n \& Robins(2024)Hern\'{a}n and Robins]{hernan2024causal}
Miguel~A. Hern\'{a}n and James~M. Robins.
\newblock \emph{Causal Inference: What If}.
\newblock Chapman \& Hall/CRC Monographs on Statistics \& Applied Probab. CRC Press, 2024.
\newblock ISBN 9781420076165.

\bibitem[Heusel et~al.(2017)Heusel, Ramsauer, Unterthiner, Nessler, and Hochreiter]{heusel2017gans}
Martin Heusel, Hubert Ramsauer, Thomas Unterthiner, Bernhard Nessler, and Sepp Hochreiter.
\newblock Gans trained by a two time-scale update rule converge to a local nash equilibrium.
\newblock \emph{Advances in neural information processing systems}, 30, 2017.

\bibitem[Ho \& Salimans(2022)Ho and Salimans]{ho2022classifier}
Jonathan Ho and Tim Salimans.
\newblock Classifier-free diffusion guidance.
\newblock \emph{arXiv preprint arXiv:2207.12598}, 2022.

\bibitem[Holk et~al.(2024)Holk, Strauch, and Trottner]{holk2024statistical}
Asbj{\o}rn Holk, Claudia Strauch, and Lukas Trottner.
\newblock Statistical guarantees for denoising reflected diffusion models.
\newblock \emph{arXiv preprint arXiv:2411.01563}, 2024.

\bibitem[Holtzman et~al.(2019)Holtzman, Buys, Du, Forbes, and Choi]{holtzman2019curious}
Ari Holtzman, Jan Buys, Li~Du, Maxwell Forbes, and Yejin Choi.
\newblock The curious case of neural text degeneration.
\newblock \emph{arXiv preprint arXiv:1904.09751}, 2019.

\bibitem[Hou et~al.(2024)Hou, Li, He, Yan, Chen, and McAuley]{hou2024bridging}
Yupeng Hou, Jiacheng Li, Zhankui He, An~Yan, Xiusi Chen, and Julian McAuley.
\newblock Bridging language and items for retrieval and recommendation.
\newblock \emph{arXiv preprint arXiv:2403.03952}, 2024.

\bibitem[Hu et~al.(2022)Hu, Shen, Wallis, Allen-Zhu, Li, Wang, Wang, Chen, et~al.]{hu2022lora}
Edward~J Hu, Yelong Shen, Phillip Wallis, Zeyuan Allen-Zhu, Yuanzhi Li, Shean Wang, Lu~Wang, Weizhu Chen, et~al.
\newblock Lora: Low-rank adaptation of large language models.
\newblock \emph{ICLR}, 1\penalty0 (2):\penalty0 3, 2022.

\bibitem[{Hugging Face}(2025)]{huggingface2025}
{Hugging Face}.
\newblock Train a diffusion model, 2025.
\newblock URL \url{https://huggingface.co/docs/diffusers/en/tutorials/basic_training}.
\newblock Accessed: 2025-05-23.

\bibitem[Javaloy et~al.(2024)Javaloy, S{\'a}nchez-Mart{\'\i}n, and Valera]{javaloy2024causal}
Adri{\'a}n Javaloy, Pablo S{\'a}nchez-Mart{\'\i}n, and Isabel Valera.
\newblock Causal normalizing flows: from theory to practice.
\newblock \emph{Advances in Neural Information Processing Systems}, 36, 2024.

\bibitem[Jelinek et~al.(1977)Jelinek, Mercer, Bahl, and Baker]{jelinek1977perplexity}
Fred Jelinek, Robert~L Mercer, Lalit~R Bahl, and James~K Baker.
\newblock Perplexity—a measure of the difficulty of speech recognition tasks.
\newblock \emph{The Journal of the Acoustical Society of America}, 62\penalty0 (S1):\penalty0 S63--S63, 1977.

\bibitem[Kallenberg(2021)]{kallenberg2021foundations}
O.~Kallenberg.
\newblock \emph{Foundations of Modern Probability}.
\newblock Probability theory and stochastic modelling. Springer, 2021.
\newblock ISBN 9783030618728.
\newblock URL \url{https://books.google.co.jp/books?id=6hJezgEACAAJ}.

\bibitem[Karatzas \& Shreve(1991)Karatzas and Shreve]{karatzas1991brownian}
Ioannis Karatzas and Steven Shreve.
\newblock \emph{Brownian motion and stochastic calculus}, volume 113.
\newblock Springer Science \& Business Media, 1991.

\bibitem[Karimi et~al.(2020)Karimi, Von~K{\"u}gelgen, Sch{\"o}lkopf, and Valera]{karimi2020algorithmic}
Amir-Hossein Karimi, Julius Von~K{\"u}gelgen, Bernhard Sch{\"o}lkopf, and Isabel Valera.
\newblock Algorithmic recourse under imperfect causal knowledge: a probabilistic approach.
\newblock \emph{Advances in neural information processing systems}, 33:\penalty0 265--277, 2020.

\bibitem[Ke et~al.(2017)Ke, Meng, Finley, Wang, Chen, Ma, Ye, and Liu]{ke2017lightgbm}
Guolin Ke, Qi~Meng, Thomas Finley, Taifeng Wang, Wei Chen, Weidong Ma, Qiwei Ye, and Tie-Yan Liu.
\newblock Lightgbm: A highly efficient gradient boosting decision tree.
\newblock \emph{Advances in neural information processing systems}, 30, 2017.

\bibitem[Kennedy(2023)]{kennedy2023towards}
Edward~H Kennedy.
\newblock Towards optimal doubly robust estimation of heterogeneous causal effects.
\newblock \emph{Electronic Journal of Statistics}, 17\penalty0 (2):\penalty0 3008--3049, 2023.

\bibitem[Kennedy et~al.(2023)Kennedy, Balakrishnan, and Wasserman]{kennedy2023semiparametric}
Edward~H Kennedy, Sivaraman Balakrishnan, and LA~Wasserman.
\newblock Semiparametric counterfactual density estimation.
\newblock \emph{Biometrika}, 110\penalty0 (4):\penalty0 875--896, 2023.

\bibitem[Kennedy et~al.(2024)Kennedy, Balakrishnan, Robins, and Wasserman]{kennedy2024minimax}
Edward~H Kennedy, Sivaraman Balakrishnan, James~M Robins, and Larry Wasserman.
\newblock Minimax rates for heterogeneous causal effect estimation.
\newblock \emph{The Annals of Statistics}, 52\penalty0 (2):\penalty0 793--816, 2024.

\bibitem[Khemakhem et~al.(2021)Khemakhem, Monti, Leech, and Hyvarinen]{khemakhem2021causal}
Ilyes Khemakhem, Ricardo Monti, Robert Leech, and Aapo Hyvarinen.
\newblock Causal autoregressive flows.
\newblock In \emph{International conference on artificial intelligence and statistics}, pp.\  3520--3528. {PMLR}, 2021.

\bibitem[Kingma(2013)]{kingma2013auto}
Diederik~P Kingma.
\newblock Auto-encoding variational {B}ayes.
\newblock \emph{arXiv preprint arXiv:1312.6114}, 2013.

\bibitem[Kirillov et~al.(2023)Kirillov, Mintun, Ravi, Mao, Rolland, Gustafson, Xiao, Whitehead, Berg, Lo, Doll{\'a}r, and Girshick]{kirillov2023segany}
Alexander Kirillov, Eric Mintun, Nikhila Ravi, Hanzi Mao, Chloe Rolland, Laura Gustafson, Tete Xiao, Spencer Whitehead, Alexander~C. Berg, Wan-Yen Lo, Piotr Doll{\'a}r, and Ross Girshick.
\newblock Segment anything.
\newblock \emph{arXiv:2304.02643}, 2023.

\bibitem[Kocaoglu et~al.(2017)Kocaoglu, Snyder, Dimakis, and Vishwanath]{kocaoglu2017causalgan}
Murat Kocaoglu, Christopher Snyder, Alexandros~G Dimakis, and Sriram Vishwanath.
\newblock Causalgan: Learning causal implicit generative models with adversarial training.
\newblock \emph{arXiv preprint arXiv:1709.02023}, 2017.

\bibitem[Komanduri et~al.(2023)Komanduri, Wu, Wu, and Chen]{komanduri2023identifiable}
Aneesh Komanduri, Xintao Wu, Yongkai Wu, and Feng Chen.
\newblock From identifiable causal representations to controllable counterfactual generation: A survey on causal generative modeling.
\newblock \emph{arXiv preprint arXiv:2310.11011}, 2023.

\bibitem[Kynk{\"a}{\"a}nniemi et~al.(2019)Kynk{\"a}{\"a}nniemi, Karras, Laine, Lehtinen, and Aila]{kynkaanniemi2019improved}
Tuomas Kynk{\"a}{\"a}nniemi, Tero Karras, Samuli Laine, Jaakko Lehtinen, and Timo Aila.
\newblock Improved precision and recall metric for assessing generative models.
\newblock \emph{Advances in neural information processing systems}, 32, 2019.

\bibitem[Laan et~al.(2006)Laan, Dudoit, and Vaart]{laan2006cross}
Mark J van~der Laan, Sandrine Dudoit, and Aad W van~der Vaart.
\newblock The cross-validated adaptive epsilon-net estimator.
\newblock \emph{Statistics \& Decisions}, 24\penalty0 (3):\penalty0 373--395, 2006.

\bibitem[Lee et~al.(2022)Lee, Lu, and Tan]{lee2022convergence}
Holden Lee, Jianfeng Lu, and Yixin Tan.
\newblock Convergence for score-based generative modeling with polynomial complexity.
\newblock \emph{Advances in Neural Information Processing Systems}, 35:\penalty0 22870--22882, 2022.

\bibitem[Lipman et~al.(2022)Lipman, Chen, Ben-Hamu, Nickel, and Le]{lipman2022flow}
Yaron Lipman, Ricky~TQ Chen, Heli Ben-Hamu, Maximilian Nickel, and Matt Le.
\newblock Flow matching for generative modeling.
\newblock \emph{arXiv preprint arXiv:2210.02747}, 2022.

\bibitem[Liu et~al.(2022)Liu, Gong, and Liu]{liu2022flow}
Xingchao Liu, Chengyue Gong, and Qiang Liu.
\newblock Flow straight and fast: Learning to generate and transfer data with rectified flow.
\newblock \emph{arXiv preprint arXiv:2209.03003}, 2022.

\bibitem[Liu et~al.(2015)Liu, Luo, Wang, and Tang]{liu2015faceattributes}
Ziwei Liu, Ping Luo, Xiaogang Wang, and Xiaoou Tang.
\newblock Deep learning face attributes in the wild.
\newblock In \emph{Proceedings of International Conference on Computer Vision (ICCV)}, December 2015.

\bibitem[Loshchilov \& Hutter(2017)Loshchilov and Hutter]{loshchilov2017decoupled}
Ilya Loshchilov and Frank Hutter.
\newblock Decoupled weight decay regularization.
\newblock \emph{arXiv preprint arXiv:1711.05101}, 2017.

\bibitem[Lotfi et~al.(2023)Lotfi, Finzi, Kuang, Rudner, Goldblum, and Wilson]{lotfi2023non}
Sanae Lotfi, Marc Finzi, Yilun Kuang, Tim~GJ Rudner, Micah Goldblum, and Andrew~Gordon Wilson.
\newblock Non-vacuous generalization bounds for large language models.
\newblock \emph{arXiv preprint arXiv:2312.17173}, 2023.

\bibitem[Luedtke \& Chung(2024)Luedtke and Chung]{luedtke2024one}
Alex Luedtke and Incheoul Chung.
\newblock One-step estimation of differentiable hilbert-valued parameters.
\newblock \emph{The Annals of Statistics}, 52\penalty0 (4):\penalty0 1534--1563, 2024.

\bibitem[Luedtke et~al.(2017)Luedtke, Sofrygin, van~der Laan, and Carone]{luedtke2017sequential}
Alexander~R Luedtke, Oleg Sofrygin, Mark~J van~der Laan, and Marco Carone.
\newblock Sequential double robustness in right-censored longitudinal models.
\newblock \emph{arXiv preprint arXiv:1705.02459}, 2017.

\bibitem[Lugaresi et~al.(2019)Lugaresi, Tang, Nash, McClanahan, Uboweja, Hays, Zhang, Chang, Yong, Lee, et~al.]{lugaresi2019mediapipe}
Camillo Lugaresi, Jiuqiang Tang, Hadon Nash, Chris McClanahan, Esha Uboweja, Michael Hays, Fan Zhang, Chuo-Ling Chang, Ming~Guang Yong, Juhyun Lee, et~al.
\newblock Mediapipe: A framework for building perception pipelines.
\newblock \emph{arXiv preprint arXiv:1906.08172}, 2019.

\bibitem[Mangrulkar et~al.(2022)Mangrulkar, Gugger, Debut, Belkada, Paul, and Bossan]{peft}
Sourab Mangrulkar, Sylvain Gugger, Lysandre Debut, Younes Belkada, Sayak Paul, and Benjamin Bossan.
\newblock {PEFT}: State-of-the-art parameter-efficient fine-tuning methods.
\newblock \url{https://github.com/huggingface/peft}, 2022.

\bibitem[Martinez~Taboada et~al.(2023)Martinez~Taboada, Ramdas, and Kennedy]{martinez2023efficient}
Diego Martinez~Taboada, Aaditya Ramdas, and Edward Kennedy.
\newblock An efficient doubly-robust test for the kernel treatment effect.
\newblock \emph{Advances in Neural Information Processing Systems}, 36:\penalty0 59924--59952, 2023.

\bibitem[Mey(2020)]{mey2020note}
Alexander Mey.
\newblock A note on high-probability versus in-expectation guarantees of generalization bounds in machine learning.
\newblock \emph{arXiv preprint arXiv:2010.02576}, 2020.

\bibitem[Morzywolek et~al.(2023)Morzywolek, Decruyenaere, and Vansteelandt]{morzywolek2023general}
Pawel Morzywolek, Johan Decruyenaere, and Stijn Vansteelandt.
\newblock On a general class of orthogonal learners for the estimation of heterogeneous treatment effects.
\newblock \emph{arXiv preprint arXiv:2303.12687}, 2023.

\bibitem[Nadeem et~al.(2020)Nadeem, He, Cho, and Glass]{nadeem2020systematic}
Moin Nadeem, Tianxing He, Kyunghyun Cho, and James Glass.
\newblock A systematic characterization of sampling algorithms for open-ended language generation.
\newblock \emph{arXiv preprint arXiv:2009.07243}, 2020.

\bibitem[Nie \& Wager(2021)Nie and Wager]{nie2021quasi}
Xinkun Nie and Stefan Wager.
\newblock Quasi-oracle estimation of heterogeneous treatment effects.
\newblock \emph{Biometrika}, 108\penalty0 (2):\penalty0 299--319, 2021.

\bibitem[Obukhov et~al.(2020)Obukhov, Seitzer, Wu, Zhydenko, Kyl, and Lin]{obukhov2020torchfidelity}
Anton Obukhov, Maximilian Seitzer, Po-Wei Wu, Semen Zhydenko, Jonathan Kyl, and Elvis Yu-Jing Lin.
\newblock High-fidelity performance metrics for generative models in pytorch, 2020.
\newblock URL \url{https://github.com/toshas/torch-fidelity}.
\newblock Version: 0.3.0, DOI: 10.5281/zenodo.4957738.

\bibitem[Oko et~al.(2023)Oko, Akiyama, and Suzuki]{oko2023diffusion}
Kazusato Oko, Shunta Akiyama, and Taiji Suzuki.
\newblock Diffusion models are minimax optimal distribution estimators.
\newblock In \emph{International Conference on Machine Learning}, pp.\  26517--26582. {PMLR}, 2023.

\bibitem[Oksendal(2013)]{oksendal2013stochastic}
Bernt Oksendal.
\newblock \emph{Stochastic differential equations: an introduction with applications}.
\newblock Springer Science \& Business Media, 2013.

\bibitem[Paszke et~al.(2019)Paszke, Gross, Massa, Lerer, Bradbury, Chanan, Killeen, Lin, Gimelshein, Antiga, Desmaison, Kopf, Yang, DeVito, Raison, Tejani, Chilamkurthy, Steiner, Fang, Bai, and Chintala]{NEURIPS2019_9015}
Adam Paszke, Sam Gross, Francisco Massa, Adam Lerer, James Bradbury, Gregory Chanan, Trevor Killeen, Zeming Lin, Natalia Gimelshein, Luca Antiga, Alban Desmaison, Andreas Kopf, Edward Yang, Zachary DeVito, Martin Raison, Alykhan Tejani, Sasank Chilamkurthy, Benoit Steiner, Lu~Fang, Junjie Bai, and Soumith Chintala.
\newblock Pytorch: An imperative style, high-performance deep learning library.
\newblock In \emph{Advances in Neural Information Processing Systems 32}, pp.\  8024--8035. Curran Associates, Inc., 2019.

\bibitem[Pawlowski et~al.(2020)Pawlowski, Coelho~de Castro, and Glocker]{pawlowski2020deep}
Nick Pawlowski, Daniel Coelho~de Castro, and Ben Glocker.
\newblock Deep structural causal models for tractable counterfactual inference.
\newblock \emph{Advances in neural information processing systems}, 33:\penalty0 857--869, 2020.

\bibitem[Pearl(2009)]{pearl2009causality}
Judea Pearl.
\newblock \emph{Causality}.
\newblock Cambridge university press, 2009.

\bibitem[Pillutla et~al.(2021)Pillutla, Swayamdipta, Zellers, Thickstun, Welleck, Choi, and Harchaoui]{pillutla2021mauve}
Krishna Pillutla, Swabha Swayamdipta, Rowan Zellers, John Thickstun, Sean Welleck, Yejin Choi, and Zaid Harchaoui.
\newblock Mauve: Measuring the gap between neural text and human text using divergence frontiers.
\newblock \emph{Advances in Neural Information Processing Systems}, 34:\penalty0 4816--4828, 2021.

\bibitem[Polyak \& Juditsky(1992)Polyak and Juditsky]{polyak1992acceleration}
Boris~T Polyak and Anatoli~B Juditsky.
\newblock Acceleration of stochastic approximation by averaging.
\newblock \emph{SIAM journal on control and optimization}, 30\penalty0 (4):\penalty0 838--855, 1992.

\bibitem[Radford(2018)]{radford2018improving}
Alec Radford.
\newblock Improving language understanding by generative pre-training.
\newblock 2018.

\bibitem[Radford et~al.(2019)Radford, Wu, Child, Luan, Amodei, Sutskever, et~al.]{radford2019language}
Alec Radford, Jeffrey Wu, Rewon Child, David Luan, Dario Amodei, Ilya Sutskever, et~al.
\newblock Language models are unsupervised multitask learners.
\newblock \emph{OpenAI blog}, 1\penalty0 (8):\penalty0 9, 2019.

\bibitem[Ramesh et~al.(2021)Ramesh, Pavlov, Goh, Gray, Voss, Radford, Chen, and Sutskever]{ramesh2021zero}
Aditya Ramesh, Mikhail Pavlov, Gabriel Goh, Scott Gray, Chelsea Voss, Alec Radford, Mark Chen, and Ilya Sutskever.
\newblock Zero-shot text-to-image generation.
\newblock In \emph{International conference on machine learning}, pp.\  8821--8831. {PMLR}, 2021.

\bibitem[Robins(1986)]{robins1986new}
James Robins.
\newblock A new approach to causal inference in mortality studies with a sustained exposure period--application to control of the healthy worker survivor effect.
\newblock \emph{Mathematical modelling}, 7\penalty0 (9-12):\penalty0 1393--1512, 1986.

\bibitem[Robins et~al.(1994)Robins, Rotnitzky, and Zhao]{robins1994estimation}
James~M Robins, Andrea Rotnitzky, and Lue~Ping Zhao.
\newblock Estimation of regression coefficients when some regressors are not always observed.
\newblock \emph{Journal of the American statistical Association}, 89\penalty0 (427):\penalty0 846--866, 1994.

\bibitem[Rombach et~al.(2021)Rombach, Blattmann, Lorenz, Esser, and Ommer]{rombach2021highresolution}
Robin Rombach, Andreas Blattmann, Dominik Lorenz, Patrick Esser, and Bj\"{o}rn Ommer.
\newblock High-resolution image synthesis with latent diffusion models, 2021.

\bibitem[Rombach et~al.(2022)Rombach, Blattmann, Lorenz, Esser, and Ommer]{rombach2022high}
Robin Rombach, Andreas Blattmann, Dominik Lorenz, Patrick Esser, and Bj{\"o}rn Ommer.
\newblock High-resolution image synthesis with latent diffusion models.
\newblock In \emph{Proceedings of the IEEE/CVF conference on computer vision and pattern recognition}, pp.\  10684--10695, 2022.

\bibitem[Ronneberger et~al.(2015)Ronneberger, Fischer, and Brox]{ronneberger2015u}
Olaf Ronneberger, Philipp Fischer, and Thomas Brox.
\newblock U-net: Convolutional networks for biomedical image segmentation.
\newblock In \emph{Medical image computing and computer-assisted intervention--MICCAI 2015: 18th international conference, Munich, Germany, October 5-9, 2015, proceedings, part III 18}, pp.\  234--241. Springer, 2015.

\bibitem[Rotnitzky et~al.(2006)Rotnitzky, Faraggi, and Schisterman]{rotnitzky2006doubly}
Andrea Rotnitzky, David Faraggi, and Enrique Schisterman.
\newblock Doubly robust estimation of the area under the receiver-operating characteristic curve in the presence of verification bias.
\newblock \emph{Journal of the American Statistical Association}, 101\penalty0 (475):\penalty0 1276--1288, 2006.

\bibitem[Rotnitzky et~al.(2017)Rotnitzky, Robins, and Babino]{rotnitzky2017multiply}
Andrea Rotnitzky, James Robins, and Lucia Babino.
\newblock On the multiply robust estimation of the mean of the g-functional.
\newblock \emph{arXiv preprint arXiv:1705.08582}, 2017.

\bibitem[Rubin \& van~der Laan(2007)Rubin and van~der Laan]{rubin2007doubly}
Daniel Rubin and Mark~J van~der Laan.
\newblock A doubly robust censoring unbiased transformation.
\newblock \emph{The international journal of biostatistics}, 3\penalty0 (1), 2007.

\bibitem[Saharia et~al.(2022)Saharia, Chan, Chang, Lee, Ho, Salimans, Fleet, and Norouzi]{saharia2022palette}
Chitwan Saharia, William Chan, Huiwen Chang, Chris Lee, Jonathan Ho, Tim Salimans, David Fleet, and Mohammad Norouzi.
\newblock Palette: Image-to-image diffusion models.
\newblock In \emph{ACM SIGGRAPH 2022 conference proceedings}, pp.\  1--10, 2022.

\bibitem[Sanchez \& Tsaftaris(2022)Sanchez and Tsaftaris]{sanchez2022diffusion}
Pedro Sanchez and Sotirios~A Tsaftaris.
\newblock Diffusion causal models for counterfactual estimation.
\newblock \emph{arXiv preprint arXiv:2202.10166}, 2022.

\bibitem[Sanchez-Martin et~al.(2021)Sanchez-Martin, Rateike, and Valera]{sanchez2021vaca}
Pablo Sanchez-Martin, Miriam Rateike, and Isabel Valera.
\newblock Vaca: Design of variational graph autoencoders for interventional and counterfactual queries.
\newblock \emph{arXiv preprint arXiv:2110.14690}, 2021.

\bibitem[Scharfstein et~al.(1999)Scharfstein, Rotnitzky, and Robins]{scharfstein1999adjusting}
Daniel~O Scharfstein, Andrea Rotnitzky, and James~M Robins.
\newblock Adjusting for nonignorable drop-out using semiparametric nonresponse models.
\newblock \emph{Journal of the American Statistical Association}, 94\penalty0 (448):\penalty0 1096--1120, 1999.

\bibitem[Sohl-Dickstein et~al.(2015)Sohl-Dickstein, Weiss, Maheswaranathan, and Ganguli]{sohl2015deep}
Jascha Sohl-Dickstein, Eric Weiss, Niru Maheswaranathan, and Surya Ganguli.
\newblock Deep unsupervised learning using nonequilibrium thermodynamics.
\newblock In \emph{International conference on machine learning}, pp.\  2256--2265. {PMLR}, 2015.

\bibitem[Sohn et~al.(2015)Sohn, Lee, and Yan]{sohn2015learning}
Kihyuk Sohn, Honglak Lee, and Xinchen Yan.
\newblock Learning structured output representation using deep conditional generative models.
\newblock \emph{Advances in neural information processing systems}, 28, 2015.

\bibitem[Song et~al.(2020)Song, Sohl-Dickstein, Kingma, Kumar, Ermon, and Poole]{song2020score}
Yang Song, Jascha Sohl-Dickstein, Diederik~P Kingma, Abhishek Kumar, Stefano Ermon, and Ben Poole.
\newblock Score-based generative modeling through stochastic differential equations.
\newblock \emph{arXiv preprint arXiv:2011.13456}, 2020.

\bibitem[Suzuki(2018)]{suzuki2018adaptivity}
Taiji Suzuki.
\newblock Adaptivity of deep relu network for learning in besov and mixed smooth besov spaces: optimal rate and curse of dimensionality.
\newblock \emph{arXiv preprint arXiv:1810.08033}, 2018.

\bibitem[Szegedy et~al.(2016)Szegedy, Vanhoucke, Ioffe, Shlens, and Wojna]{szegedy2016rethinking}
Christian Szegedy, Vincent Vanhoucke, Sergey Ioffe, Jon Shlens, and Zbigniew Wojna.
\newblock Rethinking the inception architecture for computer vision.
\newblock In \emph{Proceedings of the IEEE conference on computer vision and pattern recognition}, pp.\  2818--2826, 2016.

\bibitem[Touvron et~al.(2023)Touvron, Lavril, Izacard, Martinet, Lachaux, Lacroix, Rozi{\`e}re, Goyal, Hambro, Azhar, et~al.]{touvron2023llama}
Hugo Touvron, Thibaut Lavril, Gautier Izacard, Xavier Martinet, Marie-Anne Lachaux, Timoth{\'e}e Lacroix, Baptiste Rozi{\`e}re, Naman Goyal, Eric Hambro, Faisal Azhar, et~al.
\newblock Llama: Open and efficient foundation language models.
\newblock \emph{arXiv preprint arXiv:2302.13971}, 2023.

\bibitem[Tsybakov(2009)]{tsybakov2009nonparametric}
Alexandre~B Tsybakov.
\newblock Nonparametric estimators.
\newblock \emph{Introduction to Nonparametric Estimation}, 2009.

\bibitem[van~der Laan et~al.(2024)van~der Laan, Carone, and Luedtke]{van2024combining}
Lars van~der Laan, Marco Carone, and Alex Luedtke.
\newblock Combining t-learning and dr-learning: a framework for oracle-efficient estimation of causal contrasts.
\newblock \emph{arXiv preprint arXiv:2402.01972}, 2024.

\bibitem[van~der Laan et~al.(2003)van~der Laan, Robins, van~der Laan, and Robins]{van2003unified}
Mark~J van~der Laan, James~M Robins, Mark~J van~der Laan, and James~M Robins.
\newblock Unified approach for causal inference and censored data.
\newblock \emph{Unified Methods for Censored Longitudinal Data and Causality}, pp.\  311--370, 2003.

\bibitem[van~der Laan et~al.(2011)van~der Laan, Rose, Zheng, and van~der Laan]{van2011cross}
Mark~J van~der Laan, Sherri Rose, Wenjing Zheng, and Mark~J van~der Laan.
\newblock Cross-validated targeted minimum-loss-based estimation.
\newblock \emph{Targeted learning: causal inference for observational and experimental data}, pp.\  459--474, 2011.

\bibitem[Van Der~Vaart \& Wellner(2011)Van Der~Vaart and Wellner]{van2011local}
Aad Van Der~Vaart and Jon~A Wellner.
\newblock A local maximal inequality under uniform entropy.
\newblock \emph{Electronic Journal of Statistics}, 5\penalty0 (2011):\penalty0 192, 2011.

\bibitem[van~der Vaart \& Wellner(2023)van~der Vaart and Wellner]{vaart2023empirical}
AW~van~der Vaart and Jon~A Wellner.
\newblock Empirical processes.
\newblock In \emph{Weak Convergence and Empirical Processes: With Applications to Statistics}, pp.\  127--384. Springer, 2023.

\bibitem[Vaswani(2017)]{vaswani2017attention}
A~Vaswani.
\newblock Attention is all you need.
\newblock \emph{Advances in Neural Information Processing Systems}, 2017.

\bibitem[Vincent(2011)]{vincent2011connection}
Pascal Vincent.
\newblock A connection between score matching and denoising autoencoders.
\newblock \emph{Neural computation}, 23\penalty0 (7):\penalty0 1661--1674, 2011.

\bibitem[von Platen et~al.(2022)von Platen, Patil, Lozhkov, Cuenca, Lambert, Rasul, Davaadorj, Nair, Paul, Liu, Berman, Xu, and Wolf]{von_platen_diffusers_2022}
Patrick von Platen, Suraj Patil, Anton Lozhkov, Pedro Cuenca, Nathan Lambert, Kashif Rasul, Mishig Davaadorj, Dhruv Nair, Sayak Paul, Steven Liu, William Berman, Yiyi Xu, and Thomas Wolf.
\newblock {Diffusers}: State-of-the-art diffusion models.
\newblock \url{https://github.com/huggingface/diffusers}, 2022.
\newblock Hugging Face; accessed 23 May 2025.

\bibitem[Von~Werra et~al.(2022)Von~Werra, Tunstall, Thakur, Luccioni, Thrush, Piktus, Marty, Rajani, Mustar, Ngo, et~al.]{von2022evaluate}
Leandro Von~Werra, Lewis Tunstall, Abhishek Thakur, Alexandra~Sasha Luccioni, Tristan Thrush, Aleksandra Piktus, Felix Marty, Nazneen Rajani, Victor Mustar, Helen Ngo, et~al.
\newblock Evaluate \& evaluation on the hub: Better best practices for data and model measurements.
\newblock \emph{arXiv preprint arXiv:2210.01970}, 2022.

\bibitem[Wolf et~al.(2019)Wolf, Debut, Sanh, Chaumond, Delangue, Moi, Cistac, Rault, Louf, Funtowicz, et~al.]{wolf2019huggingface}
Thomas Wolf, Lysandre Debut, Victor Sanh, Julien Chaumond, Clement Delangue, Anthony Moi, Pierric Cistac, Tim Rault, R{\'e}mi Louf, Morgan Funtowicz, et~al.
\newblock {HuggingFace}'s transformers: {S}tate-of-the-art natural language processing.
\newblock \emph{arXiv preprint arXiv:1910.03771}, 2019.

\bibitem[Wu et~al.(2024)Wu, Zhou, Chen, and Zhu]{wu2024counterfactual}
Shenghao Wu, Wenbin Zhou, Minshuo Chen, and Shixiang Zhu.
\newblock Counterfactual generative models for time-varying treatments.
\newblock In \emph{Proceedings of the 30th ACM SIGKDD Conference on Knowledge Discovery and Data Mining}, pp.\  3402--3413, 2024.

\bibitem[Yu et~al.(2025)Yu, Mehta, Luedtke, and Harchaoui]{yu2025stochastic}
Facheng Yu, Ronak Mehta, Alex Luedtke, and Zaid Harchaoui.
\newblock Stochastic gradients under nuisances.
\newblock \emph{arXiv preprint arXiv:\#}, 2025.

\bibitem[Zhang et~al.(2023)Zhang, Yang, Wang, and Liu]{zhang2023effect}
Qi~Zhang, Jiafei Yang, Wenlong Wang, and Zhihong Liu.
\newblock Effect of extracurricular tutoring on adolescent students cognitive ability: A propensity score matching analysis.
\newblock \emph{Medicine}, 102\penalty0 (36):\penalty0 e35090, 2023.

\end{thebibliography}
}

\appendix

\setcounter{equation}{0}
\renewcommand{\theequation}{S\arabic{equation}}
\setcounter{theorem}{0}
\setcounter{figure}{0}
\setcounter{table}{0}
\setcounter{lemma}{0}
\setcounter{corollary}{0}
\setcounter{proposition}{0}
\renewcommand{\thetheorem}{S\arabic{theorem}}
\renewcommand{\thecorollary}{S\arabic{corollary}}
\renewcommand{\thelemma}{S\arabic{lemma}}
\renewcommand{\theproposition}{S\arabic{proposition}}
\renewcommand{\thefigure}{S\arabic{figure}}
\renewcommand{\thetable}{S\arabic{table}}
\renewcommand{\thealgorithm}{S\arabic{algorithm}}

\section*{\LARGE Appendices}

\DoToC

\section{Plausibility of \cref{cond:divergence} in our examples}\label{app:divergence}

As summarized in \Cref{tab:divergenceCondition}, \cref{cond:divergence} holds in each of our examples under conditions, under regularity conditions that we give below.
\begin{example}[continues=ex:flow, name=Flow matching]
    When $D$ is the 2-Wasserstein distance ($W_2$), Thm.~1 in \cite{benton2023error} gives conditions under which \eqref{eq:regBound} holds 
    with $b=1/2$, $\epsilon=0$, and $\Cdivergence$ a constant depending on the smoothness of the vector field $\theta$. These conditions are satisfied if each vector field $\theta\in\Theta$ is sufficiently smooth and corresponds to a unique, smooth flow---see \cite{benton2023error}.
\end{example}

\begin{example}[continues=ex:diffusion,name=Diffusion model]
    We follow the arguments used in \cite{oko2023diffusion} to establish \cref{cond:divergence}. These arguments rely on several regularity conditions on $\mathbb{P}$ and its density $\mathbbm{p}$.
    \begin{condenum}[resume*]
        \item\label{cond:diffusionSupport} \textit{Supported on hypercube:} $\mathbb{P}$ has support $\mathcal{Y}=[-1,1]^d$.
        \item\label{cond:diffusionBesov} \textit{Smooth density:} $\|\mathbbm{p}\|_{s|p,q}\le \Cbesov$, with $\|\cdot\|_{s|p,q}$ the usual norm on the Besov space $B_{p,q}^s(\mathcal{Y})$ \citep[][Sec.~2.7.2]{vaart2023empirical}.
        \item\label{cond:diffusionBddDensity} \textit{Density bounded away from $0$ and infinity:} $\mathbbm{p}$ is bounded in  $[1/\CdiffDens,\CdiffDens]$ on $\mathcal{Y}$.
        \item\label{cond:diffusionBoundary} \textit{Density smooth at the boundary:} For some $\varepsilon>0$, the restriction $\mathbbm{p}_\varepsilon$ of $\mathbbm{p}$ to $[-1,1]^d\backslash [-1+\varepsilon,1-\varepsilon]^d$ is infinitely differentiable with $\sum_{k=0}^\infty 2^{-k}\|D^k \mathbbm{p}_\varepsilon\|_\infty/(1+\|D^k \mathbbm{p}_\varepsilon\|_\infty)\le \CdiffusionBoundary$.
    \end{condenum}
     The above conditions are derived from Assumptions 2.4 and 2.6 of \cite{oko2023diffusion}. All our results will also hold if \cref{cond:diffusionBoundary} is weakened so that $\varepsilon$ decays with $n$ at an appropriate rate---see Assumption~2.6 from that work for details. 
    Following \cite{oko2023diffusion}, we also make the following requirement on the truncation times:
\begin{condenum}[resume*]
    \item\label{cond:diffusionTruncation} \textit{Truncation times change appropriately with $n$:} $\overline{t}=s\log n/(\underline{\beta}(2s+d))$ and $\underline{t}=n^{-\gamma}$ for any $\gamma>0$ sufficiently large so that $\TV(\mathbb{P},\Law(Y_{\underline{t}}))=O(n^{-s/(2s+d)})$. 
\end{condenum}
    By Thm.~D.2 in \cite{oko2023diffusion}, such a $\gamma$ necessarily exists under \cref{cond:diffusionSupport,cond:diffusionBddDensity,cond:diffusionBesov,cond:diffusionBoundary}, and it can be chosen uniformly over all $\mathbb{P}$ satisfying \cref{cond:diffusionBddDensity,cond:diffusionBesov,cond:diffusionBoundary} for fixed values of the conditions' constants.
    
    To ensure each reverse-time SDE we consider has a strong solution, we further require there to exist $L<\infty$ such that each $\theta\in\Theta$ is $L$-Lipschitz in its first argument. We will show \cref{cond:divergence} holds with $b=1/2$, $\Cdivergence=1/\underline{\beta}$, and $\epsilon:=O(n^{-s/(2s+d)})$, where showing $\epsilon$ is small requires \cref{cond:diffusionTruncation}.

    We will write write $\mathbb{P}_{\theta,N(0,I)}$ as shorthand for $\tau(\theta)_{\sharp}\Pi$ and $\mathbb{P}_{\theta,Y_{\overline{t}}}$ to represent the law of the solution to the same reverse-time SDE as $\tau(\theta)$, but with the $N(0_d,I_d)$ distribution used in the initial condition at time $\overline{t}$ replaced by $\Law(Y_{\overline{t}})$. By the triangle inequality,
    \begin{align*}
        \TV(\mathbb{P},\mathbb{P}_{\theta,N(0,I)})&\le \TV(\Law(Y_{\underline{t}}),\mathbb{P}_{\theta,Y_{\overline{t}}}) + \TV(\mathbb{P}_{\theta,Y_{\overline{t}}},\mathbb{P}_{\theta,N(0,I)}) + \TV(\mathbb{P},\Law(Y_{\underline{t}})).
    \end{align*}
    We bound the three terms on the right separately.     For the first term, Pinsker's inequality and Girsanov's theorem [\citealp{oko2023diffusion}, Prop. D.1; \citealp{karatzas1991brownian}] together imply the following bound:
    \begin{align*}
        \TV(\Law(Y_{\underline{t}}),\mathbb{P}_{\theta,Y_{\overline{t}}})&\le \left[\frac{1}{2}\KL{\mathbb{P}}{\mathbb{P}_{\theta,Y_{\overline{t}}}}\right]^{1/2}\le  \left(\medint\int_{\underline{t}}^{\overline{t}} \beta_t^{-2} E\left[\left\|\theta(Y_t,t)-\theta_{\mathbb{P}}(Y_t,t)\right\|^2 \right] dt\right)^{1/2}.
    \end{align*}
    The right-hand side is upper bounded by $\underline{\beta}^{-1}\mathcal{G}_{\mathbb{P}}(\theta)^{1/2}$. 
    The second term is upper bounded by a constant $C$ depending only on $\CdiffDens$ times $e^{-\overline{t}\underline{\beta}}$ \citep[][Lem.~D.3]{oko2023diffusion}, which is $O(n^{-s/(2s+d)})$ by the choice of $\overline{t}$ in \cref{cond:diffusionTruncation}. The choice of $\underline{t}$ in \cref{cond:diffusionTruncation} ensures the third term is $O(n^{-s/(2s+d)})$. Putting these three bounds together shows that $\TV(\Law(Y_{\underline{t}}),\mathbb{P}_{\theta,Y_{\overline{t}}})\le \underline{\beta}^{-1}\mathcal{G}_{\mathbb{P}}(\theta)^{1/2} + \epsilon$ for $\epsilon:=\TV(\mathbb{P},\Law(Y_{\underline{t}}))=O(n^{-s/(2s+d)})$ and an appropriately defined $C>0$, and inspecting the second and third terms above shows that the additive $\epsilon$ term is needed on the right due to the truncation of the forward diffusion process at time $\overline{t}<\infty$ and reverse diffusion process at time $\underline{t}>0$.
\end{example}

\begin{example}[continues=ex:autoreg,name=Autoregressive model]
    Let $\KLop$ denote the Kullback-Leibler (KL) divergence. To ensure $\KL{\mathbb{P}}{\tau(\theta)_{\sharp}\Pi}$ is finite, suppose $\mathbb{P}$ is dominated by $\tau(\theta)_{\sharp}\Pi$ for each $\theta\in\Theta$. 
    By the definition of the KL divergence \citep[][Eq.~2.26]{cover1999elements}, $\KL{\mathbb{P}}{\tau(\theta)_{\sharp}\Pi}=\mathcal{G}_{\mathbb{P}}(\theta)$. Hence, \cref{cond:divergence} trivially holds with $D=\KLop$, $b=\Cdivergence=1$, and $\epsilon=0$.
\end{example}

\begin{table}[tb]\centering
    \caption{In our examples, \cref{cond:divergence} holds for the below choices of $(D,b,\epsilon)$.}
    \label{tab:divergenceCondition}
    \resizebox{\textwidth}{!}{
    \begin{tabular}{l l l l l}
        Framework\hspace{3em} & $D$ & $b$ & $\epsilon$ & Key results used in proof \\\midrule
        Flow matching & 2-Wasserstein & $1/2$ & 0 & Alekseev-Gr\"{o}bner \citep{benton2023error} \\
        Diffusion model & Total variation & $1/2$ & Trunc. error & Girsanov \citep{oko2023diffusion}; Pinsker \citep{tsybakov2009nonparametric} \\
       Autoregressive language model & KL divergence & $1$ & $0$ & Definition of $D_{\mathrm{KL}}$ \citep{cover1999elements}
    \end{tabular}
    }
\end{table}

\section{Review: covering numbers, entropy integrals, and a local maximal inequality}\label{app:covNumReview}

This appendix reviews concepts from empirical process theory: covering numbers, entropy integrals, and a local maximal inequality. We focus on collections of functions $\mathcal{F}\subseteq\mathbb{R}^{\mathcal{Z}}$ and study how well an empirical distribution $P_n$ approximates its population counterpart $P$. Specifically, we examine conditions under which $E\|P_n-P\|_{\mathcal{F}}:=E\sup_{f\in\mathcal{F}}|(P_n-P)f|$ is small. While we present these concepts for functions and probability distributions on $\mathcal{Z}$, the definitions and results apply analogously to those defined on $\mathcal{Y}$, such as $\ell_{\mathbb{P}}(\theta)$ and $\mathbb{P}$.

For a class of functions $\mathcal{F}\subseteq \mathbb{R}^\mathcal{Z}$ and a probability measure $Q$ on $\mathcal{Z}$, the (external) covering number $N(\epsilon,\mathcal{F},L^2(Q))$ is defined as the smallest cardinality of a $\mathcal{F}_{\epsilon}\subset L^2(Q)$ satisfying the following: for all $f\in\mathcal{F}$, there exists $g\in\mathcal{F}_{\epsilon}$ such that $\|f-g\|_{L^2(Q)}\le \epsilon$. The uniform entropy integral is given by
\begin{align}
    J(\delta,\mathcal{F}):= \sup_Q \medint\int_0^\delta \sqrt{1+\log N(\epsilon,\mathcal{F},L^2(Q))}\,d\epsilon, \label{eq:unifEntDef}
\end{align}
where the supremum is over all finitely supported measures $Q$ on $\mathcal{Z}$.

In the following lemma, the class $\mathcal{F}\subseteq \mathbb{R}^\mathcal{Z}$ is called `suitably measurable' if $(z_i)_{i=1}^n\mapsto \sup_{f\in\mathcal{F}}\left|\sum_{i=1}^n e_i f^k(z_i)\right|$ is measurable for all $(e_i)_{i=1}^n\in \{-1,1\}^n$ and $k\in \{1,2\}$.
\begin{lemma}[Variant of Thm.~2.1 in \citealp{van2011local}]\label{lem:localMaximal}
    There exists a universal constant $K>0$ such that, for all probability measures $P$ on some space $\mathcal{Z}$, suitably measurable $\mathcal{F}\subset [-b,b]^{\mathcal{Z}}$ bounded by some $b>0$, and $\delta^2\ge \sup_{f\in\mathcal{F}}\|f\|_{L^2(P)}^2$,
    \begin{align*}
        E\|P_n-P\|_{\mathcal{F}}&\le K n^{-1/2}J(\delta,\mathcal{F})\left(1 + \frac{b J(\delta,\mathcal{F})}{\delta^2 n^{1/2}}\right).
    \end{align*}
\end{lemma}
We defer this and all proofs to the end of this appendix. The above can be used to compute a rate of convergence \citep[][Theorem 3.2.5]{vaart2023empirical} or finite-sample generalization bound \citep[][Lemma 3.5.9]{vaart2023empirical} for an empirical risk minimizer. This is the case, for example, in our \Cref{thm:excRiskOSLFormal,thm:excRiskOSLLocal}. In special cases where $J(\delta,\mathcal{F}/b)$ is upper bounded by a constant times $\delta^{1-1/(2\rho)}$ for $\rho>1/2$, such guarantees are determined by the solution to \eqref{eq:abstractEntropyBd} in the lemma below, for appropriate choices of constants $c,d$. 

\begin{lemma}[Helpful result for determining complexity term in generalization bounds]
\label{lem:deltaThresh}
    Fix $c,d\ge 0$ and $\rho>1/2$. 
    If $\delta\ge \left(1\vee [(c+d/4)^2/n]^{1/(4\rho+2)}\right)[(c+d/4)^2/n]^{\rho/(2\rho+1)}$, then
    \begin{align}
        c(\delta^{1-1/(2\rho)}\vee\delta)\left(1 + \frac{d(\delta^{-1-1/(2\rho)}\vee \delta^{-1})}{n^{1/2}}\right)\le n^{1/2}\delta^2. \label{eq:abstractEntropyBd}
    \end{align}
\end{lemma}

The following lemma is also useful when computing a critical radius $\delta_n$ such that, for all $\delta\ge \delta_n$, the right-hand side of \Cref{lem:localMaximal} is no more than $\delta^2 n^{1/2}$. Specifically, this lemma shows that $\delta_n$ can be computed by just finding the $\delta$ such that the bound from \Cref{lem:localMaximal} equals $\delta^2 n^{1/2}$.
\begin{lemma}[Monotonicity of bound from \Cref{lem:localMaximal} when divided by $\delta^2$] \label{lem:monotoneUB}
    For any function class $\mathcal{F}$ and constants $c,d>0$, the following function is strictly decreasing over $(0,\infty)$:
    \begin{align*}
        \delta\mapsto [cJ(\delta,\mathcal{F})/\delta^2]\left[1 + dJ(\delta,\mathcal{F})/\delta^2\right].
    \end{align*}
\end{lemma}

The following lemma lower bounds the complexity term in \Cref{thm:excRiskOSL}.
\begin{lemma}[Complexity term is $\Omega(n^{-1})$]\label{lem:deltaNslow}
    If $\delta_n$ satisfies $J(\delta_n,\ell_{\mathbb{P}}(\underline{\Theta}))\le n^{1/2} \delta_n^2$, then $\delta_n^2\ge n^{-1}$.
\end{lemma}

We conclude with proofs.
\begin{proof}[Proof of \Cref{lem:localMaximal}]
    First suppose $\delta<b$. Since $\mathcal{F}/b\subset [-1,1]^{\mathcal{Z}}$ is suitably measurable and $\sup_{g\in\mathcal{F}/b}\|g\|_{L^2(P)}^2\le (\delta/b)^2 < 1$, Thm.~2.1 in \cite{van2011local} with constant envelope function $1$ shows there exists a universal constant $K'$ such that
    \begin{align*}
        E\|P_n-P\|_{\mathcal{F}/b}&\le K' n^{-1/2}J(\delta/b,\mathcal{F}/b)\left(1 + \frac{b^2 J(\delta/b,\mathcal{F}/b)}{\delta^2 n^{1/2}}\right).
    \end{align*}
    Now suppose $\delta\ge b$. In this case, Thm.~2.14.1 of \cite{vaart2023empirical} shows there exists a universal constant $K''$ such that $E\|P_n-P\|_{\mathcal{F}/b}\le K''n^{-1/2}J(1,\mathcal{F}/b)$. Combining these two cases, letting $K=K'\vee K''$, and using that $J(\,\cdot\,,\mathcal{F}/b)$ is monotone nondecreasing and nonnegative shows that
    \begin{align*}
        E\|P_n-P\|_{\mathcal{F}/b}&\le K n^{-1/2}J(\delta/b,\mathcal{F}/b)\left(1 + \frac{b^2 J(\delta/b,\mathcal{F}/b)}{\delta^2 n^{1/2}}\right)
    \end{align*}
    for all $\delta>0$. Since $N(\epsilon,\mathcal{F}/b,L^2(Q))=N(b\epsilon,\mathcal{F},L^2(Q))$ for any $Q$, $J(\delta/b,\mathcal{F}/b)=b^{-1}J(\delta,\mathcal{F})$. Combining this with $E\|P_n-P\|_{\mathcal{F}/b}=b^{-1}E\|P_n-P\|_{\mathcal{F}}$ gives the desired bound.
\end{proof}

\begin{proof}[Proof of \Cref{lem:deltaThresh}]
    We first show that
    \begin{align}
        \delta\ge n^{-\rho/(2\rho+1)}2^{-2\rho/(2\rho+1)}\left[c + c^{1/2}(c+d)^{1/2}\right]^{2\rho/(2\rho+1)}, \label{eq:deltaIneq1}
    \end{align}    
    implies
    \begin{align}
        c\delta^{1-1/(2\rho)}\left(1 + \frac{d\delta^{-1-1/(2\rho)}}{n^{1/2}}\right)\le n^{1/2}\delta^2. \label{eq:abstractEntropyBd1}
    \end{align}
    When $\delta$ equals the right-hand side of \eqref{eq:deltaIneq1}, \eqref{eq:abstractEntropyBd1} is an equality. Divide both sides of \eqref{eq:abstractEntropyBd1} by $\delta^2$ and note that the resulting left-hand side and right-hand side are monotone decreasing and constant functions of $\delta$, respectively. Hence, \eqref{eq:abstractEntropyBd1} holds whenever \eqref{eq:deltaIneq1} holds, that is, \eqref{eq:deltaIneq1} implies \eqref{eq:abstractEntropyBd1}. A similar argument shows $\delta\ge n^{-1/2}2^{-1}\left[c + c^{1/2}(c+d)^{1/2}\right]$ implies
    \begin{align*}
        c\delta^{1-1/(2\rho)}\left(1 + \frac{d\delta^{-1-1/(2\rho)}}{n^{1/2}}\right)\le n^{1/2}\delta^2.
    \end{align*}
    Putting these two results together shows \eqref{eq:abstractEntropyBd} holds whenever
    \begin{align*}
        \delta\ge \left(n^{-1/2}2^{-1}\left[c + c^{1/2}(c+d)^{1/2}\right]\right)\vee \left(n^{-\rho/(2\rho+1)}2^{-2\rho/(2\rho+1)}\left[c + c^{1/2}(c+d)^{1/2}\right]^{2\rho/(2\rho+1)}\right).
    \end{align*}
    By Young's inequality, $c^{1/2}(c+d)^{1/2}\le c+d/2$, and so the right-hand side above is upper bounded by $[n^{-1/2}(c+d/4)]\vee [n^{-1/2}(c+d/4)]^{2\rho/(2\rho+1)}$, which gives the result.
\end{proof}

\begin{proof}[Proof of \Cref{lem:monotoneUB}]
    Fix a finitely supported distribution $Q$. Note that $\delta\mapsto J_Q(\delta,\mathcal{F}):= \int_0^\delta \sqrt{1+\log N(\epsilon,\mathcal{F},L^2(Q))}\,d\epsilon$ is the integral of a nonincreasing function, and so is concave, and hence has nonincreasing average derivative $\delta\mapsto J_Q(\delta,\mathcal{F})/\delta$. As $Q$ was arbitrary and the pointwise supreumum of a collection of nonincreasing functions is itself nonincreasing, $\delta\mapsto \sup_Q J_Q(\delta,\mathcal{F})/\delta=J(\delta,\mathcal{F})/\delta$ is nonincreasing. Hence, $\delta\mapsto J(\delta,\mathcal{F})/\delta^2$ is strictly decreasing. The result follows since this function is also positive and the product of two strictly decreasing positive functions is strictly decreasing.
\end{proof}

\begin{proof}[Proof of \Cref{lem:deltaNslow}]
By \eqref{eq:unifEntDef}, $J(\delta_n,\ell_{\mathbb{P}}(\underline{\Theta}))\ge \int_0^{\delta_n} d\epsilon= \delta_n$. Combining this with $J(\delta_n,\ell_{\mathbb{P}}(\underline{\Theta}))\le n^{1/2} \delta_n^2$ gives the result.
\end{proof}

\section{Discussion of \nameref*{alg:double} generalization bound (\Cref{thm:excRiskOSL})}\label{app:genBoundDiscussion}

We begin by discussing the behavior of terms in the finite-sample bound in \eqref{eq:excRiskBdOSLinformal}. To simplify the discussion, we do this under an asymptotic regime where $\underline{\Theta}$ may vary as $n\rightarrow\infty$.

The complexity term $\delta_n^2$ never decays to $0$ faster than $1/n$, and will typically decay slower unless unless $\underline{\Theta}$ is finite-dimensional and fixed with sample size. Upper bounds on the uniform entropy integral $J(\,\cdot\,,\ell_{\mathbb{P}}(\underline{\Theta}))$, and therefore $\delta_n$, can be obtained provided $\underline{\Theta}$ is small enough so that $\ell_{\mathbb{P}}(\underline{\Theta})$ is smooth \citep[e.g.,][Chaps. 2.7.1 and 2.7.2]{vaart2023empirical}, sparsely indexed by some parameter \citep[e.g.,][Chap. 14.2]{buhlmann2011statistics}, or otherwise structured in some way \citep[e.g.,][Chaps. 2.7.3 and 2.7.4]{vaart2023empirical}. An example is given \Cref{thm:diffusionMinimax}, where $\underline{\Theta}$ is a neural network class and $\ell_{\mathbb{P}}$ is the denoising score matching loss.

The doubly robust term $\max_{j\in [2]}\|\alpha_n^j-\alpha_P\|_{L^2(P_X)}^2 d_\Psi^2(\psi_n^j,\Psi_P)$ is a product of squared errors, making it fourth order. In contrast, the $\Omega(1/n)$ complexity term $\delta_n^2$ is only second-order. Hence, $n^{-1/4}$ rates for $\|\alpha_n^j-\alpha_P\|_{L^2(P_X)}$ and $ d_\Psi(\psi_n^j,\Psi_P)$ suffice to ensure negligibility of the fourth-order term. Even slower rates can suffice when the complexity term decays more slowly. If $X$ is low-dimensional relative to $Y$ (e.g., because $Y$ is an image or text and $X$ is not), the negligibility of the fourth-order term is even more plausible. Indeed, under reasonable smoothness conditions, the mean-squared prediction error $\|\alpha_n^j-\alpha_P\|_{L^2(P_X)}^2$ from estimating the low-dimensional function $\alpha_P$ will be of smaller order than the complexity term, even before multiplying by the additional quadratic term $d_\Psi^2(\psi_n^j,\Psi_P)$.

The entropy integral $J(\delta_n/(8\Cpos),\ell_{\mathbb{P}}(\underline{\Theta}))$ in the definition of $\delta_n$ measures the size of the loss class $\ell_{\mathbb{P}}(\underline{\Theta})$ over all of $\underline{\Theta}$. In \Cref{thm:excRiskOSLLocal} in \Cref{app:doubleGenLocal}, an alternative, localized version of this result is given, which instead considers the entropy of $\ell_{\mathbb{P}}(\underline{\Theta}_{\delta_n})$, where $\underline{\Theta}_{\delta_n}:=\{\theta\in\underline{\Theta} : \|\ell_{\mathbb{P}}(\theta)\|_{L^2(\mathbb{P})}\le \delta_n\}$. That result yields a similar conclusion to \Cref{thm:excRiskOSL} and replaces \cref{cond:entropyStronger} by a weaker condition, though makes additional requirements on the nuisance estimators. Under conditions, this localized generalization bound yields the same rate as one for \nameref*{alg:oracle} derived using the same techniques---see \Cref{app:oracleGenBenchmark}. This oracle optimality property is notable given that \nameref*{alg:oracle} has access to all $n$ counterfactuals, whereas \nameref*{alg:double} only sees a biased subset of them.

An alternative approach to deriving generalization bounds would be to apply Thm.~3 from \cite{foster2023orthogonal}, which uses local Rademacher complexity \citep{bartlett2006local} instead of entropy integrals. While their framework applies to general risk minimization problems with nuisances, including ours, it does not explicitly characterize the fourth-order term in \eqref{eq:excRiskBdOSLinformal} or establish its double robustness. As \cite{bonvini2022fast} noted in a different setting, tailored analyses such as ours yield faster convergence guarantees when nuisances are estimated at different rates.

\section{Formal statement of \nameref*{alg:double} generalization bound (\Cref*{thm:excRiskOSL})}\label{app:formalTheorem}

In the following theorem, we let $K$ denote the universal constant from \Cref{lem:localMaximal} for a review. We further let
\begin{align}
    K_0&:= 2\max_{j\in [2]}\min\Bigg\{\Ccurv\Cpos \left\|\alpha_n^j/\alpha_P\right\|_{L^\infty(P_X)} + \Cmixedweak(1+\Cpos)\Cpos d_\Psi^2(\psi_n^j,\Psi_P), \label{eq:K0def} \\
    &\hspace{7.25em}\Cpos\left[ \Ccurv + \Cmixedweak d_\Psi^2(\psi_n^j,\Psi_P)\right] + 2\|\alpha_n^j-\alpha_P\|_{L^\infty(P_X)}^2\left[\Ccurv+\Cmixedweak d_\Psi^2(\psi_n^j,\Psi_P)\right]\Bigg\}. \nonumber
\end{align}
Though we do not require this in the theorem below, if $d_\Psi^2(\psi_n^j,\Psi_P)$ is a.s. bounded, then $K_0$ is bounded by a constant under \cref{cond:strongPositivity}. 
If $\psi_n^j$ and $\alpha_n^j$ were perfect estimators, so that $\psi_n^j\in \Psi_P$ and $\alpha_n^j=\alpha_P$, $K_0$ would simplify to $2\Ccurv\Cpos$.
\begin{theorem}[Formal statement of \Cref{thm:excRiskOSL}]\label{thm:excRiskOSLFormal}
    Suppose \cref{cond:strongPositivity,cond:RMexists,cond:bddLoss,cond:mixedLipschitzWeak,cond:curvature,cond:entropyStronger} and that there exists an empirical risk minimizer $\theta_n\in\argmin_{\theta\in\underline{\Theta}} R_n(\theta)$. Fix a $\delta_n$ satisfying
    \begin{align}
        8K \Cpos J(\delta_n/(8\Cpos),\ell_{\mathbb{P}}(\underline{\Theta}))\left[ 1 +16\Closs\Cpos^2 J(\delta_n/(8\Cpos),\ell_{\mathbb{P}}(\underline{\Theta}))/(\delta_n^2 n_{\bar{j}}^{1/2})\right]\le \lfloor n/2\rfloor^{1/2} \delta_n^2.\label{eq:doubleGenDeltaN}
    \end{align}
    Fix $s>0$. With probability at least $1-e^{-s}$,
    \begin{align}
        \mathcal{G}_{\mathbb{P}}(\theta_n)&\le 4\inf_{\theta\in\underline{\Theta}}\mathcal{G}_{\mathbb{P}}(\theta) + K_1\delta_n^2 + K_2 (s+2)/n + {\color{CBteal}\bm{K_3\, \max_{j\in [2]}\|\alpha_n^j-\alpha_P\|_{L^2(P_X)}^2d_\Psi^2(\psi_n^j,\Psi_P)}},   \label{eq:excRiskBdOSL}   
    \end{align}
    where $K_1=1350(1\vee 135 K_0)$, $K_2=2700\left[\Closs\Cpos + (1\vee 135 K_0)\right]$, and $K_3=13\Cmixedweak$. The final term is \textbf{\color{CBteal}doubly robust}, vanishing if $\alpha_n^j=\alpha_P$ or $\psi_n^j\in\Psi_P$ for $j\in [2]$.
    
\end{theorem}
Though beyond the scope of this work, a more careful analysis could sharpen the constants.

\section{A localized \nameref*{alg:double} generalization bound}\label{app:doubleGenLocal}

In this appendix, we derive a generalization bound under an alternative to the global entropy bound from the main text, \cref{cond:entropyStronger}. The bound we present here only requires the local entropy condition given below---which is weaker than the global one assumed in \Cref{thm:excRiskOSLFormal}---but will also make two additional requirements. In the local entropy condition, $\ell_{\mathbb{P}}(\underline{\Theta}_\delta):=\{\ell_{\mathbb{P}}(\theta) : \theta\in\underline{\Theta}_\delta\}$ with $\underline{\Theta}_\delta:=\{\theta\in\underline{\Theta} : \|\ell_{\mathbb{P}}(\theta)\|_{L^2(\mathbb{P})}\le \delta\}$, $\delta>0$.
\begin{condenum}[resume*]
    \item\label{cond:entropy} \textit{Local uniform entropy bound:} $J(\delta,\ell_{\mathbb{P}}(\underline{\Theta}_\delta))<\infty$ for some $\delta>0$.
\end{condenum}
The first additional requirement compared to \Cref{thm:excRiskOSLFormal} is the following:
\begin{condenum}[resume*]
    \item\label{cond:oneGoodNuis} \textit{At least one nuisance not estimated poorly:} it is almost surely true that, for each $j\in\{1,2\}$, at least one of the following holds:
    \begin{enumerate}[label=(\roman*)]
        \item\label{cond:oneGoodNuisPsi} $ d_\Psi^2(\psi_n^j,\Psi_P)\le \left[4(1\vee \|\alpha_P/\alpha_n^j\|_{L^\infty(P_X)})\Cpos\Cmixedstrong\|1+\alpha_n^j/\alpha_P\|_{L^\infty(P_X)}\right]^{-1}$;
        \item\label{cond:oneGoodNuisAlpha} $\|\alpha_n^j-\alpha_P\|_{L^\infty(P_X)}^2\le \left(8\left[1+\Cmixedstrong d_\Psi^2(\psi_n^j,\Psi_P)\right]\right)^{-1}$.
    \end{enumerate}
\end{condenum}
The almost-sure requirement made in the condition can be straightforwardly relaxed to a `w.p. at least $1-\varepsilon$' requirement, at the cost of a slightly more involved theorem statement. 
The two inequalities respectively impose conditions on the quality of the estimator $\psi_n$ and $\alpha_n$. Under \cref{cond:strongPositivity}, a sufficient condition for the first inequality to hold is that $d_\Psi^2(\psi_n^j,\Psi_P)\le [8\Cpos^3\Cmixedstrong]^{-1}$.

We will need an alternative version of \cref{cond:mixedLipschitzWeak} that measures the discrepancy of $\theta$ from $\theta_{\mathbb{P}}$ by $\|\ell_{\mathbb{P}}(\theta)\|_{L^2(\mathbb{P})}^2$, rather than $\mathcal{G}_{\mathbb{P}}(\theta)$.
\begin{condenum}[resume*]
    \item\label{cond:mixedLipschitzStrong} \textit{Mixed-Lipschitz loss (stronger version):} there exists $\Cmixedstrong<\infty$ such that, for all $\theta\in\underline{\Theta}$, $\psi\in\Psi$, and $\psi_P\in\Psi_P$,
    \begin{align*}
       &\medint\int\left\{\medint\int \left[ \ell_{\mathbb{P}}(\theta)(\psi(u|x)) - \ell_{\mathbb{P}}(\theta)(\psi_P(u|x))\right]\,\Pi(du)\right\}^2 P_X(dx) \\
       &\quad\le \Cmixedstrong \|\ell_{\mathbb{P}}(\theta)\|_{L^2(\mathbb{P})}^2\,d_{\Psi}^2(\psi,\Psi_P).
    \end{align*}
\end{condenum}
The above is stronger than \cref{cond:mixedLipschitzWeak} in the sense that, when it and \cref{cond:curvature} hold, \cref{cond:mixedLipschitzWeak} holds with $\Cmixedweak=\Ccurv\Cmixedstrong$.

In the following theorem, $K$, $K_1$, $K_2$, and $K_3$ are as defined in and above \Cref{thm:excRiskOSLFormal}.
\begin{theorem}[Localized version \nameref*{alg:double} generalization bound]\label{thm:excRiskOSLLocal}
    Suppose \cref{cond:strongPositivity,cond:RMexists,cond:bddLoss,cond:curvature,cond:entropy,cond:oneGoodNuis,cond:mixedLipschitzStrong} and that there exists an empirical risk minimizer $\theta_n\in\argmin_{\theta\in\underline{\Theta}} R_n(\theta)$. Fix $s>0$ and $\delta_n>0$ such that
    \begin{align}
        8K \Cpos J(2\Cpos^{1/2} \delta_n,\ell_{\mathbb{P}}(\underline{\Theta}_{2\Cpos^{1/2}\delta_n}))\left( 1 +\frac{16\Closs\Cpos^2 J(2\Cpos^{1/2} \delta_n,\ell_{\mathbb{P}}(\underline{\Theta}_{2\Cpos^{1/2}\delta_n}))}{\delta_n^2 \lfloor n/2\rfloor^{1/2}}\right)\le \lfloor n/2\rfloor^{1/2} \delta_n^2.\label{eq:doubleGenDeltaNLocal}
    \end{align}
    With probability at least $1-e^{-s}$, \eqref{eq:doubleGenDeltaN} holds.
\end{theorem}

\section{Benchmark: generalization error of \nameref*{alg:oracle}}\label{app:oracleGenBenchmark}

We now establish a generalization bound for \nameref*{alg:oracle} when implemented via an empirical risk minimizer over $\underline{\Theta}\subseteq\Theta$. The resulting guarantee---which is proved using the same techniques as for the \nameref*{alg:double} generalization bounds from \Cref{thm:excRiskOSLFormal,thm:excRiskOSLLocal}---serves as a benchmark for those bounds. Despite having access only to factual rather than counterfactual outcomes, we show that \nameref*{alg:double} can match this benchmark up to constant factors. When it achieves this, we call \nameref*{alg:double} oracle optimal.


In what follows, $K$ is the same universal constant appearing in \Cref{thm:excRiskOSLFormal,thm:excRiskOSLLocal}.

\begin{proposition}[Generalization bound for \nameref*{alg:oracle}]\label{thm:excRiskOracle}
    Suppose \cref{cond:RMexists,cond:bddLoss,cond:entropy,cond:curvature} and there exists an empirical risk minimizer $\theta_n^\star\in\argmin_{\theta\in\underline{\Theta}} R_n^\star(\theta)$. Fix $s>0$ and $\delta_n>0$ satisfying
    \begin{align}
        K J(\delta_n,\ell_{\mathbb{P}}(\underline{\Theta}_{\delta_n}))\left[ 1 +\Closs J(\delta_n,\ell_{\mathbb{P}}(\underline{\Theta}_{\delta_n}))/(\delta_n^2 n^{1/2})\right]\le n^{1/2} \delta_n^2. \label{eq:oracleDeltaN}
    \end{align}
    With probability at least $1-e^{-s}$,
    \begin{align}
        \mathcal{G}_{\mathbb{P}}(\theta_n^\star)&\le 3\inf_{\theta\in\underline{\Theta}}\mathcal{G}_{\mathbb{P}}(\theta) + K_1^\star \delta_n^2 + K_2^\star (s+1)/n,  \label{eq:excRiskBd}    
    \end{align}
    where $K_1^\star:=540(2\vee 270\Ccurv)$ and $K_2^\star:= 540\left[\Closs + (2\vee 270\Ccurv)\right]$.
\end{proposition}
The bound for \nameref*{alg:oracle} resembles the ones for \nameref*{alg:double} in \Cref{thm:excRiskOSLFormal,thm:excRiskOSLLocal}, but requires fewer conditions and replaces the doubly robust error term by zero. This is not surprising given that \nameref*{alg:oracle} does not rely on any nuisances. The additional conditions for \nameref*{alg:double} require a mixed Lipschitz loss, strong positivity, and either a stronger uniform entropy condition (\Cref{thm:excRiskOSLFormal}) or the nuisances to be estimated well enough (\Cref{thm:excRiskOSLLocal}). When these conditions hold and the doubly robust term is negligible, the bounds for \nameref*{alg:double} can be of the same order as the one for \nameref*{alg:oracle}, making \nameref*{alg:double} oracle optimal. This is the case in our study of \nameref*{alg:double} diffusion modeling in \Cref{sec:minimax}. There, we also show something stronger: the rate of the generalization upper bound for \nameref*{alg:double} not only matches that for \nameref*{alg:oracle}, but also, under regularity conditions, matches a minimax lower bound for all oracle algorithms---that is, those with access to counterfactual data.

\begin{proof}[Proof of \Cref{thm:excRiskOracle}]
    By \Cref{lem:localMaximal}, the following holds for any $\delta>0$:
    \begin{align*}
        n^{1/2}E\left\|\mathbb{P}_n-\mathbb{P}\right\|_{\ell_{\mathbb{P}}(\underline{\Theta}_\delta)}&\le K J(\delta,\ell_{\mathbb{P}}(\underline{\Theta}_\delta))\left[ 1 +\frac{\Closs J(\delta,\ell_{\mathbb{P}}(\underline{\Theta}_\delta))}{\delta^2 n^{1/2}}\right]. 
    \end{align*}
    By the choice of $\delta_n$, the right-hand side is no more than $n^{1/2}\delta^2$ when $\delta=\delta_n$ and, by \Cref{lem:monotoneUB}, the same holds true for all $\delta\ge \delta_n$. 
    
    Let $\eta:=1/\max\{2,270\Ccurv\}\in (0,1)$. Applying the consequence of Talagrand's inequality given in Lem.~3.5.9 of \cite{vaart2023empirical} to the class $\{-\ell_{\mathbb{P}}(\theta) : \theta\in\underline{\Theta}\}$ then shows that, w.p. at least $1-e^{-[s+\log 2]}$,
    \begin{align}
        \mathbb{P}\ell_{\mathbb{P}}(\theta)&\le \mathbb{P}_n\ell_{\mathbb{P}}(\theta) + 135 \eta \mathbb{P}\ell_{\mathbb{P}}^2(\theta) + 135\frac{\delta_n^2}{\eta} + 135\left(\Closs + \frac{1}{\eta}\right)\frac{s+\log 2}{n} \;\;\forall\theta\in \underline{\Theta}.\label{eq:talagrand1}
    \end{align}
    Similarly, applying  Lem.~3.5.9 of \cite{vaart2023empirical} to $\{\ell_{\mathbb{P}}(\theta) : \theta\in\underline{\Theta}\}$ shows that, w.p. at least $1-e^{-[s+\log 2]}$,
    \begin{align}
        \mathbb{P}_n\ell_{\mathbb{P}}(\theta)&\le \mathbb{P}\ell_{\mathbb{P}}(\theta) + 135 \eta \mathbb{P}\ell_{\mathbb{P}}^2(\theta) + 135\frac{\delta_n^2}{\eta} + 135\left(\Closs + \frac{1}{\eta}\right)\frac{s+\log 2}{n} \;\;\forall\theta\in \underline{\Theta}, \label{eq:talagrand2}
    \end{align}
    and a union bound shows that both \eqref{eq:talagrand1} and \eqref{eq:talagrand2} hold w.p. at least $1-e^{-s}$. We assume that both of these inequalities hold hereafter and study their consequences.
    
    Fix $\underline{\theta}\in\underline{\Theta}$. Since $\theta_n^\star$ is an empirical risk minimizer over $\underline{\Theta}$, $\mathbb{P}_n\ell_{\mathbb{P}}(\theta_n^\star)\le \mathbb{P}_n\ell_{\mathbb{P}}(\underline{\theta})$, and so combining \eqref{eq:talagrand1} with $\theta=\theta_n^\star$ and \eqref{eq:talagrand2} with $\theta=\underline{\Theta}$ yields
    \begin{align*}
        \mathbb{P}\ell_{\mathbb{P}}(\theta_n^\star)&\le \mathbb{P}\ell_{\mathbb{P}}(\underline{\theta}) + 135 \eta [\mathbb{P}\ell_{\mathbb{P}}^2(\theta_n^\star)+\mathbb{P}\ell_{\mathbb{P}}^2(\underline{\theta})] + 270\frac{\delta_n^2}{\eta} + 270\left(\Closs + \frac{1}{\eta}\right)\frac{s+\log 2}{n}.
    \end{align*}
    Recognizing that $\mathbb{P}\ell_{\mathbb{P}}(\theta)=\mathcal{G}_{\mathbb{P}}(\theta)$ for any $\theta$, noting that $135\eta=135/\max\{2,270\Ccurv\}\le 1/(2\Ccurv)$, leveraging \cref{cond:curvature} to upper bound $\mathbb{P}\ell_{\mathbb{P}}^2(\theta_n^\star)+\mathbb{P}\ell_{\mathbb{P}}^2(\underline{\theta})$, and subtracting $\frac{1}{2}\mathcal{G}_{\mathbb{P}}(\theta_n^\star)$ from both sides yields
    \begin{align*}
        \frac{1}{2}\mathcal{G}_{\mathbb{P}}(\theta_n^\star)&\le \frac{3}{2}\mathcal{G}_{\mathbb{P}}(\underline{\theta}) + 270\frac{\delta_n^2}{\eta} + 270\left(\Closs + \frac{1}{\eta}\right)\frac{s+\log 2}{n}.
    \end{align*}
    Multiplying both sides by $2$ and recalling the definitions of $K_1^\star$, $K_2^\star$, $\eta$, and $\delta_n$ shows that
    \begin{align*}
        \mathcal{G}_{\mathbb{P}}(\theta_n^\star)&\le 3\mathcal{G}_{\mathbb{P}}(\underline{\theta}) + 540\frac{\delta_n^2}{\eta} + 540\left(\Closs + \frac{1}{\eta}\right)\frac{s+\log 2}{n} \\
        &= 3\mathcal{G}_{\mathbb{P}}(\underline{\theta}) + K_1^\star \delta_n^2 + K_2^\star (s+\log 2)/n.
    \end{align*}
    We have shown that, for any fixed $\underline{\theta}\in\underline{\Theta}$, the above holds w.p. at least $1-e^{-s}$. Choosing $\underline{\theta}$ such that $3\mathcal{G}_{\mathbb{P}}(\underline{\theta})\le 3\inf_{\theta\in\underline{\Theta}}\mathcal{G}_{\mathbb{P}}(\theta) + K_2^\star (1-\log 2)/n$ yields the desired result.
\end{proof}

\section{Proofs of general guarantees for \nameref*{alg:double} (\Cref*{sec:guarantees})}

\subsection{Proofs of generalization error upper bounds (\Cref*{thm:excRiskOSLFormal,thm:excRiskOSLLocal})}

Let $\zeta_P(\theta,x):=\int \ell_{\mathbb{P}}(\theta)(\psi_P(u\mymid x)) \Pi(du)$ and, for any $(\alpha_\diamond,\psi_\diamond)\in [1,\Cpos]^{\mathcal{X}}\times \Psi$, let $\zeta_\diamond(\theta,x):=\int \ell_{\mathbb{P}}(\theta)(\psi_\diamond(u\mymid x)) \Pi(du)$ and
\begin{align*}
    L_\diamond(\theta)(z)&:= 1(a=a^\star)\alpha_\diamond(x)\left\{\ell_{\mathbb{P}}(\theta)(y)-\zeta_\diamond(\theta,x)\right\} + \zeta_\diamond(\theta,x).
\end{align*}
Similarly, let $\zeta_n^j(\theta,x):=\int \ell_{\mathbb{P}}(\theta)(\psi_n^j(u\mymid x)) \Pi(du)$ and
\begin{align*}
    L_n^j(\theta)(z)&:= 1(a=a^\star)\alpha_n^j(x)\left\{\ell_{\mathbb{P}}(\theta)(y)-\zeta_n^j(\theta,x)\right\} + \zeta_n^j(\theta,x).
\end{align*}

\begin{lemma}[Double robustness of loss]\label{lem:lossDR}
    Suppose \cref{cond:strongPositivity,cond:mixedLipschitzWeak,cond:bddLoss,cond:RMexists}. Let $(s_1,s_2)$ belong to the $2$-simplex and
    \begin{align*}
        B_n:= \Cmixedweak^{1/2}  \max_{j\in [2]}\|\alpha_n^j-\alpha_P\|_{L^2(P_X)}d_\Psi(\psi_n^j,\Psi_P).
    \end{align*}
    For any $\theta\in\underline{\Theta}$, $\left|\mathcal{G}_{\mathbb{P}}(\theta)-\sum_{j=1}^2 s_j PL_n^j(\theta)\right|\le B_n \mathcal{G}_{\mathbb{P}}^{1/2}(\theta)$ and also
    \begin{align}
        \frac{9}{10}\mathcal{G}_{\mathbb{P}}(\theta) - \frac{5}{2}B_n^2\le \sum_{j=1}^2 s_j PL_n^j(\theta) \le \frac{11}{10}\mathcal{G}_{\mathbb{P}}(\theta) + \frac{5}{2}B_n^2. \label{eq:lossDRMainBd}
    \end{align}
\end{lemma}
\begin{proof}[Proof of \Cref{lem:lossDR}]
    Let $\alpha_P(x):=1/P(A=a^\star\mymid X=x)$ and $\psi_P\in \Psi_P$. Observe that 
    \begin{align*}
    \left|\mathcal{G}_{\mathbb{P}}(\theta)-\sum_{j=1}^2 s_j PL_n^j(\theta)\right|&= \left|\sum_{j=1}^2 s_j \medint\int [1-\alpha_n^j(x)/\alpha_P(x)]\left[\zeta_n^j(\theta,x)-\zeta_P(\theta,x)\right]P_X(dx)\right| \\
     &\le \max_{j\in [2]}\left\|1-\alpha_n^j/\alpha_P\right\|_{L^2(P_X)} \left[\medint\int \left[\zeta_n^j(\theta,x)-\zeta_P(\theta,x)\right]^2 P_X(dx)\right]^{1/2} \\
     &\le \Cmixedweak^{1/2} \mathcal{G}_{\mathbb{P}}^{1/2}(\theta) \max_{j\in [2]}\left\|1-\alpha_n^j/\alpha_P\right\|_{L^2(P_X)} d_\Psi(\psi_n^j,\Psi_P) \\
     &\le \Cmixedweak^{1/2} \mathcal{G}_{\mathbb{P}}^{1/2}(\theta) \max_{j\in [2]}\left\|\alpha_P-\alpha_n^j\right\|_{L^2(P_X)} d_\Psi(\psi_n^j,\Psi_P),
    \end{align*}
    where the consecutive inequalities hold by convexity paired with Cauchy-Schwarz, \cref{cond:mixedLipschitzWeak}, and H\"{o}lder's inequality paired with the fact that $1/\alpha_P\le 1$. This proves the first claimed inequality.

    For the remaining pair of inequalities, we combine the Peter-Paul inequality $bc\le b^2/(2\epsilon)+c^2\epsilon/2$ (with $\epsilon=5$) with the inequality we just established, yielding
    \begin{align*}
        \sum_{j=1}^2 s_j PL_n^j(\theta)&\ge \mathcal{G}_{\mathbb{P}}(\theta) - B_n\mathcal{G}_{\mathbb{P}}^{1/2}(\theta)\ge \frac{9}{10}\mathcal{G}_{\mathbb{P}}(\theta) - \frac{5}{2}B_n^2, \\
        \sum_{j=1}^2 s_j PL_n^j(\theta)&\le \mathcal{G}_{\mathbb{P}}(\theta) + B_n\mathcal{G}_{\mathbb{P}}^{1/2}(\theta)\le \frac{11}{10}\mathcal{G}_{\mathbb{P}}(\theta) + \frac{5}{2}B_n^2.
    \end{align*}
\end{proof}


\begin{lemma}[Centered losses are similar, as measured by $L^2$ norm of loss]\label{lem:RemVar}
    Fix $j\in [2]$. If \cref{cond:strongPositivity,cond:bddLoss,cond:RMexists,cond:mixedLipschitzStrong}, then, for any $\theta\in\underline{\Theta}$ and $(\alpha_\diamond,\psi_\diamond)\in \{(\alpha_n^j,\psi_P) : \psi_P\in\Psi_P\}\cup \{(\alpha_P,\psi_n^j)\}$, it holds that $\|L_{\diamond}(\theta)-L_n^j(\theta)\|_{L^2(P)}^2\le C_{n\diamond}^j\|\ell_{\mathbb{P}}(\theta)\|_{L^2(\mathbb{P})}^2$ with
    \begin{align*}
        &C_{n\diamond}^j:=\begin{cases}
            \Cpos\Cmixedstrong\|1+\alpha_n^j/\alpha_P\|_{L^\infty(P_X)}d_\Psi^2(\psi_n^j,\Psi_P),&\mbox{ if }(\alpha_\diamond,\psi_\diamond)\in \{(\alpha_n^j,\psi_P) : \psi_P\in\Psi_P\}, \\
            2\|\alpha_n^j-\alpha_P\|_{L^\infty(P_X)}^2\left[1+\Cmixedstrong d_\Psi^2(\psi_n^j,\Psi_P)\right],&\mbox{ if }(\alpha_\diamond,\psi_\diamond)=(\alpha_P,\psi_n^j).
        \end{cases}
    \end{align*}
\end{lemma}
\begin{proof}[Proof of \Cref{lem:RemVar}]
The proof is broken into two cases.

\noindent \textbf{Case 1:} $(\alpha_\diamond,\psi_\diamond)\in \{(\alpha_n^j,\psi_P) : \psi_P\in\Psi_P\}$. In this case $L_\diamond(\theta)(z)-L_n^j(\theta)(z)= [1(a=a^\star)\alpha_n^j(x)-1][\zeta_n^j(\theta,x)-\zeta_P(\theta,x)]$, and so, by the law of total expectation and \cref{cond:strongPositivity,cond:mixedLipschitzStrong},
\begin{align*}
    \|L_\diamond(\theta)-L_n^j(\theta)\|_{L^2(P)}^2&=\medint\int \left\{[1(a=a^\star)\alpha_n^j(x)-1][\zeta_n^j(\theta,x)-\zeta_P(\theta,x)]\right\}^2 \, P(dz) \\
    &= \medint\int (1-2\alpha_n^j/\alpha_P + (\alpha_n^j)^2/\alpha_P)(x)\left[\zeta_n^j(\theta,x)-\zeta_P(\theta,x)\right]^2 \, P(dz) \\
    &\le \Cpos\|1+\alpha_n^j/\alpha_P\|_{L^\infty(P_X)}\medint\int\left[\zeta_n^j(\theta,x)-\zeta_P(\theta,x)\right]^2 \, P(dz) \\
    &\le \Cpos\Cmixedstrong\|1+\alpha_n^j/\alpha_P\|_{L^\infty(P_X)}d_\Psi^2(\psi_n^j,\Psi_P)\|\ell_{\mathbb{P}}(\theta)\|_{L^2(\mathbb{P})}^2.
\end{align*}
\noindent \textbf{Case 2:} $(\alpha_\diamond,\psi_\diamond)=(\alpha_P,\psi_n^j)$. In this case,
\begin{align*}
    L_\diamond(\theta)(z)-L_n^j(\theta)(z)&= 1(a=a^\star)(\alpha_P-\alpha_n^j)(x)[\ell_{\mathbb{P}}(\theta)(y)-\zeta_P(\theta,x)] \\
    &\quad+ 1(a=a^\star)(\alpha_n^j-\alpha_P)(x)[\zeta_n^j(\theta,x)-\zeta_P(\theta,x)].
\end{align*}
Applying the basic inequality $(b+c)^2\le 2(b^2+c^2)$ followed by H\"{o}lder's inequality yields
\begin{align*}
    \|L_{\diamond}(\theta)-L_n^j(\theta)\|_{L^2(P)}^2&\le 2\|\alpha_P-\alpha_n^j\|_{L^\infty(P_X)}^2\medint\int 1(a=a^\star)[\ell_{\mathbb{P}}(\theta)(y)-\zeta_P(\theta,x)]^2 \, P(dz) \\
    &\quad+ 2\|\alpha_P-\alpha_n^j\|_{L^\infty(P_X)}^2\medint\int [\zeta_n^j(\theta,x)-\zeta_P(\theta,x)]^2 \, P(dz).
\end{align*}
By \cref{cond:mixedLipschitzStrong}, the integral in the second term is upper bounded by $\Cmixedstrong d_\Psi^2(\psi_n^j,\Psi_P)\|\ell_{\mathbb{P}}(\theta)\|_{L^2(\mathbb{P})}^2$. Using that conditional means are $L^2$ projections and \eqref{eq:ident}, the integral in the first term upper bounds as follows:
\begin{align*}
    \medint\int 1(a=a^\star)[\ell_{\mathbb{P}}(\theta)(y)-\zeta_P(\theta,x)]^2 \, P(dz)&\le \medint\int 1(a=a^\star)\ell_{\mathbb{P}}^2(\theta)(y) \, P(dz)\le \|\ell_{\mathbb{P}}(\theta)\|_{L^2(\mathbb{P})}^2.
\end{align*}
Putting these bounds together yields the desired upper bound on $\|L_{\diamond}(\theta)-L_n^j(\theta)\|_{L^2(P)}^2$ for the second case.
\end{proof}

\begin{lemma}[Centered losses are similar, as measured by generalization error]\label{lem:RemVarWeak}
    Fix $j\in [2]$. If \cref{cond:strongPositivity,cond:bddLoss,cond:curvature,cond:RMexists,cond:mixedLipschitzWeak}, then, for any $\theta\in\underline{\Theta}$ and $(\alpha_\diamond,\psi_\diamond)\in \{(\alpha_n^j,\psi_P) : \psi_P\in\Psi_P\}\cup \{(\alpha_P,\psi_n^j)\}$, it holds that $\|L_{\diamond}(\theta)-L_n^j(\theta)\|_{L^2(P)}^2\le \widetilde{C}_{n\diamond}^j\mathcal{G}_{\mathbb{P}}(\theta)$ with
    \begin{align*}
        &\widetilde{C}_{n\diamond}^j:=\begin{cases}
            \Cpos\Cmixedweak\|1+\alpha_n^j/\alpha_P\|_{L^\infty(P_X)}d_\Psi^2(\psi_n^j,\Psi_P),&\mbox{ if }(\alpha_\diamond,\psi_\diamond)=(\alpha_n^j,\psi_P), \\
            2\|\alpha_n^j-\alpha_P\|_{L^\infty(P_X)}^2\left[\Ccurv+\Cmixedweak d_\Psi^2(\psi_n^j,\Psi_P)\right],&\mbox{ if }(\alpha_\diamond,\psi_\diamond)=(\alpha_P,\psi_n^j).
        \end{cases}
    \end{align*}
\end{lemma}
\begin{proof}[Proof of \Cref{lem:RemVarWeak}]
    The proof of this result is nearly identical to that of \Cref{lem:RemVar} and so is omitted.
\end{proof}

\begin{lemma}[Bounding norms of $\ell_{\mathbb{P}}$ by generalization error]\label{lem:L2up}
Fix $j\in [2]$. Suppose \cref{cond:bddLoss,cond:RMexists,cond:strongPositivity}. For any $(\alpha_\diamond,\psi_\diamond)\in\{(\alpha_n^j,\psi_P) : \psi_P\in\Psi_P\}\cup \{(\alpha_P,\psi_n^j)\}$ and $\theta\in\underline{\Theta}$, $\|\ell_{\mathbb{P}}(\theta)\|_{L^2(\mathbb{P})}^2\le C_\diamond\|L_\diamond(\theta)\|_{L^2(P)}^2$, where $C_\diamond:=1\vee \|\alpha_P/\alpha_\diamond\|_{L^\infty(P_X)}\le \Cpos$.
\end{lemma}
\begin{proof}[Proof of \Cref{lem:L2up}]
Observe that
\begin{align}
    \|\ell_{\mathbb{P}}(\theta)\|_{L^2(\mathbb{P})}^2&= \medint\int 1(a=a^\star)\alpha_P(x) \ell_{\mathbb{P}}^2(\theta)(y) P(dz) \nonumber \\
    &=\medint\int 1(a=a^\star)\alpha_P(x) \left[\ell_{\mathbb{P}}(\theta)(y)-\zeta_P(\theta,x)\right]^2 P(dz) + \medint\int \zeta_P^2(\theta,x) P(dz), \label{eq:mainL2display}
\end{align}
where the second equality holds by adding and subtracting a term inside the square, expanding the square and noticing that the cross term is zero, and finally using the law of total expectation to replace the $1(a=a^\star)\alpha_P(x)$ in the second term on by $1$. The remainder of the proof is broken into two cases.

\noindent \textbf{Case 1:} $\alpha_\diamond=\alpha_P$. In this case $C_\diamond=1$. Starting with \eqref{eq:mainL2display}, and then using that $\alpha_P=\alpha_\diamond\ge 1$ and applying the law of total expectation, we find that
\begin{align*}
    &\|\ell_{\mathbb{P}}(\theta)\|_{L^2(\mathbb{P})}^2 \\
    &\le\medint\int \left[1(a=a^\star)\alpha_\diamond(x)\left\{\ell_{\mathbb{P}}(\theta)(y)-\zeta_P(\theta,x)\right\}\right]^2 P(dz) + \medint\int \zeta_P^2(\theta,x) P(dz) \\
    &= \medint\int \left[1(a=a^\star)\alpha_\diamond(x)\left\{\ell_{\mathbb{P}}(\theta)(y)-\zeta_P(\theta,x)\right\} +\zeta_P(\theta,x)\right]^2 P(dz) \\
    &= \medint\int \left[1(a=a^\star)\alpha_\diamond(x)\left\{\ell_{\mathbb{P}}(\theta)(y)-\zeta_\diamond(\theta,x)\right\} +\zeta_\diamond(\theta,x)\right]^2 P(dz) \\
    &\quad+ \medint\int \left[1-\alpha_\diamond(x)\right]\left[\zeta_P(\theta,x)-\zeta_\diamond(\theta,x)\right]^2 P(dz) \\
    &\le \medint\int \left[1(a=a^\star)\alpha_\diamond(x)\left\{\ell_{\mathbb{P}}(\theta)(y)-\zeta_\diamond(\theta,x)\right\} +\zeta_\diamond(\theta,x)\right]^2 P(dz) = \|L_\diamond(\theta)\|_{L^2(P)}^2.
\end{align*}

\noindent \textbf{Case 2:} $\alpha_\diamond\not=\alpha_P$. Since $(\alpha_\diamond,\psi_\diamond)\in\{(\alpha_n^j,\psi_P) : \psi_P\in\Psi_P\}\cup \{(\alpha_P,\psi_n^j)\}$, $\psi_\diamond=\psi_P\in \Psi_P$ in this case, and so $\zeta_\diamond=\zeta_P$. Recalling \eqref{eq:mainL2display} and applying \cref{cond:strongPositivity}, H\"{o}lder's inequality and the basic inequality $\|\alpha_P/\alpha_\diamond\|_{L^\infty(P_X)}b + c\le (\|\alpha_P/\alpha_\diamond\|_{L^\infty(P_X)}\vee 1)(b+c)= C_\diamond(b+c)$ yields
\begin{align*}
    \|\ell_{\mathbb{P}}(\theta)\|_{L^2(\mathbb{P})}^2/C_\diamond&\le \medint\int \left[1(a=a^\star)\alpha_\diamond(x)\left\{\ell_{\mathbb{P}}(\theta)(y)-\zeta_\diamond(\theta,x)\right\}\right]^2 P(dz) + \medint\int \left[\zeta_\diamond(\theta,x)\right]^2 P(dz).
\end{align*}
Using the identity $\int f^2 + \int g^2 = \int (f+g)^2 - 2\int fg$ and noticing that $\int fg$ is equal to zero when this identity is applied on the right-hand side above, this shows that $\|\ell_{\mathbb{P}}(\theta)\|_{L^2(\mathbb{P})}^2/C_\diamond\le \|L_\diamond(\theta)\|_{L^2(P)}^2$. Since $\alpha_\diamond\ge 1$, \cref{cond:strongPositivity} yields $C_\diamond\le \Cpos$. 
\end{proof}

\begin{lemma}[Bounding norms of $L_\diamond$ by norms of $\ell_{\mathbb{P}}$]\label{lem:oracleVarBd}
    Fix $j\in [2]$. Suppose $(\alpha_\diamond,\psi_\diamond)\in\{(\alpha_n^j,\psi_P) : \psi_P\in\Psi_P\}\cup \{(\alpha_P,\psi_n^j)\}$ and that \cref{cond:strongPositivity,cond:curvature,cond:RMexists,cond:bddLoss,cond:mixedLipschitzWeak} hold. Then, for any $\theta\in\underline{\Theta}$,
    \begin{align*}
        \|L_\diamond(\theta)\|_{L^2(P)}^2&\le \Cpos\left[ \Ccurv\left\|\alpha_\diamond/\alpha_P\right\|_{L^\infty(P_X)} + \Cmixedweak d_\Psi^2(\psi_\diamond,\Psi_P)\right]\mathcal{G}_{\mathbb{P}}(\theta).
    \end{align*}
\end{lemma}
\begin{proof}[Proof of \Cref{lem:oracleVarBd}]
Let $\Pi_{A,X}[L_\diamond(\theta)](z):=E_P[L_\diamond(\theta)(Z)\mymid A=a,X=x]$ and $\Pi_{X}[L_\diamond(\theta)](z):=E_P[L_\diamond(\theta)(Z)\mymid X=x]$. By the law of total expectation,
\begin{align}
    \|L_\diamond(\theta)\|_{L^2(P)}^2&= \|L_\diamond(\theta)-\Pi_{A,X}[L_\diamond(\theta)]\|_{L^2(P)}^2 + \|\Pi_{X}[L_\diamond(\theta)]\|_{L^2(P)}^2 \nonumber \\
    &\quad+ \|\Pi_{A,X}[L_\diamond(\theta)]-\Pi_{X}[L_\diamond(\theta)]\|_{L^2(P)}^2. \label{eq:Ldiamondnorm}
\end{align}
We begin by bounding the sum of the first two terms on the right, and then we conclude by bounding the third. These calculations will all use that $\Pi_{A,X}[L_\diamond(\theta)](z)=1(a=a^\star)\alpha_{\diamond}(x)[\zeta_P(\theta,x)-\zeta_\diamond(\theta,x)] + \zeta_\diamond(\theta,x)$ and $\Pi_{X}[L_\diamond(\theta)](z)= \zeta_P(\theta,x)$, where the first holds by the law of total expectation and the second by the double robustness of the loss $L_\diamond$. Studying the first term above, we apply the law of total expectation, H\"{o}lder's inequality, and \cref{cond:strongPositivity} to show that
\begin{align*}
    &\|L_\diamond(\theta)-\Pi_{A,X}[L_\diamond(\theta)]\|_{L^2(P)}^2 \\
    &= \medint\int \left[1(a=a^\star)\alpha_{\diamond}(x)\left\{\ell_{\mathbb{P}}(\theta)(y)-\zeta_P(\theta,x)\right\}\right]^2 P(dz) \\
    &= \medint\iint \alpha_\diamond(x)\frac{\alpha_\diamond(x)}{\alpha_P(x)}\left[\ell_{\mathbb{P}}(\theta)(y)-\zeta_P(\theta,x)\right]^2 P_{Y|A,X}(dy|a^\star,x)P_X(dx) \\
    &\le \Cpos \left\|\alpha_\diamond/\alpha_P\right\|_{L^\infty(P_X)}\medint\iint \left[\ell_{\mathbb{P}}(\theta)(y)-\zeta_P(\theta,x)\right]^2 P_{Y|A,X}(dy|a^\star,x)P_X(dx).
\end{align*}
If we add the second term from \eqref{eq:Ldiamondnorm}, $\|\Pi_{X}[L_\diamond(\theta)]\|_{L^2(P)}^2=\int \zeta_P^2(\theta,x) P_X(dx)$, to both sides above and then further upper bound the right-hand side, we find that
\begin{align*}
    &\|L_\diamond(\theta)-\Pi_{A,X}[L_\diamond(\theta)]\|_{L^2(P)}^2 + \|\Pi_{X}[L_\diamond(\theta)]\|_{L^2(P)}^2 \\
    &\le \Cpos \left\|\alpha_\diamond/\alpha_P\right\|_{L^\infty(P_X)}\medint\iint \left\{\left[\ell_{\mathbb{P}}(\theta)(y)-\zeta_P(\theta,x)\right]^2 + \zeta_P^2(\theta,x)\right\} P_{Y|A,X}(dy|a^\star,x)P_X(dx).
\end{align*}
The double integral above equals $\|\ell_{\mathbb{P}}(\theta)\|_{L^2(\mathbb{P})}^2$, and so by \cref{cond:curvature} the left-hand side is bounded by $\Ccurv\Cpos \left\|\alpha_\diamond/\alpha_P\right\|_{L^\infty(P_X)}\mathcal{G}_{\mathbb{P}}(\theta)$.

For the third term in \eqref{eq:Ldiamondnorm}, we use the identities we derived for $\Pi_{A,X}[L_\diamond(\theta)]$ and $\Pi_{X}[L_\diamond(\theta)]$, expand a square, and apply the law of total expectation to find that
\begin{align*}
    &\|\Pi_{A,X}[L_\diamond(\theta)]-\Pi_{X}[L_\diamond(\theta)]\|_{L^2(P)}^2 \\
    &= \medint\int \left[1-1(a=a^\star)\alpha_\diamond(x)\right]^2 \left[\zeta_\diamond(\theta,x)-\zeta_P(\theta,x)\right]^2 \,P(dz) \\
    &= \medint\int \left[1 - 2\alpha_\diamond(x)/\alpha_P(x) + \alpha_\diamond^2(x)/\alpha_P(x)\right] \left[\zeta_\diamond(\theta,x)-\zeta_P(\theta,x)\right]^2 \,P(dz).
\end{align*}
We will show that the right-hand side is no more than $\Cpos\Cmixedweak d_\Psi^2(\psi_\diamond,\Psi_P)\mathcal{G}_{\mathbb{P}}(\theta)$. If $\zeta_\diamond=\zeta_P$, then this claimed bound is zero and so is the right-hand side above; hence, the bound is valid. Otherwise, $\alpha_\diamond=\alpha_P$, and so $1 - 2\alpha_\diamond/\alpha_P + \alpha_\diamond^2/\alpha_P=\alpha_\diamond-1\le \Cpos$. Combining this with \cref{cond:mixedLipschitzWeak} shows that the right-hand side is indeed no more than $\Cpos\Cmixedweak d_\Psi^2(\psi_\diamond,\Psi_P)\mathcal{G}_{\mathbb{P}}(\theta)$.

Combining our bounds for the three terms in \eqref{eq:Ldiamondnorm} gives the desired result.
\end{proof}

In the following lemma, we let
\begin{align*}
    K_0^j&:= 2\min\Bigg\{\Ccurv\Cpos\left\|\alpha_n^j/\alpha_P\right\|_{L^\infty(P_X)} + \Cmixedweak(1+\Cpos)\Cpos d_\Psi^2(\psi_n^j,\Psi_P),  \\
    &\hspace{7.25em}\Cpos\left[ \Ccurv + \Cmixedweak d_\Psi^2(\psi_n^j,\Psi_P)\right] + 2\|\alpha_n^j-\alpha_P\|_{L^\infty(P_X)}^2\left[\Ccurv+\Cmixedweak d_\Psi^2(\psi_n^j,\Psi_P)\right]\Bigg\}.
\end{align*}
Note that $K_0$ from \eqref{eq:K0def} is then equal to $\max_j K_0^j$.
\begin{lemma}[Bounding $L^2$ norm of loss by generalization error]\label{lem:lossL2bd}
    Fix $j\in [2]$ and $\theta\in\underline{\Theta}$. If \cref{cond:strongPositivity,cond:RMexists,cond:bddLoss,cond:mixedLipschitzWeak,cond:curvature}, then $\|L_n^j(\theta)\|_{L^2(P)}^2\le K_0^j\mathcal{G}_{\mathbb{P}}(\theta)$.
\end{lemma}
\begin{proof}[Proof of \Cref{lem:lossL2bd}]
    By \Cref{lem:oracleVarBd,lem:RemVarWeak}, the triangle inequality, and the basic inequality $(b+c)^2\le 2(b^2+c^2)$,
    \begin{align*}
        \|L_n^j(\theta)\|_{L^2(P)}^2\le 2 \left(\Cpos\left[ \Ccurv\left\|\alpha_\diamond/\alpha_P\right\|_{L^\infty(P_X)} + \Cmixedweak d_\Psi^2(\psi_\diamond,\Psi_P)\right] + \widetilde{C}_{n\diamond}^j\right)\mathcal{G}_{\mathbb{P}}(\theta)
    \end{align*}
    for any $(\alpha_\diamond,\psi_\diamond)\in \{(\alpha_n^j,\psi_P) : \psi_P\in\Psi_P\}\cup \{(\alpha_P,\psi_n^j)\}$. In particular, $(\alpha_\diamond,\psi_\diamond)$ can be chosen to minimize the right-hand side. Upper bounding further to simplify the expression then yields $\|L_n^j(\theta)\|_{L^2(P)}^2\le K_0^j\mathcal{G}_{\mathbb{P}}(\theta)$.
\end{proof}




For $j\in [2]$ and $\delta>0$, let $\mathcal{L}_{n,\delta}^j:=\{L_n^j(\theta) : \theta\in\underline{\Theta},\|L_n^j(\theta)\|_{L^2(P)}\le \delta\}$.
\begin{lemma}[Constant envelope for $\mathcal{L}_{n,\delta}^j$]\label{lem:envelope}
    Fix $j\in [2]$. Suppose \cref{cond:bddLoss,cond:RMexists,cond:strongPositivity}. Conditionally on $\mathcal{Z}_{3-j}$, the following holds for all $z\in \mathcal{Z}$: $\sup_{\theta\in\underline{\Theta}}|L_n^j(\theta)(z)|\le 2\Closs\Cpos$. 
\end{lemma}
\begin{proof}[Proof of \Cref{lem:envelope}]
    First applying Jensen's inequality and then applying \cref{cond:bddLoss,cond:strongPositivity} yields
    \begin{align*}
        &\sup_{\theta\in\underline{\Theta}} |L_n^j(\theta)(z)| \\
        &\le \sup_{\theta\in\underline{\Theta}}\medint\int \left[1(a=a^\star)\alpha_n^j(x)\left|\ell_{\mathbb{P}}(\theta)(y)\right| + \left|1-1(a=a^\star)\alpha_n^j(x)\right|\left|\ell_{\mathbb{P}}(\theta)(\psi_n^j(u|x))\right|\right]\Pi(du) \\
        &\le \Cpos\Closs + \max\{1,\Cpos-1\}\Closs\le 2\Cpos\Closs.
    \end{align*}
\end{proof}

\begin{lemma}[Preliminary entropy integral bound]\label{lem:prelimEntropy}
Fix $j\in [2]$. Suppose \cref{cond:RMexists,cond:strongPositivity,cond:bddLoss}. For any $\widetilde{\Theta}\subseteq\underline{\Theta}$ and $\delta>0$, $J(8\Cpos\delta,L_n^j(\widetilde{\Theta}))\le 8\Cpos J(\delta,\ell_{\mathbb{P}}(\widetilde{\Theta}))$.
\end{lemma}
\begin{proof}[Proof of \Cref{lem:prelimEntropy}]
    The bound is trivial when $J(\delta,\ell_{\mathbb{P}}(\widetilde{\Theta})\mymid \Closs,L^2)=\infty$, so we focus on the case where this quantity is finite. 
    Let $Q$ be a distribution on $\mathcal{X}\times\mathcal{A}\times\mathcal{Y}$ with support $\{(x_i,a_i,y_i)\}_{i=1}^v$. Let $\mathbb{Q}^{(1)}$ be the distribution supported on $\{y_i\}_{i=1}^v$ that satisfies $\mathbb{Q}^{(1)}\{y_i\}=Q\{(x_{i'},a_{i'},y_{i'}) : y_{i'}=y_i\}$ and $\mathbb{Q}^{(2)}$ be the distribution on $\mathcal{Y}$ satisfying $\mathbb{Q}^{(2)}(\mathcal{B})=\sum_{i=1}^v Q\{(x_i,a_i,y_i)\}\int 1\{\psi_n^j(u|x_i)\in\mathcal{B}\} \Pi(du)$ for all measurable $\mathcal{B}\subseteq\mathcal{Y}$. Further let $\mathbb{Q}^{(3)}=(\mathbb{Q}^{(1)}+\mathbb{Q}^{(2)})/2$.

    Let $\{f_i\}_{i=1}^{N}$ be a minimal $4\epsilon$-cover of $\ell_{\mathbb{P}}(\widetilde{\Theta})$ with respect to $L^2(\mathbb{Q}^{(3)})$ and
    \begin{align*}
        g_i(x,a,y):=\medint\int \big[1(a=a^\star)\alpha_n^j(x)\left\{f_i(y)-f_i(\psi_n^j(u|x))\right\} + f_i(\psi_n^j(u|x))\big]\,\Pi(du).
    \end{align*}
    We will show that $\{g_i\}_{i=1}^N$ is an $8\Cpos\epsilon$-cover of $L_n^j(\widetilde{\Theta})$ with respect to $L^2(Q)$. To this end, fix $\theta\in\widetilde{\Theta}$ and let $f_i$ be such that $\|\ell_{\mathbb{P}}(\theta)-f_i\|_{L^2(\mathbb{Q}^{(3)})}\le 4\epsilon$. Observe that
    \begin{align*}
    L_n^j(\theta)(z)-g_i(z)&= 1(a=a^\star)\alpha_n^j(x)[\ell_{\mathbb{P}}(\theta)(y)-f_i(y)] \\
    &\quad+ \medint\int \left[1-1(a=a^\star)\alpha_n^j(x)\right]\left[\ell_{\mathbb{P}}(\theta)(\psi_n^j(u|x))-f_i(\psi_n^j(u|x))\right]\,\Pi(du).
    \end{align*}
    Combining this with the triangle inequality, Jensen's inequality, and \cref{cond:strongPositivity} shows that $\|L_n^j(\theta)-g_i\|_{L^2(Q)}\le \Cpos(\|\ell_{\mathbb{P}}(\theta)-f_i\|_{L^2(\mathbb{Q}^{(1)})} + \|\ell_{\mathbb{P}}(\theta)-f_i\|_{L^2(\mathbb{Q}^{(2)})})$. Squaring both sides, applying the inequality $(b+c)^2\le 2(b^2+c^2)$, and recalling the definitions of $\mathbb{Q}^{(3)}$ and $f_i$ then shows that $\|L_n^j(\theta)-g_i\|_{L^2(Q)}^2\le 2\Cpos^2\|\ell_{\mathbb{P}}(\theta)-f_i\|_{L^2(\mathbb{Q}^{(1)}+\mathbb{Q}^{(2)})}^2=4\Cpos^2\|\ell_{\mathbb{P}}(\theta)-f_i\|_{L^2(\mathbb{Q}^{(3)})}^2=64\Cpos^2\epsilon^2$. Hence, $\{g_i\}_{i=1}^N$ is an $8\Cpos\epsilon$-cover of $L_n^j(\widetilde{\Theta})$ with respect to $L^2(Q)$. As $\epsilon$ was arbitary, we have shown that, for any $\gamma\in (0,\delta)$,
    \begin{align}
        \medint\int_\gamma^\delta \sqrt{1+\log N(8\Cpos\epsilon,L_n^j(\widetilde{\Theta}),L^2(Q))}\,d\epsilon\le \medint\int_\gamma^\delta \sqrt{1+\log N(4\epsilon,\ell_{\mathbb{P}}(\widetilde{\Theta}),L^2(\mathbb{Q}^{(3)}))}\,d\epsilon. \label{eq:coveringNumBd}
    \end{align}
    
    The next part of this proof deals with the fact that the uniform entropy integral $J$ is defined using $L^2$ covering numbers with respect to finitely supported distributions, a property that $\mathbb{Q}^{(3)}$ may not satisfy. In particular, we show that there exists a finitely supported distribution $\mathbb{Q}$ on $\mathcal{Y}$ such that
    \begin{align}
    \medint\int_\gamma^\delta \sqrt{1+\log N(4\epsilon,\ell_{\mathbb{P}}(\widetilde{\Theta}),L^2(\mathbb{Q}^{(3)}))}\,d\epsilon\le \medint\int_\gamma^\delta \sqrt{1+\log N(\epsilon,\ell_{\mathbb{P}}(\widetilde{\Theta}),L^2(\mathbb{Q}))}\,d\epsilon. \label{eq:QfinitelySupp}
    \end{align}
    This argument is based on the hint given on Problem 2.5.1 of \cite{vaart2023empirical}, which leverages the following relationship between covering numbers $N$ and packing numbers $M$: $N(\epsilon)\le M(\epsilon)\le N(\epsilon/2)$ \citep[][page 147]{vaart2023empirical}. In particular, this relationship shows it suffices to exhibit a finitely supported distribution $\mathbb{Q}$ on $\mathcal{Y}$ that satisfies $M(4\epsilon,\ell_{\mathbb{P}}(\widetilde{\Theta}),L^2(\mathbb{Q}^{(3)}))\le M(2\epsilon,\ell_{\mathbb{P}}(\widetilde{\Theta}),L^2(\mathbb{Q}))$ for all $\epsilon\in [\gamma,\delta]$. To this end, for each $m\in \mathcal{I}:=\{M(4\gamma,\ell_{\mathbb{P}}(\widetilde{\Theta}),L^2(\mathbb{Q}^{(3)})),M(4\gamma,\ell_{\mathbb{P}}(\widetilde{\Theta}),L^2(\mathbb{Q}^{(3)}))+1,\ldots,M(4\delta,\ell_{\mathbb{P}}(\widetilde{\Theta}),L^2(\mathbb{Q}^{(3)}))\}$, we let $\{f_{m,j}\}_{j=1}^m$ denote a maximal packing of $\ell_{\mathbb{P}}(\widetilde{\Theta})$ with respect to $L^2(\mathbb{Q}^{(3)})$. Let $\{Y_{i'}^\star(\omega)\}_{i'=1}^\infty$ be a sequence of iid draws from $\mathbb{Q}^{(3)}$ and denote the empirical distribution of its first $k$ draws by $\mathbb{Q}_k(\omega)$. By the strong law of large numbers, for all $\omega$ belonging to a probability one set $\Omega_\gamma$, there exists $K(\omega)<\infty$ such that, for all $m\in\mathcal{I}$, $\|f_{m,i}-f_{m,i'}\|_{L^2(\mathbb{Q}_{K(\omega)}(\omega))}\ge \|f_{m,i}-f_{m,i'}\|_{L^2(\mathbb{Q}^{(3)})}-2\gamma$. Fixing some $\omega\in\Omega_\gamma$ and letting $\mathbb{Q}$ be the finitely supported distribution $\mathbb{Q}_{K(\omega)}(\omega)$ then shows that $M(4\epsilon,\ell_{\mathbb{P}}(\widetilde{\Theta}),L^2(\mathbb{Q}^{(3)}))\le M(2\epsilon,\ell_{\mathbb{P}}(\widetilde{\Theta}),L^2(\mathbb{Q}))$ for all $\epsilon\in [\gamma,\delta]$, which establishes \eqref{eq:QfinitelySupp}.

    Combining \eqref{eq:coveringNumBd} and \eqref{eq:QfinitelySupp}, taking a limit as $\gamma\downarrow 0$, and then taking a supremum on the right over all finitely supported distributions $\mathbb{Q}$ on $\mathcal{Y}$ followed by a supremum on the left over all finitely supported distributions $Q$ on $\mathcal{X}\times\mathcal{A}\times\mathcal{Y}$ yields
    \begin{align*}
        \sup_Q \medint\int_0^\delta \sqrt{1+\log N(8\Cpos\epsilon,L_n^j(\widetilde{\Theta}),L^2(Q))}\,d\epsilon\le \sup_{\mathbb{Q}}\medint\int_0^\delta \sqrt{1+\log N(\epsilon,\ell_{\mathbb{P}}(\widetilde{\Theta}),L^2(\mathbb{Q}))}\,d\epsilon.
    \end{align*}
    Applying the change of variables $\tilde{\epsilon}=8\Cpos\epsilon$ on the left yields the result.
\end{proof}

\begin{lemma}[Entropy integral bounds for $\mathcal{L}_{n,\delta}^j$]\label{lem:entropy}
Fix $j\in [2]$ and $(\alpha_\diamond,\psi_\diamond)\in \{(\alpha_n^j,\psi_P) : \psi_P\in\Psi_P\}\cup \{(\alpha_P,\psi_n^j)\}$. Both of the following bounds are valid under the stated conditions:
\begin{align*}
    J(\delta,\mathcal{L}_{n,\delta}^j)&\le \begin{cases}
        8\Cpos J(\delta/(8\Cpos),\ell_{\mathbb{P}}(\underline{\Theta})),&\mbox{ if \cref{cond:RMexists,cond:strongPositivity,cond:bddLoss} hold,} \\
        8\Cpos J(2\Cpos^{1/2} \delta,\ell_{\mathbb{P}}(\underline{\Theta}_{2\Cpos^{1/2}\delta})),&\mbox{ if \cref{cond:RMexists,cond:strongPositivity,cond:bddLoss,cond:oneGoodNuis,cond:mixedLipschitzStrong} hold.}
    \end{cases}
\end{align*}
\end{lemma}
\begin{proof}[Proof of \Cref{lem:entropy}]
    We begin by proving the first bound. By \Cref{lem:prelimEntropy} with $\widetilde{\Theta}=\{\theta\in\underline{\Theta} : \|L_n^j(\theta)\|_{L^2(P)}\le \delta\}$, $J(\delta,\mathcal{L}_{n,\delta}^j)\le 8\Cpos J(\delta/(8\Cpos),\ell_{\mathbb{P}}(\widetilde{\Theta}))$. Since $\ell_{\mathbb{P}}(\widetilde{\Theta})\subseteq \ell_{\mathbb{P}}(\underline{\Theta})$, $J(\delta/(8\Cpos),\ell_{\mathbb{P}}(\widetilde{\Theta}))\le J(\delta/(8\Cpos),\ell_{\mathbb{P}}(\underline{\Theta}))$. Combining the inequalities from the preceding two sentences gives the result.
    
    We now prove the second bound. If \ref{cond:oneGoodNuisPsi} from \cref{cond:oneGoodNuis} holds we take $(\alpha_\diamond,\psi_\diamond)=(\alpha_n^j,\psi_P)$ for some $\psi_P\in\Psi_P$, and otherwise we take $(\alpha_\diamond,\psi_\diamond)=(\alpha_P,\psi_n^j)$. For any $\theta\in\underline{\Theta}$, \Cref{lem:L2up}, the triangle inequality, \Cref{lem:RemVar}, and \cref{cond:oneGoodNuis} yield
    \begin{align*}
   \|\ell_{\mathbb{P}}(\theta)\|_{L^2(\mathbb{P})}&\le C_\diamond^{1/2} \|L_\diamond(\theta)\|_{L^2(P)}\le C_\diamond^{1/2} \left(\|L_n^j(\theta)\|_{L^2(P)}+\|L_n^j(\theta)-L_\diamond(\theta)\|_{L^2(P)}\right) \\
   &\le C_\diamond^{1/2} \left(\|L_n^j(\theta)\|_{L^2(P)}+(C_{n\diamond}^j)^{1/2}\|\ell_{\mathbb{P}}(\theta)\|_{L^2(\mathbb{P})}\right) \\
   &\le C_\diamond^{1/2} \|L_n^j(\theta)\|_{L^2(P)}+\|\ell_{\mathbb{P}}(\theta)\|_{L^2(\mathbb{P})}/2.
    \end{align*}
    Hence, $\|\ell_{\mathbb{P}}(\theta)\|_{L^2(\mathbb{P})}\le 2 C_\diamond^{1/2}\|L_n^j(\theta)\|_{L^2(P)}$. Since $C_\diamond^{1/2}\le \Cpos^{1/2}$, this yields the bound $\|\ell_{\mathbb{P}}(\theta)\|_{L^2(\mathbb{P})}\le 2 \Cpos^{1/2}\|L_n^j(\theta)\|_{L^2(P)}$.
    
    Hence, letting $\underline{\Theta}_{2\Cpos^{1/2}\delta}:=\{\theta\in\underline{\Theta} : \|\ell_{\mathbb{P}}(\theta)\|_{L^2(\mathbb{P})}\le  2\Cpos^{1/2}\delta\}$, the following holds for any $\delta>0$:
    \begin{align*}
        \mathcal{L}_{n,\delta}^j&:=\{L_n^j(\theta) : \theta\in\underline{\Theta},\|L_n^j(\theta)\|_{L^2(P)}\le \delta\}\subseteq L_n^j(\underline{\Theta}_{2\Cpos^{1/2}\delta}).
    \end{align*}
    Hence, $J(\delta,\mathcal{L}_{n,\delta}^j)\le J(\delta,L_n^j(\underline{\Theta}_{2\Cpos^{1/2}\delta}))$. Combining this with \Cref{lem:prelimEntropy} with $\widetilde{\Theta}=\underline{\Theta}_{2\Cpos^{1/2}\delta}$ gives $J(\delta,\mathcal{L}_{n,\delta}^j)\le 8\Cpos J(\delta/(8\Cpos),\ell_{\mathbb{P}}(\underline{\Theta}_{2\Cpos^{1/2}\delta}))$. By the monotonicity of $J$ in its first argument and the fact that $1/(8\Cpos)<2\Cpos^{1/2}$, this implies that $J(\delta,\mathcal{L}_{n,\delta}^j)\le 8\Cpos J(2\Cpos^{1/2} \delta,\ell_{\mathbb{P}}(\underline{\Theta}_{2\Cpos^{1/2}\delta}))$.
\end{proof}

\begin{proof}[Proof of \Cref{thm:excRiskOSLFormal}]
    For now fix $j\in [2]$. Let $\bar{j}:=3-j$, $n_{\bar{j}}:=|\mathcal{Z}_{\bar{j}}|$, and $P_{n,\bar{j}}$ denote the empirical distribution of the observations in $\mathcal{Z}_{\bar{j}}$. 
    By \Cref{lem:localMaximal,lem:envelope,lem:entropy}, the following holds for any $\delta>0$:
    \begin{align}
        n_{\bar{j}}^{1/2}E\left\|P_{n,j}-P\right\|_{\mathcal{L}_{n,\delta}^j}&\le K J(\delta,\mathcal{L}_{n,\delta}^j)\left[ 1 +\frac{2\Closs\Cpos J(\delta,\mathcal{L}_{n,\delta}^j)}{\delta^2 n_{\bar{j}}^{1/2}}\right]\nonumber \\
        &\le 8K \Cpos J(\delta/(8\Cpos),\ell_{\mathbb{P}}(\underline{\Theta}))\left[ 1 +\frac{16\Closs\Cpos^2 J(\delta/(8\Cpos),\ell_{\mathbb{P}}(\underline{\Theta}))}{\delta^2 n_{\bar{j}}^{1/2}}\right]. \label{eq:critRadBoundDG}
    \end{align}
    By \eqref{eq:doubleGenDeltaN} and the fact that $n_{\bar{j}}\ge\lfloor n/2\rfloor$, the right-hand side is no more than $n_{\bar{j}}^{1/2}\delta^2$ when $\delta= \delta_n$ and, by \Cref{lem:monotoneUB}, the same holds true when $\delta\ge \delta_n$.

    Let $\eta^j:=1/\max\{2,270 K_0^j\}\in (0,1)$, which is deterministic conditionally on $\mathcal{Z}_n^j$ and the $j$ superscript on $\eta^j$ denotes the current fold rather than exponentiation. Applying the consequence of Talagrand's inequality given in Lem.~3.5.9 of \cite{vaart2023empirical} to the class $\{-L_n^j(\theta) : \theta\in\underline{\Theta}\}$, which is bounded by \Cref{lem:envelope}, shows that, w.p. at least $1-e^{-[s+\log 4]}$ conditionally on $\mathcal{Z}_n^j$,
    \begin{align*}
        P L_n^j(\theta)&\le P_{n,\bar{j}} L_n^j(\theta) + 135 \eta^j \|L_n^j(\theta)\|_{L^2(P)}^2 + 135\frac{\delta_n^2}{\eta^j} + 135\left(2\Cpos\Closs + \frac{1}{\eta^j}\right)\frac{s+\log 4}{n_{\bar{j}}} \;\forall\theta\in \underline{\Theta}.
    \end{align*}
    Similarly, applying  Lem.~3.5.9 of \cite{vaart2023empirical} to $\{L_n^j(\theta) : \theta\in\underline{\Theta}\}$ shows that, w.p. at least $1-e^{-[s+\log 4]}$  conditionally on $\mathcal{Z}_n^j$,
    \begin{align*}
        P_{n,\bar{j}}L_n^j(\theta)&\le PL_n^j(\theta) + 135 \eta^j \|L_n^j(\theta)\|_{L^2(P)}^2 + 135\frac{\delta_n^2}{\eta^j} + 135\left(2\Cpos\Closs + \frac{1}{\eta^j}\right)\frac{s+\log 4}{n_{\bar{j}}}\;\forall\theta\in \underline{\Theta}.
    \end{align*}
    Since each of the above inequalities holds w.p. at least $1-e^{-[s+\log 4]}$  conditionally on $\mathcal{Z}_n^j$, they each also hold marginally w.p. at least $1-e^{-[s+\log 4]}$.
    
    Hereafter we work on the event where the above two inequalities hold for both $j\in [2]$, which occurs with marginal probability at least $1-e^{-s}$ by a union bound. For the first of these inequalities, we are only concerned with the fact that it holds for $\theta=\theta_n$ and, for the second, with the fact that it holds for a generic fixed $\underline{\theta}\in\underline{\Theta}$. Applying \Cref{lem:lossL2bd} to the inequalities and using that $135 \eta^j K_0^j\le 1/2$ shows that, for each $j$,
    \begin{align*}
        P L_n^j(\theta_n)&\le P_{n,\bar{j}} L_n^j(\theta_n) + \frac{1}{2}\mathcal{G}_{\mathbb{P}}(\theta_n) + 135\frac{\delta_n^2}{\eta^j} + 135\left(2\Cpos\Closs + \frac{1}{\eta^j}\right)\frac{s+\log 4}{n_{\bar{j}}}, \\
        P_{n,\bar{j}}L_n^j(\underline{\theta})&\le PL_n^j(\underline{\theta}) + \frac{1}{2}\mathcal{G}_{\mathbb{P}}(\underline{\theta}) + 135\frac{\delta_n^2}{\eta^j} + 135\left(2\Cpos\Closs + \frac{1}{\eta^j}\right)\frac{s+\log 4}{n_{\bar{j}}}.
    \end{align*}
    Letting $\eta:=1/\max\{2,270 K_0\}=\min_j \eta^j$, the instances of $\eta^j$ on the right-hand sides above can be replaced by $\eta$, at the possible cost of looser inequalities. Moreover, since the above hold for both $j\in [2]$, $\sum_{j=1}^2 \frac{n_{\bar{j}}}{n} P_{n,\bar{j}} L_n^j(\theta)=R_n(\theta)$, and $\sum_{j=1}^2 \frac{n_{\bar{j}}}{n}=1$, this shows that
    \begin{align*}
        \sum_{j=1}^2 \frac{n_{\bar{j}}}{n} P L_n^j(\theta_n)&\le R_n(\theta_n) + \frac{1}{2}\mathcal{G}_{\mathbb{P}}(\theta_n) + 135\frac{\delta_n^2}{\eta} + 270\left(2\Cpos\Closs + \frac{1}{\eta}\right)\frac{s+\log 4}{n} \\
        R_n(\underline{\theta})&\le \sum_{j=1}^2 \frac{n_{\bar{j}}}{n}PL_n^j(\underline{\theta}) + \frac{1}{2}\mathcal{G}_{\mathbb{P}}(\underline{\theta}) + 135\frac{\delta_n^2}{\eta} + 270\left(2\Cpos\Closs + \frac{1}{\eta}\right)\frac{s+\log 4}{n}.
    \end{align*}
    Combining the above inequalities with the fact that $R_n(\theta_n)\le R_n(\underline{\theta})$ and \Cref{lem:lossDR} yields
    \begin{align*}
        \frac{9}{10}\mathcal{G}_{\mathbb{P}}(\theta_n)&\le \frac{11}{10}\mathcal{G}_{\mathbb{P}}(\underline{\theta}) + \frac{1}{2}[\mathcal{G}_{\mathbb{P}}(\theta_n)+\mathcal{G}_{\mathbb{P}}(\underline{\theta})] + 270\frac{\delta_n^2}{\eta} + 540\left(2\Cpos\Closs + \frac{1}{\eta}\right)\frac{s+\log 4}{n} + 5B_n^2,
    \end{align*}
    where $B_n$ is as defined in that lemma. Subtracting $\mathcal{G}_{\mathbb{P}}(\theta_n)/2$ from both sides and then multiplying both sides by $5/2$ yields
    \begin{align*}
        \mathcal{G}_{\mathbb{P}}(\theta_n)&\le 4\mathcal{G}_{\mathbb{P}}(\underline{\theta}) + \frac{5}{2}270\frac{\delta_n^2}{\eta}+ 1350\left(2\Cpos\Closs + \frac{1}{\eta}\right)\frac{s+\log 4}{n} + \frac{25}{2}B_n^2.
    \end{align*}
    Plugging in the value of $\eta$ and further upper bounding the right-hand side yields
    \begin{align*}
        \mathcal{G}_{\mathbb{P}}(\theta_n)&\le 4\mathcal{G}_{\mathbb{P}}(\underline{\theta}) + 1350\delta_n^2(1\vee 135 K_0)+ 2700\left[\Cpos\Closs + (1\vee 135 K_0)\right]\frac{s+\log 4}{n} + 13B_n^2.
    \end{align*}
    Choosing $\underline{\theta}$ such that $4\mathcal{G}_{\mathbb{P}}(\underline{\theta})\le 4\inf_{\theta\in\underline{\Theta}}\mathcal{G}_{\mathbb{P}}(\theta) + K_2 (2-\log 4)/n$ yields the claimed bound, \eqref{eq:excRiskBdOSL}.

    The double robustness of the final term in the bound follows directly from the fact that $\alpha_n^j=\alpha_P$ implies $\|\alpha_n^j-\alpha_P\|_{L^2(P_X)}^2=0$ and $\psi_n^j\in \Psi_P$ implies $d_\Psi^2(\psi_n^j,\Psi_P)=0$.
\end{proof}

\begin{proof}[Proof of \Cref{thm:excRiskOSLLocal}]
    The proof follows in almost exactly the same way as that of \Cref{thm:excRiskOSLFormal}, except that the second half of \Cref{lem:entropy} is used when establishing \eqref{eq:critRadBoundDG}, yielding the alternative bound
    \begin{align*}
    n_{\bar{j}}^{1/2}E\left\|P_{n,j}-P\right\|_{\mathcal{L}_{n,\delta}^j}&\le 8K \Cpos J(2\Cpos^{1/2} \delta,\ell_{\mathbb{P}}(\underline{\Theta}_{2\Cpos^{1/2}\delta}))\left[ 1 +\frac{16\Closs\Cpos^2 J(2\Cpos^{1/2} \delta,\ell_{\mathbb{P}}(\underline{\Theta}_{2\Cpos^{1/2}\delta}))}{\delta^2 n_{\bar{j}}^{1/2}}\right].
    \end{align*}
    By \eqref{eq:doubleGenDeltaNLocal} and the fact that $n_{\bar{j}}\ge\lfloor n/2\rfloor$, the right-hand side is no more than $n_{\bar{j}}^{1/2}\delta_n^2$ when $\delta= \delta_n$. The rest of the proof follows in the same way as that of \Cref{thm:excRiskOSLFormal}.
\end{proof}

\subsection{Proof of minimax lower bound (\Cref{thm:minimax})}
When proving this result, we assume that every conditional and marginal probability measure of interest, such as $P$ or its conditionals, is the pushforward of $\nu=\mathrm{Unif}[0,1]$ with respect to some transport map, which holds under mild regularity conditions \citep[Lemma 4.22 of][]{kallenberg2021foundations}.
\begin{proof}[Proof of \Cref{thm:minimax}]
    For now fix $T\in\mathcal{T}$, $\mathbb{P}\in\mathcal{P}^\star$, and $P\in\mathcal{P}(\mathbb{P})$. Define the function $g_P$ so that, when $V\sim \nu$, $g_P(y_{[n]}^\star,V)$ follows the conditional distribution of $Z_{[n]}\mymid (1[A_i=a^\star]Y_i)_{i=1}^n=(1[A_i=a^\star]y_i^\star)_{i=1}^n$ implied by $P^n$. Define $f_1 : [0,1]\rightarrow [0,1]$ and $f_2 : [0,1]\rightarrow [0,1]$ so that, when $V\sim\nu$, $f_1(V)\overset{d}{=}f_2(V)\overset{d}{=}V$ and $f_1(V)\indep f_2(V)$ --- see the proof of Lem.~4.21 in \cite{kallenberg2021foundations} for a construction. By the definitions of $g_P$, $f_1$, and $f_2$, we have that $(g_P(Y_{[n]}^\star,f_1(V)),f_2(V))\overset{d}{=}(Z_{[n]},V')$ when $(Y_{[n]}^\star,V)\sim \mathbb{P}^n\times\nu$ and $(Z_{[n]},V')\sim P^n\times \nu$.

    We now define an element of $T_{P,T}^\star$ of $\mathcal{T}^\star$, where the subscripts indicate that this choice depends on $P$ and $T$. Specifically, let $T_{P,T}^\star(y_{[n]}^\star,v):= T(g_P(y_{[n]}^\star,f_1(v)),f_2(v))$. By the conclusion of the preceding paragraph,
    \begin{align*}
        E_{(Z_{[n]},V)\sim P^n\times \nu}\left[D\left(\mathbb{P},T(Z_{[n]},V)_{\sharp}\Pi\right)\right]&= E_{(Y_{[n]}^\star,V)\sim \mathbb{P}^n\times \nu}\left[D\left(\mathbb{P},T_{P,T}^\star(Y_{[n]}^\star,V)_{\sharp}\Pi\right)\right].
    \end{align*}
    Taking a supremum on both sides over $P\in\mathcal{P}(\mathbb{P})$ followed by one over $\mathbb{P}\in\mathcal{P}^\star$ yields
    \begin{align}
        &\sup_{\mathbb{P}\in\mathcal{P}^\star}\sup_{P\in\mathcal{P}(\mathbb{P})} E_{(Z_{[n]},V)\sim P^n\times \nu}\left[D\left(\mathbb{P},T(Z_{[n]},V)_{\sharp}\Pi\right)\right] \nonumber \\
        &= \sup_{\mathbb{P}\in\mathcal{P}^\star}\sup_{P\in\mathcal{P}(\mathbb{P})} E_{(Y_{[n]}^\star,V)\sim \mathbb{P}^n\times \nu}\left[D\left(\mathbb{P},T_{P,T}^\star(Y_{[n]}^\star,V)_{\sharp}\Pi\right)\right]. \label{eq:comparingMaximalRisks}
    \end{align}
    Since $T_{P,T}^\star\in\mathcal{T}^\star$, the right-hand side is lower bounded by 
    \begin{align*}
        \inf_{T^\star\in\mathcal{T}^\star}\sup_{\mathbb{P}\in\mathcal{P}^\star}\sup_{P\in\mathcal{P}(\mathbb{P})} E_{(Y_{[n]}^\star,V)\sim \mathbb{P}^n\times \nu}\left[D\left(\mathbb{P},T^\star(Y_{[n]}^\star,V)_{\sharp}\Pi\right)\right],
    \end{align*}
    which is equal to the left-hand side of the display in the theorem statement since the expectation does not depend on $P$. Plugging the above lower bound into the right-hand side of \eqref{eq:comparingMaximalRisks} and then taking an infimum over $T\in\mathcal{T}$ on the left gives the result.
\end{proof}

\section{Beyond empirical risk minimization: a generic generalization bound}\label{app:beyondERM}

\nameref*{alg:double} can be run with learners other than empirical risk minimizers, such as random forests or neural networks trained with stochastic gradient descent. Theoretical guarantees can also be derived in these cases. For instance, a recent work establishes convergence guarantees for stochastic gradient algorithms in the presence of nuisance parameters, showing that stochastic gradient descent can still converge under appropriate conditions \citep{yu2025stochastic}. 

Here, we focus on the general case where a black-box learner selects $\theta_n$ based on the data. This bound motivates having that learner make the population counterpart to $R_n$ small, even if done through means other than empirical risk minimization. This population counterpart is $PL_n(\theta)$, where
\begin{align}
    L_n(\theta)(z)&:= \frac{1}{2}\sum_{j=1}^2\int \big[1(a=a^\star)\alpha_n^j(x)\left\{\ell_{\mathbb{P}}(\theta)(y)-\ell_{\mathbb{P}}(\theta)(\psi_n^j(u|x))\right\} + \ell_{\mathbb{P}}(\theta)(\psi_n^j(u|x))\big]\,\Pi(du). \label{eq:LnDef}
\end{align}
\begin{lemma}[Generic generalization bound]\label{lem:risksUniformlyClose}
 If \cref{cond:strongPositivity,cond:mixedLipschitzWeak,cond:curvature,cond:bddLoss,cond:RMexists} hold with $\underline{\Theta}=\Theta$, then
\begin{align*}
    \mathcal{G}_{\mathbb{P}}^{1/2}(\theta)&\le \left[0\vee P L_n(\theta)\right]^{1/2} + \Cmixedweak^{1/2}\max_{j\in [2]}\|\alpha_n^j-\alpha_P\|_{L^2(P_X)}d_\Psi(\psi_n^j,\Psi_P) \ \textnormal{ for all $\theta\in\Theta$.}
\end{align*}
\end{lemma}
The above shows that the generalization error $\mathcal{G}_{\mathbb{P}}(\theta)$ arising from the loss $\ell_{\mathbb{P}}$ will be small if the generalization error from the loss $L_n$ is small and the nuisances are estimated well. 

A similar black-box generalization bound can be derived from Thm.~1 of \cite{foster2023orthogonal}. That result has the benefit of applying to more general estimation problems than the counterfactual generation problem tackled by \nameref*{alg:double}, but the disadvantage of taking a more complex form and not yielding a doubly robust remainder term.
\begin{proof}[Proof of \Cref{lem:risksUniformlyClose}]
For any $\theta\in\Theta$, adding and subtracting terms and applying \Cref{lem:lossDR} yields
\begin{align*}
    \mathcal{G}_{\mathbb{P}}(\theta)&= \mathcal{G}_{\mathbb{P}}(\theta) - P L_n(\theta) + P L_n(\theta) \\
    &\le \Cmixedweak^{1/2} \mathcal{G}_{\mathbb{P}}^{1/2}(\theta) \max_{j\in [2]}\|\alpha_n^j-\alpha_P\|_{L^2(P_X)}d_\Psi(\psi_n^j,\Psi_P) + 0\vee P L_n(\theta).
\end{align*}
When combined with the below \Cref{lem:basic}, this yields the claimed bound.
\end{proof}

\begin{lemma}[Basic inequality]\label{lem:basic}
    If $b,c,d\ge 0$ satisfy $b\le cb^{1/2} + d$, then $b^{1/2}\le d^{1/2} + c$.
\end{lemma}
\begin{proof}
    The condition on $b,c,d$ can be rewritten as $(b^{1/2}-c/2)^2\le d + c^2/4$, which is only possible if $b^{1/2}\le c/2 + (d+c^2/4)^{1/2}$. By the triangle inequality, $b^{1/2}\le d^{1/2} + c$.
\end{proof}

\section{Sufficient conditions for mixed Lipschitz conditions (\cref*{cond:mixedLipschitzWeak} and \cref*{cond:mixedLipschitzStrong})}\label{app:mixedLipSuff}

For any $\psi\in\Psi$, let $P_{\psi|x}(\,\cdot\,):=\psi(\,\cdot\mymid x)_{\sharp}\Pi$. For distributions $\nu_1,\nu_2$ defined on the same probability space, define the chi-squared divergence as
\begin{align}
    \chisq{\nu_1}{\nu_2}&= \begin{cases}
        \int \left(\frac{d\nu_1}{d\nu_2}-1\right)^2 d\nu_2,&\mbox{ if $\nu_1\ll\nu_2$} \\
        \infty,&\mbox{ otherwise.}
    \end{cases} \label{eq:chisq}
\end{align}
\begin{lemma}\label{lem:mixedLipSuff}
    If $d_\Psi^2(\psi,\Psi_P):=\esssup_{X\sim P_X} \chisq{P_{\psi| X}}{P_{Y|A=a^\star,X}}$ for all $\psi\in\Psi$, then \cref{cond:mixedLipschitzStrong} holds with $\Cmixedstrong=1$. If \cref{cond:curvature} also holds, then \cref{cond:mixedLipschitzWeak} holds with $\Cmixedweak=\Ccurv$.
\end{lemma}
By the definition of $\Psi_P$, $P_{\psi_P|X}=P_{Y|A=a^\star,X}$ $P_X$-a.s. for all $\psi_P\in\Psi_P$. Hence, an equivalent expression for $d_\Psi$ from the above lemma is $d_\Psi^2(\psi,\Psi_P):=\esssup_{X\sim P_X} \chisq{P_{\psi| X}}{P_{\psi_P|X}}$ for a generic $\psi_P\in\Psi_P$.
\begin{proof}[Proof of \Cref{lem:mixedLipSuff}]
    Let $\psi_P\in\Psi_P$, so that $P_{\psi_P|X}=P_{Y|A=a^\star,X}$.  For the first claim, we apply the definition of $P_{\psi|x},P_{\psi_P|x}$ and Cauchy-Schwarz to show that, for $P_X$-almost all $X$,
    \begin{align*}
        \left\{\medint\int \left[ \ell_{\mathbb{P}}(\theta)(\psi(u|x)) - \ell_{\mathbb{P}}(\theta)(\psi_P(u|x))\right]\,\Pi(du)\right\}^2&= \left\{\medint\int \ell_{\mathbb{P}}(\theta)(y)\,(P_{\psi|x}-P_{\psi_P|x})(dy)\right\}^2 \\
        &\le \chisq{P_{\psi| x}}{P_{\psi_P| x}}\medint\int \ell_{\mathbb{P}}^2(\theta)(y)P_{\psi_P|x}(dy).
    \end{align*}
    Taking a $P_X$-essential supremum of the first term on the right and then integrating both sides against $P_X$ yields that \cref{cond:mixedLipschitzStrong} holds with $d_\Psi^2(\psi,\Psi_P):=\esssup_{X\sim P_X} \chisq{P_{\psi| X}}{P_{\psi_P| X}}$ and $\Cmixedstrong=1$.
    
    For the second claim, we use the general result that \cref{cond:mixedLipschitzStrong} and \cref{cond:curvature} together imply \cref{cond:mixedLipschitzWeak} with $\Cmixedweak=\Ccurv\Cmixedstrong$. Since we have already established that \cref{cond:mixedLipschitzStrong} holds with $\Cmixedstrong=1$ when $d_\Psi^2(\psi,\Psi_P):=\esssup_{X\sim P_X} \chisq{P_{\psi| X}}{P_{\psi_P| X}}$, \cref{cond:curvature} implies \cref{cond:mixedLipschitzWeak} holds with $\Cmixedweak=\Ccurv$ when $d_\Psi$ takes this form.
\end{proof}

\section{Total variation guarantee for DoubleGen diffusion}\label{app:diffusionMinimax}






\subsection{Score network class}

We establish a rate of convergence for \nameref*{alg:double} diffusion modeling (\Cref{ex:diffusion}) when an empirical risk minimizer is evaluated over a particular neural network class $\underline{\Theta}$. This class is selected to be large enough to make the approximation error small. In the following lemma, $\|\cdot\|_0$ denotes the sparsity `norm' that counts the number of nonzero entries of a Euclidean vector and $w:=(d+1,W,W,\ldots,W,1)\in\mathbb{R}^{D+1}$. Here and throughout this appendix we let $\mathcal{P}_1^\star$ denote the collection of all $\mathbb{P}$ satisfying \cref{cond:diffusionBddDensity,cond:diffusionBesov,cond:diffusionBoundary} for fixed constants $\Cbesov$, $\CdiffDens$, and $\CdiffusionBoundary$. Unless otherwise specified, we use `$\lesssim$' to denote inequalities up to constants that do not depend on $\mathbb{P}\in\mathcal{P}_1^\star$ or $n$, including through $n$-dependent quantities like $\underline{t}$ or $\underline{\Theta}$.
\begin{lemma}[Score network for diffusion model, \citealp{oko2023diffusion}]\label{lem:neuralNet}
    Suppose \cref{cond:diffusionBddDensity,cond:diffusionBesov,cond:diffusionBoundary,cond:diffusionSupport,cond:diffusionTruncation}. There exists a depth $D\lesssim \log^4 n$, width $W\lesssim n\log^6 n$, sparsity level $S\lesssim n^{d/(2s+d)}\log^8 n$, weight bound $B_{\mathrm{wgt}}$ satisfying $\log B_{\mathrm{wgt}}\lesssim \log^4 n$, and output bound $B_{\mathrm{out}}\lesssim \log^{1/2} n$ such that, for all $n$ large enough,
    \begin{align*}
        \underline{\Theta}&:=\{\theta\in (\mathbb{R}^d)^{\mathbb{R}^d\times [\underline{t},\overline{t}]} : \|\theta(\cdot,t)\|_\infty\le B_{\mathrm{out}}/\sigma_t\,\forall t\in [\underline{t},\overline{t}]\}\,\cap \tag{bounded $d$-dimensional output} \\
        &\hspace{1.5em}\Big\{(M_D\ReLU(\cdot)+v_D)\circ \ldots\circ (M_2\ReLU(\cdot)+v_2)\circ (M_1 (\cdot) + v_1) : \tag{sparse ReLU network} \\
            &\hspace{2.5em}M_j\in [-B_{\mathrm{wgt}},B_{\mathrm{wgt}}]^{w_{j+1}\times w_j},v_j\in [-B_{\mathrm{wgt}},B_{\mathrm{wgt}}]^{w_{j+1}},\sum_{j=1}^D (\|M_j\|_0+\|v_j\|_0)\le S\Big\}
    \end{align*}
    satisfies the approximation error bound $\inf_{\theta\in \underline{\Theta}}\mathcal{G}_{\mathbb{P}}(\theta)\lesssim n^{-2s/(2s+d)} \log^2 n$.
\end{lemma}
The above lemma above is a restatement of Thm.~3.1 from \cite{oko2023diffusion}, and so the proof is omitted. We require the condition that $n$ is sufficiently large so that our \cref{cond:diffusionBoundary} implies Assumption 2.6 from \cite{oko2023diffusion}, which is a slightly weaker version of our boundary smoothness condition that allows $\varepsilon$ to decay with $n$. Our results would also hold under their weaker condition---we use the stronger condition because it is simpler to state.

\subsection{Verifying the conditions of \Cref*{thm:excRiskOSLFormal}}\label{app:diffusionConditions}

Under the conditions of \Cref{thm:diffusionMinimax}, we show that \nameref*{alg:double} diffusion satisfies the conditions of \Cref*{thm:excRiskOSLFormal}. We suppose throughout that \cref{cond:strongPositivity} holds, since this standard causal condition \citep{hernan2024causal} is directly assumed in \Cref{thm:diffusionMinimax}. We further assume that $n$ is large enough so the conclusion of \Cref{lem:neuralNet} holds.

\noindent \textbf{Condition \cref{cond:RMexists}:} Applying Thm.~4.1.15 in \cite[][]{durrett2019probability} coordinatewise shows $\theta_{\mathbb{P}}\in \argmin_{\theta\in\Theta} E_{\mathbb{P}}[\ell(\theta,Y)]$.

\noindent \textbf{Condition \cref{cond:bddLoss}:} This condition follows by the following lemma.

\begin{lemma}[Bounds on denoising score matching loss, \citealp{oko2023diffusion}]\label{lem:diffusionLossBdd}
    Suppose \cref{cond:diffusionSupport,cond:diffusionBddDensity,cond:diffusionTruncation} and $n\ge 2$. For $\underline{\Theta}$ chosen as in \Cref{lem:neuralNet}, $\sup_{\theta\in\underline{\Theta},y\in\mathcal{Y}}|\ell(\theta,y)|\lesssim \log^2 n$ and $\sup_{y\in\mathcal{Y}}|\ell(\theta_{\mathbb{P}},y)|\lesssim \log^2 n$. As a consequence, \cref{cond:bddLoss} holds with $\Closs\lesssim \log^2 n$.
\end{lemma} 
The proof of the above relies on the following helper lemma.
\begin{lemma}\label{lem:sigmaInt}
    If \cref{cond:diffusionTruncation}, then there exists $C>0$ such that $\int_{\underline{t}}^{\overline{t}} \sigma_t^{-2} dt\le C \log n$ for all $n\ge 2$ and $t\ge \underline{t}$.
\end{lemma}
\begin{proof}[Proof of \Cref{lem:sigmaInt}]
    For all $x\ge 0$, $1-\exp(-x)\ge x/(1+x)$, and so $1/[1+\exp(-x)]\le 1+1/x$. Plugging in $2\int_0^t \beta_v dv$ for $x$ shows $\sigma_t^{-2}\le 1+1/(2\int_0^t \beta_v dv)\le 1+1/(2\underline{\beta}t)$. Hence,
    \begin{align*}
        \medint\int_{\underline{t}}^{\overline{t}} \sigma_t^{-2} dt&\le \medint\int_{\underline{t}}^{\overline{t}} [1+1/(2\underline{\beta}t)] dt= \overline{t}-\underline{t} + \frac{1}{2\underline{\beta}}\medint\int_{\underline{t}}^{\overline{t}} \frac{1}{t} dt = \overline{t}-\underline{t} + \frac{1}{2\underline{\beta}}\left(\log\overline{t}-\log\underline{t}\right).
    \end{align*}
    By \cref{cond:diffusionTruncation} and the fact that $n\ge 2$, the right-hand side upper bounds by a constant times $\log n$.
\end{proof}

\begin{proof}[Proof of \Cref{lem:diffusionLossBdd}]
    Lem.~C.1 from \cite{oko2023diffusion} shows $\sup_{\theta\in\underline{\Theta},y\in\mathcal{Y}}|\ell(\theta,y)|\lesssim \log^2 n$ provided $\underline{\Theta}$ satisfies the bound $\sup_{\theta\in\underline{\Theta},t\in[\underline{t},\overline{t}],y_t\in\mathbb{R}^d}|\sigma_t\theta(y_t,t)|\lesssim \log^{1/2} n$, which it does by the choice of $B_{\mathrm{out}}$ \Cref{lem:neuralNet}.

    It remains to show that $\sup_{y_0\in\mathcal{Y}}|\ell(\theta_{\mathbb{P}},y_0)|\lesssim \log^2 n$. This fact was used in the proof of Thm.~C.4 in \cite{oko2023diffusion}, though the details for deriving it were omitted. We give them here. We start by following similar arguments to those used in the proof of Lem.~C.1 in \cite{oko2023diffusion}, yielding the following for each $y_0\in [-1,1]^d$: 
    \begin{align}
        \frac{1}{2}\ell(\theta_{\mathbb{P}},y_0)&= \frac{1}{2}\medint\int_{\underline{t}}^{\overline{t}} \medint\int_{\mathbb{R}^d} \left\|\frac{\mu_t y_0-y_t}{\sigma_t^2}-\theta_{\mathbb{P}}(y_t,t)\right\|^2 p(y_t\mymid y_0)dy_t\, dt \nonumber \\
        &\le \medint\int_{\underline{t}}^{\overline{t}} \medint\int_{\mathbb{R}^d} \left\|\frac{\mu_t y_0-y_t}{\sigma_t^2}\right\|^2 p(y_t\mymid y_0)dy_t\, dt + \medint\int_{\underline{t}}^{\overline{t}} \medint\int_{\mathbb{R}^d} \left\|\theta_{\mathbb{P}}(y_t,t)\right\|^2 p(y_t\mymid y_0)dy_t\, dt \nonumber \\
        &\lesssim \log n + \medint\int_{\underline{t}}^{\overline{t}} \medint\int_{\mathbb{R}^d} \left\|\theta_{\mathbb{P}}(y_t,t)\right\|^2 p(y_t\mymid y_0)dy_t\, dt, \label{eq:diffusionLossBdScore}
    \end{align}
    where the suppressed multiplicative constant on the right-hand side of `$\lesssim$' does not depend on $y_0$ and the bound on the first term follows from $Y_t\mymid Y_0\sim N(\mu_t  Y_0,\sigma_t^2 I_d)$ and \Cref{lem:sigmaInt}. 
    By the bound on the score $\theta_{\mathbb{P}}$  given in Eq.~17 of Lem.~A.3 in \cite{oko2023diffusion}, the latter term above satisfies
    \begin{align}
        \medint\int_{\underline{t}}^{\overline{t}} \medint\int_{\mathbb{R}^d} \left\|\theta_{\mathbb{P}}(y_t,t)\right\|^2 p(y_t\mymid y_0)dy_t\, dt&\lesssim \medint\int_{\underline{t}}^{\overline{t}} \sigma_t^{-2} \medint\int_{\mathbb{R}^d} \left(1\vee \sigma_t^{-2}(\|y_t\|_\infty - \mu_t)_{+}^2\right) p(y_t\mymid y_0)dy_t\, dt \nonumber \\
        &\le \medint\int_{\underline{t}}^{\overline{t}} \sigma_t^{-2} \medint\int_{\mathbb{R}^d} \left(1 + \sigma_t^{-2}(\|y_t\|_\infty - \mu_t)_{+}^2\right) p(y_t\mymid y_0)dy_t\, dt \nonumber \\
        &\lesssim \log^2 n + \medint\int_{\underline{t}}^{\overline{t}} \sigma_t^{-4} \medint\int_{\mathbb{R}^d}(\|y_t\|_\infty - \mu_t)_{+}^2p(y_t\mymid y_0)dy_t\, dt, \label{eq:scoreBd}
    \end{align}
    where the final inequality used \Cref{lem:sigmaInt}. 
    For the latter term, we use that $(\|b\|_\infty-c)_{+}= \|(b-c)_{+}\|_\infty\vee \|(b+c)_-\|_\infty$ for all $b\in\mathbb{R}^d$ and $c\in\mathbb{R}$, with the positive and negative part functions applied elementwise. 
    Combining this with the fact that $y_0\in [-1,1]^d$ yields
    \begin{align*}
        \medint\int_{\mathbb{R}^d}(\|y_t\|_\infty - \mu_t)_{+}^2p(y_t\mymid y_0)&\le \sup_{\tilde{y}_0\in [-1,1]^d} \medint\int_{\mathbb{R}^d}\|(y_t-\mu_t)_{+}\|_\infty^2 p(y_t\mymid Y_0=\tilde{y}_0)\, dy_t \\
        &\quad+ \sup_{\tilde{y}_0\in [-1,1]^d} \medint\int_{\mathbb{R}^d}\|(y_t+\mu_t)_{-}\|_\infty^2 \, p(y_t\mymid Y_0=\tilde{y}_0)\, dy_t.
    \end{align*}
    Since $Y_t\mymid Y_0\sim N(\mu_t  Y_0,\sigma_t^2 I_d)$, it is straightforward to verify that the first supremum is achieved at $\tilde{y}_0=(1,1,\ldots,1)$ and the second at $\tilde{y}_0=(-1,-1,\ldots,-1)$. Hence, for $V\sim N(0_d,I_d)$,
    \begin{align*}
        \medint\int_{\mathbb{R}^d}(\|y_t\|_\infty - \mu_t)_{+}^2p(y_t\mymid y_0)&\le E\left[\|(\sigma_t V)_{+}\|_\infty^2 \right] + E\left[\|(\sigma_t V)_{-}\|_\infty^2 \right] = \sigma_t^2 E[\|V\|_\infty^2].
    \end{align*}
    By properties of Gaussian random variables, $E[\|V\|_\infty^2]\le C\log d$, with $C$ a universal constant. Hence, the second term on the right-hand side of \eqref{eq:scoreBd} is upper bounded by $C\log(d)\int_{\underline{t}}^{\overline{t}} \sigma_t^{-2} dt\lesssim \log n$ (\Cref{lem:sigmaInt}). Plugging this into \eqref{eq:scoreBd} and then returning to \eqref{eq:diffusionLossBdScore} shows that $\sup_{y_0\in\mathcal{Y}}|\ell(\theta_{\mathbb{P}},y_0)|\lesssim \log^2 n$.

    Since $\ell_{\mathbb{P}}(\theta)(\cdot):=\ell(\theta,\cdot) - \ell(\theta_{\mathbb{P}},\cdot)$ and we have shown that $\sup_{\theta\in\underline{\Theta},y\in\mathcal{Y}}|\ell(\theta,y)|\vee \sup_{y\in\mathcal{Y}}|\ell(\theta_{\mathbb{P}},y)|\lesssim \log^2 n$, the triangle inequality yields \cref{cond:bddLoss} with $\Closs\lesssim \log^2 n$.
\end{proof}

\noindent \textbf{Condition \cref{cond:curvature}:} For any $\theta\in \underline{\Theta}$ and $y\in \mathcal{Y}$, Cauchy-Schwarz yields that
\begin{align*}
     \ell_{\mathbb{P}}^2(\theta)(y)&= \left(\medint\int_{\underline{t}}^{\overline{t}} E\left[\left\{\frac{2(\mu_t y-Y_t)}{\sigma_t^2}-\theta(Y_t,t)-\theta_{\mathbb{P}}(Y_t,t)\right\}^\top \left\{\theta_{\mathbb{P}}(Y_t,t)-\theta(Y_t,t)\right\}\,\middle|\, Y_0=y\right] dt\right)^2 \\ 
     &\le \left(\medint\int_{\underline{t}}^{\overline{t}} E\left[\left\|\frac{2(\mu_t y-Y_t)}{\sigma_t^2}-\theta(Y_t,t)-\theta_{\mathbb{P}}(Y_t,t)\right\|^2\,\middle|\, Y_0=y\right] dt\right) \\
     &\hspace{4em}\times \medint\int_{\underline{t}}^{\overline{t}} E\left[\left\|\theta_{\mathbb{P}}(Y_t,t)-\theta(Y_t,t)\right\|^2\,\middle|\, Y_0=y\right] dt.
\end{align*}
The leading integral above is upper bounded by $2[\ell(\theta,y)+\ell(\theta_{\mathbb{P}},y)]$. Using \Cref{lem:diffusionLossBdd} to bound this quantity uniformly over $(\theta,y)\in\underline{\Theta}\times \mathcal{Y}$, plugging this bound into the above, and finally integrating both sides against $\mathbb{P}$ yields that $\|\ell_{\mathbb{P}}(\theta)\|_{L^2(\mathbb{P})}^2\lesssim \log^2(n)\mathcal{G}_{\mathbb{P}}(\theta)$, and so \cref{cond:curvature} holds with $\Ccurv\lesssim \log^2 n$.

\noindent \textbf{Condition \cref{cond:entropyStronger}:} We will provide a finite bound on $J(\delta,\ell_{\mathbb{P}}(\underline{\Theta}))$ when $\underline{\Theta}$ is the neural network class from \Cref{lem:neuralNet}. The entropy of $\ell(\underline{\Theta})$ depends on $n$ through the architecture of the neural nets in $\underline{\Theta}$ and the truncation times $\underline{t}$ and $\overline{t}$ used in the loss.
\begin{lemma}[Entropy integral bound for score network loss class]\label{lem:neuralNetEntropy}
    If \cref{cond:diffusionBddDensity,cond:diffusionBesov,cond:diffusionBoundary,cond:diffusionSupport,cond:diffusionTruncation} and $\underline{\Theta}$ is chosen as in \Cref{lem:neuralNet}, then there exists $\Centstrong<\infty$ such that, for all $n\ge 2$, $\mathbb{P}\in\mathcal{P}$, and $\delta\le c$,
    \begin{align*}
        J(\delta,\ell_{\mathbb{P}}(\underline{\Theta}))\le \Centstrong\log^{8}(n) n^{d/(4s+2d)}\delta \sqrt{\log\delta^{-1}}.
    \end{align*}
\end{lemma}
Since the above entropy integral bound is finite, \cref{cond:entropyStronger} holds. Our proof of \Cref{thm:diffusionMinimax} will make more precise use of this bound to obtain a valid choice of $\delta_n$.

We provide the proof of \Cref{lem:neuralNetEntropy} at the end of this appendix. That proof will rely on a bound on the metric entropy of $\ell(\underline{\Theta})$, which we state and prove now. In what follows, $\|\cdot\|_\infty$ denotes the usual norm on $L^\infty([-1,1]^d)$.
\begin{lemma}[Metric entropy bound for score network loss class]\label{lem:neuralNetCovering}
    If \cref{cond:diffusionBddDensity,cond:diffusionBesov,cond:diffusionBoundary,cond:diffusionSupport,cond:diffusionTruncation} and $n\ge 2$, then there exist constants $c\in(0,1)$ and $C\in(0,\infty)$ that do not depend on $n$ such that, for all $\epsilon\le c$,
    \begin{align*}
        \log N(\epsilon,\ell(\underline{\Theta}),\|\cdot\|_\infty)\le C n^{d/(2s+d)}[\log^{16}(n) + \log^{12}(n)\log\left(1/\epsilon \right)].
    \end{align*}
\end{lemma}
\begin{proof}[Proof of \Cref{lem:neuralNetCovering}]
    We imitate the proof of Lem.~C.2 from \cite{oko2023diffusion}, but modify it slightly so that it gives a covering number bound even when $\epsilon$ is arbitrarily small.

    We will show there exists a constant $C'<\infty$ such that, for all $\varepsilon\le 1/2$ and a constant $c_3$ that we will specify  later,
    \begin{align}
        \log N(c_3\varepsilon\log^{7/2}n,\ell(\underline{\Theta}),\|\cdot\|_\infty)\le C'n^{d/(2s+d)}[\log^{16}(n) + \log^{12}(n)\log\left(1/\varepsilon \right)]. \label{eq:nnEntropyCrux}
    \end{align}
    Applying the change of variables $\epsilon:=c_3\varepsilon\log^{7/2}n$ then gives the desired result with $c:=c_3\log^{7/2}(n)/2$ and an appropriately specified $C$ that depends only on $c,C'$.
    
    Fix $\varepsilon\le 1/2$. Below we let $\tilde{Y}_t:=(\mu_t Y_0-Y_t)/\sigma_t^2$. For any $\theta_1,\theta_2\in\underline{\Theta}$ and $y\in\mathcal{Y}$, we have that
    \begin{align}
        &\left|\ell(\theta_1,y)-\ell(\theta_2,y)\right| \label{eq:diffusionCoverLossDiff} \\
        &\le \left|\medint\int_{\underline{t}}^{\overline{t}} E\left[\left(\left\|\tilde{Y}_t-\theta_1(Y_t,t)\right\|_2^2-\left\|\tilde{Y}_t-\theta_2(Y_t,t)\right\|_2^2\right) 1\{\|\sigma_t\tilde{Y}_t\|_\infty\ge \log^{1/2} \varepsilon^{-1}\}\,\middle|\, Y_0=y\right] dt\right| \nonumber \\
        &\quad+ \left|\medint\int_{\underline{t}}^{\overline{t}} E\left[\left(\left\|\tilde{Y}_t-\theta_1(Y_t,t)\right\|_2^2-\left\|\tilde{Y}_t-\theta_2(Y_t,t)\right\|_2^2\right) 1\{\|\sigma_t\tilde{Y}_t\|_\infty< \log^{1/2} \varepsilon^{-1}\}\,\middle|\, Y_0=y\right] dt\right|. \nonumber
    \end{align}
    The bounds on $\theta_1,\theta_2$ from \Cref{lem:neuralNet}, Cauchy-Schwarz, and Lem.~F.12 from \cite{oko2023diffusion} show that the first term is no more than a constant $c_1$ that may depend on $d$ times $\varepsilon\log^2 n$, where we used that $\int_{\underline{t}}^{\overline{t}} \sigma_t^{-2}dt\lesssim \log n$ by \Cref{lem:sigmaInt}. For the second term, we use that, for any $\tilde{y}_t\in[-\sigma_t^{-1}\log^{1/2}\varepsilon^{-1},\sigma_t^{-1}\log^{1/2}\varepsilon^{-1}]^d$,
    \begin{align*}
        \left|\left\|\tilde{y}_t-\theta_1(y_t,t)\right\|_2^2-\left\|\tilde{y}_t-\theta_2(y_t,t)\right\|_2^2\right|&= \left|\left\langle 2\tilde{y}_t-(\theta_1+\theta_2)(y_t,t),(\theta_1-\theta_2)(y_t,t)\right\rangle\right| \\
        &\le \frac{2d^{1/2}}{\sigma_t}\left(\log^{1/2}\varepsilon^{-1}+B_{\mathrm{out}}\right)\|(\theta_1-\theta_2)(y_t,t)\|_2.
    \end{align*}
    Plugging this into the second term in \eqref{eq:diffusionCoverLossDiff} shows that term is no more than a constant $c_2$ times
    \begin{align*}
        2d^{1/2}\log^2(n)(\sqrt{\log\varepsilon^{-1}}+B_{\mathrm{out}})\medint\int_{\underline{t}}^{\overline{t}} E\Big[\|(\theta_1-\theta_2)(Y_t,t)\|_2 I\{\|\sigma_t \tilde{Y}_t\|_\infty\le \sqrt{\log\varepsilon^{-1}}\}\Big|Y_0=y\Big] dt.
    \end{align*}
    The expectation above is upper bounded by $\sup_{y_t : \|\mu_t y-y_t\|_\infty\le \sigma_t\log^{1/2} \varepsilon^{-1}}\|(\theta_1-\theta_2)(y_t,t)\|_2$, which by the bounds on $y$ (\cref{cond:diffusionSupport}) is further bounded by $\sup_{y_t : \|y_t\|_\infty\le \mu_t + \sigma_t\log^{1/2} \varepsilon^{-1}}\|(\theta_1-\theta_2)(y_t,t)\|_2=\|\|(\theta_1-\theta_2)(\cdot,t)\|_2\|_{L^\infty([-\mu_t - \sigma_t\log^{1/2} \varepsilon^{-1},\mu_t + \sigma_t\log^{1/2} \varepsilon^{-1}]^d)}$. Hence, the second term in \eqref{eq:diffusionCoverLossDiff} is no more than $2c_2 d^{1/2}\log^2(n)\left(\log^{1/2} \varepsilon^{-1}+B_{\mathrm{out}}\right)$ times
    \begin{align*}
        \medint\int_{\underline{t}}^{\overline{t}} \|\|(\theta_1-\theta_2)(\cdot,t)\|_2\|_{L^\infty([-\mu_t - \sigma_t\log^{1/2} \varepsilon^{-1},\mu_t + \sigma_t\log^{1/2} \varepsilon^{-1}]^d)} dt.
    \end{align*}
    Combining our bounds of the two terms in \eqref{eq:diffusionCoverLossDiff} shows that, if
    \begin{align}
        \|\|(\theta_1-\theta_2)(\cdot,t)\|_2\|_{L^\infty([-\mu_t - \sigma_t\log^{1/2} \varepsilon^{-1},\mu_t + \sigma_t\log^{1/2} \varepsilon^{-1}]^d)}\le \varepsilon/\log^{1/2} \varepsilon^{-1}\textnormal{ for all $t\in[\underline{t},\overline{t}]$}, \label{eq:theta1theta2supNorm}
    \end{align}
    then, for all $y$,
    \begin{align*}
        \left|\ell(\theta_1,y)-\ell(\theta_2,y)\right|&\le \left[c_1 + 2c_2 d^{1/2}(\overline{t}-\underline{t})\left(1+\log^{-1/2} \varepsilon^{-1}B_{\mathrm{out}}\right)\right]\varepsilon\log^2 n.
    \end{align*}
    By the bounds on $\overline{t}$ and $B_{\mathrm{out}}$ from \cref{cond:diffusionTruncation} and \Cref{lem:neuralNet} and the fact that $\varepsilon\le 1/2$, there exists a constant $c_3$ that does not depend on $n$ such that, for all $\theta_1,\theta_2$ satisfying \eqref{eq:theta1theta2supNorm},
    \begin{align}
        \left|\ell(\theta_1,y)-\ell(\theta_2,y)\right|&\le c_3 \varepsilon\log^{7/2} n\ \textnormal{ for all }y\in [-1,1]^d. \label{eq:lossLipBdNN}
    \end{align}
    The bound on $\overline{t}$ and the fact that $\mu_t,\sigma_t\in [0,1]$ further imply there exists a constant $c_4$ such that, for $b(n,\varepsilon):=c_4(\log^{1/2}\varepsilon^{-1}\vee \log n)$,
    \begin{align*}
        [\underline{t},\overline{t}]\cup [-\mu_t - \sigma_t\log^{1/2} \varepsilon^{-1},\mu_t + \sigma_t\log^{1/2} \varepsilon^{-1}] \subseteq [-b(n,\varepsilon),b(n,\varepsilon)],
    \end{align*}
    and so a sufficient condition for \eqref{eq:theta1theta2supNorm} is that
    \begin{align*}
         \|\|(\theta_1-\theta_2)(\cdot)\|_2\|_{L^\infty([-b(n,\varepsilon),b(n,\varepsilon)]^{d+1})}\le \varepsilon/\log^{1/2} \varepsilon^{-1}.
    \end{align*}
    Hence, the above implies \eqref{eq:lossLipBdNN}, which yields that
    \begin{align*}
        \log N(c_3\varepsilon\log^{7/2}n,\ell(\underline{\Theta}),L^\infty([-1,1]^d)&\le \log N(\varepsilon/\log^{1/2}\varepsilon^{-1},\underline{\Theta},\|\|\cdot\|_2\|_{L^\infty([-b(n,\varepsilon),b(n,\varepsilon)]^{d+1})}).
    \end{align*}
    We apply Eq.~57 from Lem.~C.2 of \cite{oko2023diffusion} (see also \cite{suzuki2018adaptivity}) and the architecture of $\underline{\Theta}$ from \Cref{lem:neuralNet} to bound the right-hand side above by a constant times
    \begin{align*}
        n^{d/(2s+d)}\log^{12}(n)\left[\log^4 n + \log\left(\varepsilon^{-1}\log^{1/2}(\varepsilon^{-1})\log(n)n b(n,\varepsilon)\right)\right].
    \end{align*}
    This is in turn upper bounded by a constant $C'$ times $n^{d/(2s+d)}[\log^{16}(n) + \log^{12}(n)\log\left(1/\varepsilon \right)]$. Hence, \eqref{eq:nnEntropyCrux} holds, completing the proof.
\end{proof}

\begin{proof}[Proof of \Cref{lem:neuralNetEntropy}]
    Let $\|\cdot\|_\infty$ denote the supremum norm on $\mathcal{Y}$. Combining the definition of the entropy integral in \eqref{eq:unifEntDef} with the fact that $\|\cdot\|_{L^2(Q)}\le \|\cdot\|_\infty$ for any distribution $Q$ on $\mathcal{Y}$ shows that
    \begin{align*}
        J(\delta,\ell_{\mathbb{P}}(\underline{\Theta}))&\le \medint\int_0^\delta \sqrt{1+\log N(\epsilon,\ell_{\mathbb{P}}(\underline{\Theta}),\|\cdot\|_\infty)}\,d\epsilon.
    \end{align*}
    Since the uncentered loss class $\ell(\underline{\Theta}):=\{\ell(\theta,\,\cdot\,) : \theta\in\underline{\Theta}\}$ equals the centered loss class $\ell_{\mathbb{P}}(\underline{\Theta}):=\{\ell(\theta,\,\cdot\,)-\ell(\theta_{\mathbb{P}},\,\cdot\,) : \theta\in\underline{\Theta}\}$ up to the translation $\ell(\theta_{\mathbb{P}},\,\cdot\,)$, the two classes have the same sup-norm covering number. Moreover, by \Cref{lem:neuralNetCovering}, there exists $C>0$ such that $\log N(\epsilon,\ell(\underline{\Theta}),\|\cdot\|_\infty)\le C n^{d/(2s+d)}[\log^{16} n + \log^{12}(n)\log \epsilon^{-1}]$ for all $\epsilon\in (0,c]$. Applying this fact and then the triangle inequality shows that, whenever $\delta\le c$,
    \begin{align}
        J(\delta,\ell_{\mathbb{P}}(\underline{\Theta}))&\le \medint\int_0^\delta \sqrt{1+C n^{d/(2s+d)}[\log^{16} n + \log^{12}(n)\log \epsilon^{-1}]}\,d\epsilon \nonumber \\
        &\le \frac{\delta}{2}(1+C^{1/2} n^{d/(4s+2d)}\log^{8}n) + C^{1/2}n^{d/(4s+2d)}\log^6(n)\medint\int_0^\delta \sqrt{\log \epsilon^{-1}}\,d\epsilon. \label{eq:NNentIntBd}
    \end{align}
    For the remaining integral, we apply the change of variables $\varepsilon=\epsilon/\delta$ and the triangle inequality to find that, for all $\delta\le c<1$,
    \begin{align*}
        \medint\int_0^\delta \sqrt{\log \epsilon^{-1}}\,d\epsilon&= \delta\medint\int_0^1 \sqrt{\log \delta^{-1} + \log \varepsilon^{-1}}\,d\varepsilon\le \delta\log^{1/2} \delta^{-1} + \delta\medint\int_0^1 \log^{1/2} \varepsilon^{-1}\,d\varepsilon.
    \end{align*}
    The remaining integral equals $\pi^{1/2}/2$. Hence, there exists a constant $k_1$ such that, for all $\delta\le c$, $\int_0^\delta \sqrt{\log \epsilon^{-1}}\,d\epsilon\le k_1 \delta \log^{1/2}\delta^{-1}$. Combining this with \eqref{eq:NNentIntBd} then shows that there exists a constant $\Centstrong$ such that, for all $\delta\le c$, $J(\delta,\ell_{\mathbb{P}}(\underline{\Theta}))\le \log^{8}(n) n^{d/(4s+2d)}\delta \sqrt{\log\delta^{-1}}$, as desired.
\end{proof}

\noindent \textbf{Condition \cref{cond:mixedLipschitzWeak}:} Since \cref{cond:curvature} holds, \Cref{lem:mixedLipSuff} shows this condition holds with $\Cmixedweak=\Ccurv$ and $d_\Psi^2(\psi,\Psi_P):=\esssup_{X\sim P_X} \chisq{P_{\psi| X}}{P_{\psi_P| X}}$.

\subsection{Proof of total variation bound (\Cref*{thm:diffusionMinimax})}

\begin{proof}[Proof of \Cref{thm:diffusionMinimax}]
    Throughout this proof, we write `$\lesssim$' to denote inequality up to a multiplicative constant that does not depend on the value of $n\ge 2$ or $\mathbb{P}\in \mathcal{P}_1^\star$. We suppose throughout that $n$ is large enough so the conclusion of \Cref{lem:neuralNet} holds. In \Cref{app:diffusionConditions}, we show that \cref{cond:strongPositivity,cond:diffusionBddDensity,cond:diffusionBesov,cond:diffusionSupport,cond:diffusionTruncation,cond:diffusionBoundary} imply that the conditions of \Cref{thm:excRiskOSLFormal} hold with $d_\Psi$ as in \Cref{lem:mixedLipSuff} and, for all $\delta$ small enough,
    \begin{align*}
        J(\delta,\ell_{\mathbb{P}}(\underline{\Theta}))\lesssim \log^{8}(n) n^{d/(4s+2d)}\delta \sqrt{\log\delta^{-1}}.
    \end{align*}
    Plugging this into \eqref{eq:doubleGenDeltaN} shows that $\delta_n\lesssim \log^{17/2}(n)n^{-s/(2s+d)}$. Hence, \Cref{thm:excRiskOSLFormal} shows that, w.p. at least $1-e^{-r}$,
    \begin{align*}
        \mathcal{G}_{\mathbb{P}}(\theta_n)&\lesssim \inf_{\theta\in\underline{\Theta}}\mathcal{G}_{\mathbb{P}}(\theta) + \log^{17}(n)\,n^{-2s/(2s+d)} + r/n + \max_{j\in [2]}\|\alpha_n^j-\alpha_P\|_{L^2(P_X)}^2d_{\Psi}^2(\psi_n^j,\Psi_P).
    \end{align*}
    Here, we used that the constant $K_0$ from \Cref{thm:excRiskOSLFormal} is a.s. bounded by a constant not depending on $n$ since, by assumption, \cref{cond:strongPositivity} holds and $\max_j d_\Psi(\psi_n^j,\Psi_P)$ is a.s. uniformly bounded. 
    By \Cref{lem:neuralNet}, the leading approximation term above is upper bounded by $n^{-2s/(2s+d)} \log^2 n$, uniformly over $\mathbb{P}\in\mathcal{P}_1^\star$, and so, w.p. at least $1-e^{-r}$,
    \begin{align*}
        \mathcal{G}_{\mathbb{P}}(\theta_n)&\lesssim \log^{17}(n)\,n^{-2s/(2s+d)} + r/n + \max_{j\in [2]}\|\alpha_n^j-\alpha_P\|_{L^2(P_X)}^2d_{\Psi}^2(\psi_n^j,\Psi_P).
    \end{align*}
    By \Cref{prop:divGuarantee}---which is applicable since the above implies \cref{cond:genError} and \cref{cond:divergence} holds by the arguments in \Cref{app:divergence}---the above implies that the following holds w.p. at least $1-e^{-r}$:
    \begin{align*}
        \TV(\mathbb{P},\tau(\theta_n)_{\sharp}\Pi)&\lesssim \log^{17/2}(n)\,n^{-s/(2s+d)} + \sqrt{r/n} + \max_{j\in [2]}\|\alpha_n^j-\alpha_P\|_{L^2(P_X)}d_{\Psi}(\psi_n^j,\Psi_P) + \epsilon,
    \end{align*}
    with $\epsilon=O(n^{-s/(2s+d)})$ as defined in the study of \Cref{ex:diffusion} in \Cref{app:divergence}. Since $\epsilon$ decays faster than the first term on the right, it can be absorbed into the leading constant of the first term above.
\end{proof}

\section{Wasserstein guarantee for \nameref*{alg:double} flow matching} \label{app:flow}


\subsection{Statement of guarantee}

We now present a Wasserstein guarantee for \nameref*{alg:double} flow matching (\Cref{ex:flow}) for an empirical risk minimizer $\theta_n$ over some class $\underline{\Theta}$. This guarantee will rely on the following conditions.
\begin{condenum}[resume*]
    \item\label{cond:flowBddSupport} \textit{Bounded support:} $\mathrm{support}(\mathbb{P})\subseteq [-1,1]^d$.
    \item\label{cond:flowBddVec} \textit{Uniformly bounded candidate vector fields:} $\sup_{\theta\in\underline{\Theta}}\|\theta\|_{L^\infty(\mathbb{R}^d\times [0,1];\|\cdot\|_2)}<\infty$, where $\|\theta\|_{L^\infty(\mathbb{R}^d\times [0,1];\|\cdot\|_2)}:=\sup_{(y,t)\in\mathbb{R}^d\times [0,1]} \|\theta(y,t)\|_2$.

    \item\label{cond:flowEntInt} \textit{Set of candidate vector fields has finite sup-norm entropy integral:}
    $$\medint\int_0^\delta \sqrt{1+\log N(\epsilon,\underline{\Theta},\|\cdot\|_{L^\infty(\mathbb{R}^d\times [0,1];\|\cdot\|_2)})}\,d\epsilon<\infty.$$
\end{condenum}  
The latter two conditions are satisfied by appropriately specified neural network classes \citep{suzuki2018adaptivity,fukumizu2024flow}.

When stating the guarantee, we directly suppose \cref{cond:divergence} holds. This is reasonable since, as noted in \Cref{app:divergence}, \cite{fukumizu2024flow} gives conditions under which it holds when the vector field is sufficiently smooth. We also directly assume \cref{cond:strongPositivity}, which is standard in causal inference---see Chap. 3.3 of \cite{hernan2024causal} for a discussion.
\begin{corollary}[Wasserstein guarantee for \nameref*{alg:double} flow matching]\label{cor:flow}
Suppose there exists an empirical risk minimizer $\theta_n\in\argmin_{\theta\in\underline{\Theta}} R_n(\theta)$ and \cref{cond:divergence} holds with $D=W_2$, $b=1/2$, $\epsilon=0$ and $\Cdivergence$ a finite constant. Further suppose \cref{cond:flowBddSupport,cond:flowBddVec,cond:strongPositivity}. Under these conditions, the conditions of \Cref{prop:divGuarantee,thm:excRiskOSLFormal} hold for appropriately specified constants. Hence, letting $\delta_n$ be as defined in \eqref{eq:doubleGenDeltaN}, the following holds w.p. at least $1-e^{-s}$:
\begin{align}
    W_2(\mathbb{P},\phi_{n\sharp}\Pi)&\lesssim \inf_{\theta\in\underline{\Theta}}\mathcal{G}_{\mathbb{P}}^{1/2}(\theta) + \delta_n + (s^{1/2}+1)/n^{1/2} + \max_{j\in [2]}\|\alpha_n^j-\alpha_P\|_{L^2(P_X)}d_\Psi(\psi_n^j,\Psi_P), \label{eq:flowW2Bd}
\end{align}
where $d_\Psi$ is as in \Cref{lem:mixedLipSuff} and `$\lesssim$' denotes inequality up to a multiplicative constant that does not depend on $n$ or $s$.
\end{corollary}
\begin{proof}[Proof of \Cref{cor:flow}]
    In \Cref{app:flowConditions}, we show that \cref{cond:flowBddSupport,cond:flowBddVec,cond:strongPositivity} imply that the conditions of \Cref{thm:excRiskOSLFormal} hold. This, in turn, implies that \cref{cond:genError} holds with $r$ equal to the right-hand side of \eqref{eq:excRiskBdOSL}. Hence, the conditions of \Cref{prop:divGuarantee} hold, and combining the result with the triangle inequality yields \eqref{eq:flowW2Bd}.
\end{proof}
For ease of presentation we have suppressed the explicit constants on the right-hand of \eqref{eq:flowW2Bd}, but it is straightforward to compute them using the arguments in \Cref{app:flowConditions}. The argument that establishes \cref{cond:entropyStronger} in \Cref{app:flowConditions} relates the uniform entropy integral of $\ell_{\mathbb{P}}(\underline{\Theta})$ to the entropy integral of $\underline{\Theta}$ from \cref{cond:flowEntInt}. This, in turn, makes it possible to compute the value of $\delta_n$ by directly considering the size of $\underline{\Theta}$. 

In future work, it would be interesting to establish that the bound in \eqref{eq:flowW2Bd} is minimax rate optimal. One possible approach would involve adapting the arguments in \cite{fukumizu2024flow}.

\subsection{Verifying the conditions of \Cref*{thm:excRiskOSLFormal}}\label{app:flowConditions}

Under the conditions of \Cref{cor:flow}, we show that \nameref*{alg:double} flow matching (\Cref{ex:flow}) satisfies the conditions of \Cref*{thm:excRiskOSLFormal}.

\noindent \textbf{Condition \cref{cond:RMexists}:} Applying Thm.~4.1.15 in \cite[][]{durrett2019probability} coordinatewise shows $\theta_{\mathbb{P}}\in \argmin_{\theta\in\Theta} E_{\mathbb{P}}[\ell(\theta,Y)]$.

\noindent \textbf{Condition \cref{cond:bddLoss}:} 
Since the functions in $\underline{\Theta}$ are uniformly bounded (\cref{cond:flowBddVec}), the facts that $U\sim \Pi= N(0_d,I_d)$ and $\mathrm{support}(\mathbb{P})\subseteq [-1,1]^d$ (\cref{cond:flowBddSupport}) imply $\sup_{\theta\in\underline{\Theta}}\|\ell(\theta,\cdot\,)\|_{L^\infty(\mathbb{P})}<\infty$. The function $\ell(\theta_{\mathbb{P}},\cdot\,)$ is also $\mathbb{P}$-a.s. bounded, since
\begin{align*}
    \ell(\theta_{\mathbb{P}},y)=\min_{\theta : \mathbb{R}^d\rightarrow\mathbb{R}}\int_0^1 E_\Pi[\|y-U-\theta([1-t]U+ty,t)\|_2^2] dt\le \int_0^1 E_\Pi[\|y-U\|_2^2] dt.
\end{align*}
 By the triangle inequality, the two established bounds imply that $\ell_{\mathbb{P}}(\theta)(y):=\ell(\theta,y)-\ell(\theta_{\mathbb{P}},y)$ satisfies \cref{cond:bddLoss}.

\noindent \textbf{Condition \cref{cond:curvature}:} For any $\theta\in\underline{\Theta}$ and $y\in [-1,1]^d$, Cauchy-Schwarz yields that
\begin{align*}
    \ell_{\mathbb{P}}^2(\theta)(y) 
    &=\left(\medint\int_0^1 E_\Pi[\left\{2(y-U) - (\theta+\theta_{\mathbb{P}})([1-t]U+ty,t)\right\}^\top (\theta_{\mathbb{P}}-\theta)([1-t]U+ty,t)] dt\right)^2 \\
    &\le \left(\medint\int_0^1 E_\Pi\left[\left\|2(y-U) - (\theta+\theta_{\mathbb{P}})([1-t]U+ty,t)\right\|_2^2\right] dt\right) \\
    &\quad\times \left(\medint\int_0^1 E_\Pi\left[\left\|(\theta_{\mathbb{P}}-\theta)([1-t]U+ty,t)\right\|_2^2\right] dt\right).
\end{align*}
The leading integral above is upper bounded by $2[\ell(\theta,y)+\ell(\theta_{\mathbb{P}},y)]$. By the arguments we used to establish \cref{cond:bddLoss}, this term is bounded uniformly over $(\theta,y)$. 
Integrating both sides above against $\mathbb{P}$ and using that $\mathcal{G}_{\mathbb{P}}(\theta)=\int_0^1 E_{\mathbb{P}\times\Pi}\left[\left\|(\theta_{\mathbb{P}}-\theta)([1-t]U+tY^\star,t)\right\|^2\right] dt$ shows that $\|\ell_{\mathbb{P}}(\theta)\|_{L^2(\mathbb{P})}^2\le \Ccurv\mathcal{G}_{\mathbb{P}}(\theta)$ for some constant $\Ccurv$ that does not depend on $\theta\in\underline{\Theta}$.


\noindent \textbf{Condition \cref{cond:entropyStronger}:} 
We will show that $\ell_{\mathbb{P}}$ is Lipschitz and then use properties of covering numbers to relate the uniform entropy integral of $\ell_{\mathbb{P}}(\underline{\Theta})$ to one measuring the size of $\underline{\Theta}$, which is finite by \cref{cond:flowEntInt}.

To establish the Lipschitz property, we use that, for any $\theta_1,\theta_2\in\underline{\Theta}$ and $y$,
\begin{align*}
    &\ell_{\mathbb{P}}(\theta_1)(y)-\ell_{\mathbb{P}}(\theta_2)(y) \\
    &=\medint\int_0^1 E_\Pi[\left\{2(y-U) - (\theta_1+\theta_2)([1-t]U+ty,t)\right\}^\top (\theta_1-\theta_2)([1-t]U+ty,t)] dt,
\end{align*}
and so, by H\"{o}lder's inequality and the triangle inequality,
\begin{align*}
    \|\ell_{\mathbb{P}}(\theta_1)-\ell_{\mathbb{P}}(\theta_2)\|_{\infty}&\le C\|\theta_1-\theta_2\|_{L^\infty(\mathbb{R}^d\times [0,1];\|\cdot\|_2)},
\end{align*}
where $C:=2(\sup_{\theta\in\underline{\Theta}}\|\theta\|_{L^\infty(\mathbb{R}^d\times [0,1];\|\cdot\|_2)} + \sup_{y\in [-1,1]^d}E_\Pi\left[\|y-U\|_2\right])$ and $\|f\|_{\infty}:= \sup_{y\in[-1,1]^d}|f(y)|$. The constant $C$ is finite since functions in $\underline{\Theta}$ are uniformly bounded (\cref{cond:flowBddVec}), and so the above shows $\ell_{\mathbb{P}}(\cdot)$ is $C$-Lipschitz continuous on $\underline{\Theta}$ when its domain is equipped with $\|\cdot\|_{L^\infty(\mathbb{R}^d\times [0,1];\|\cdot\|_2)}$ and codomain is equipped with the supremum norm. 

When $\ell_{\mathbb{P}}(\underline{\Theta})$ is $C$-Lipschitz, properties of covering numbers yield that $N(\epsilon,\ell_{\mathbb{P}}(\underline{\Theta}),\|\cdot\|_{\infty})\le N(\epsilon/C,\underline{\Theta},\|\cdot\|_{L^\infty(\mathbb{R}^d\times [0,1];\|\cdot\|_2)})$. Combining this with the fact that $\sup_Q N(\epsilon,\ell_{\mathbb{P}}(\underline{\Theta}),L^2(Q))\le N(\epsilon,\ell_{\mathbb{P}}(\underline{\Theta}),\|\cdot\|_{\infty})$, with the supremum taken over all finitely supported measures $Q$ on $[-1,1]^d$, and then subsequently plugging these bounds into \eqref{eq:unifEntDef}, we find that, for all $\delta>0$,
\begin{align}
    J(\delta,\ell_{\mathbb{P}}(\underline{\Theta}))\le \medint\int_0^\delta \sqrt{1+\log N(\epsilon/C,\underline{\Theta},\|\cdot\|_{L^\infty(\mathbb{R}^d\times [0,1];\|\cdot\|_2)})}\,d\epsilon. \label{eq:flowThetaEntInt} 
\end{align}
Applying a change of variables to the right-hand side and using \cref{cond:flowEntInt} shows that the right-hand side is finite, which establishes \cref{cond:entropyStronger}.

\noindent \textbf{Condition \cref{cond:mixedLipschitzWeak}:} Since \cref{cond:curvature} holds, \Cref{lem:mixedLipSuff} shows this condition holds with $\Cmixedweak=\Ccurv$ and $d_\Psi^2(\psi,\Psi_P):=\esssup_{X\sim P_X} \chisq{P_{\psi| X}}{P_{\psi_P| X}}$.

\section{Kullback-Leibler guarantee for \nameref*{alg:double} autoregressive language models}\label{app:autoreg}

\subsection{Statement of guarantee}

We now present a Kullback-Leibler guarantee for \nameref*{alg:double} autoregressive language models  (\Cref{ex:autoreg}) for an empirical risk minimizer $\theta_n$ over some class $\underline{\Theta}$, such as one based on a transformer architecture \citep{vaswani2017attention}. This guarantee will rely on the following condition:

\begin{condenum}[resume*]
    \item\label{cond:autoregThetaNontrivial} \textit{Non-trivial probability of any possible next token:} For all $j\in [d]$, there exists $\delta_j>0$ such that, for all $\theta\in\underline{\Theta}\cup\{\theta_{\mathbb{P}}\}$,
    \begin{align*}
        \mathbb{P}\big\{\delta_j\le \theta_{Y^\star(j)}(1,1,\ldots,1,Y^\star(1),Y^\star(2),\ldots,Y^\star(j-1))\big\}=1.
    \end{align*}
\end{condenum}
Because \cref{cond:autoregThetaNontrivial} only involves token sequences that are generated under $\mathbb{P}$, it allows some tokens to have zero probability conditional on earlier tokens. However, for token sequences that occur with positive probability, no hypothesis $\theta$ can assign tokens vanishing probability given past tokens.

\begin{corollary}[Kullback-Leibler guarantee for \nameref*{alg:double} autoregressive language models]\label{cor:autoreg}
    If there exists an empirical risk minimizer $\theta_n\in\argmin_{\theta\in\underline{\Theta}} R_n(\theta)$ and \cref{cond:autoregThetaNontrivial,cond:strongPositivity,cond:entropyStronger} hold, then the conditions of \Cref{prop:divGuarantee,thm:excRiskOSLFormal} hold for appropriately specified constants. Hence, letting $\delta_n$ be as defined in \eqref{eq:doubleGenDeltaN}, the following holds w.p. at least $1-e^{-s}$:
\begin{align*}
    \KL{\mathbb{P}}{\phi_{n\sharp}\Pi}&\lesssim \inf_{\theta\in\underline{\Theta}}\mathcal{G}_{\mathbb{P}}(\theta) + \delta_n^2 + (s+1)/n + \max_{j\in [2]}\|\alpha_n^j-\alpha_P\|_{L^2(P_X)}^2 d_\Psi^2(\psi_n^j,\Psi_P),
\end{align*}
where $d_\Psi$ is as in \Cref{lem:mixedLipSuff} and `$\lesssim$' denotes inequality up to a multiplicative constant that does not depend on $n$ or $s$.
\end{corollary}
\begin{proof}[Proof of \Cref{cor:autoreg}]
    In \Cref{app:autoregConditions}, we show that \cref{cond:RMexists,cond:bddLoss,cond:curvature,cond:mixedLipschitzWeak} hold. Combining this with the assumed \cref{cond:strongPositivity,cond:entropyStronger}  implies the conditions of \Cref{thm:excRiskOSLFormal} hold. This, in turn, implies that \cref{cond:genError} holds with $r$ equal to the right-hand side of \eqref{eq:excRiskBdOSL}. Moreover, by \Cref{app:divergence}, \cref{cond:divergence} holds with $D=\KLop$, $b=\Cdivergence=1$, and $\epsilon=0$. Hence, the conditions of \Cref{prop:divGuarantee} hold, and applying that proposition gives the result.
\end{proof}

\subsection{Verifying the conditions of \Cref*{thm:excRiskOSLFormal}}\label{app:autoregConditions}

Under the conditions of \Cref{cor:autoreg}, we show that \nameref*{alg:double} autoregressive language models (\Cref{ex:autoreg}) satisfy the conditions of \Cref*{thm:excRiskOSLFormal}. This corollary directly assumes \cref{cond:entropyStronger,cond:strongPositivity}, and so we only need to establish the remaining conditions.


\noindent \textbf{Condition \cref{cond:RMexists}:} This follows from Thm.~2.6.3 in \cite{cover1999elements}.

\noindent \textbf{Condition \cref{cond:bddLoss}:} This follows from the definition of the cross-entropy loss and \cref{cond:autoregThetaNontrivial}.

\noindent \textbf{Condition \cref{cond:curvature}:} Let $\bar{\theta}(y)=\prod_{j=1}^d \theta_{y(j)}(1,\ldots,1,y(1),\ldots,y(j-1))$, and define $\bar{\theta}_{\mathbb{P}}$ similarly. Observe that $\bar{\theta}_{\mathbb{P}}$ is the probability mass function of $\mathbb{P}$. Let $\mathbb{P}_{\theta}$ be the distribution of $[k]^d$ with probability mass function $\bar{\theta}$. By \cref{cond:autoregThetaNontrivial}, $\bar{\theta}(y)>0$ if $y\in\mathcal{Y}_{\mathbb{P}}:=\{y : \bar{\theta}_{\mathbb{P}}(y)>0\}$. Combining this with the fact that $\log^2 b\le (b-1)^2(b^{-1}\vee 1)^2$ for all $b>0$ yields
\begin{align*}
    \|\ell_{\mathbb{P}}(\theta)\|_{L^2(\mathbb{P})}^2&= 4\sum_{y\in\mathcal{Y}_{\mathbb{P}}} \log^2\left[\frac{\bar{\theta}^{1/2}(y)}{\bar{\theta}_{\mathbb{P}}^{1/2}(y)}\right] \bar{\theta}_{\mathbb{P}}(y)\le 4\sum_{y\in\mathcal{Y}_{\mathbb{P}}} \left[\frac{\bar{\theta}_{\mathbb{P}}^{1/2}(y)}{\bar{\theta}^{1/2}(y)}\vee 1\right]^2 \left[\bar{\theta}^{1/2}(y)-\bar{\theta}_{\mathbb{P}}^{1/2}(y)\right]^2.
\end{align*}
By \cref{cond:autoregThetaNontrivial}, the squared maximum on the right is upper bounded by $\prod_{j=1}^d \delta_j^{-1}$, and so the right-hand side is upper bounded by $\Ccurv:=4\prod_{j=1}^d \delta_j^{-1}$ times $\sum_{y\in\mathcal{Y}_{\mathbb{P}}} [\bar{\theta}^{1/2}(y)-\bar{\theta}_{\mathbb{P}}^{1/2}(y)]^2$, the squared Hellinger distance between $\mathbb{P}$ and $\mathbb{P}_\theta$. This squared distance is upper bounded by $\KL{\mathbb{P}}{\mathbb{P}_{\theta}}$ \citep[][Lemma 2.4]{tsybakov2009nonparametric}, which in turn is equal to the generalization error $\mathcal{G}_{\mathbb{P}}(\theta)$. Hence, \cref{cond:curvature} holds with constant $\Ccurv$.

While we have shown \cref{cond:curvature} holds, the corresponding constant grows exponentially with $d$. Given this, alternative bounds that do not rely on \cref{cond:curvature} may be preferred in this scenario. Section~3.2 of \cite{foster2023orthogonal} provides one approach for deriving such bounds. While those bounds tend to have a slower rate in $n$ than the one in \Cref{thm:excRiskOSLFormal}, they may be much better behaved in terms of $d$. In future work, it would be interesting to apply such results to derive generalization bounds for \nameref*{alg:double} autoregressive language models.

\noindent \textbf{Condition \cref{cond:mixedLipschitzWeak}:} Since \cref{cond:curvature} holds, \Cref{lem:mixedLipSuff} shows this condition holds with $\Cmixedweak=\Ccurv$ and $d_\Psi$ as specified in that lemma.

\section{Further details on numerical experiments}

\subsection{Generating counterfactual smiling faces}

\subsubsection{Experimental setup}\label{app:moreExperimentalSet}

\myparagraph{Code availability.} The code used to implement our experiments can be found on GitHub: \texttt{https://github.com/alexluedtke12/DoubleGen-Experiments}.


\myparagraph{Preprocessing.} Training and test images were cropped to remove artifacts and resized to 64$\times$64 pixels.

\myparagraph{Features.} 
Each image was accompanied by 39 binary features indicating attributes such as hair color, age, and presence of accessories like hats and eyeglasses. We removed 8 features judged to plausibly be caused by whether someone smiles (raised eyebrows, open mouth, etc.), leaving 31 features in $X$. The names of these 31 features were: \texttt{5\_o\_Clock\_Shadow}, \texttt{Bald}, \texttt{Bangs}, \texttt{Big\_Lips}, \texttt{Big\_Nose}, \texttt{Black\_Hair}, \texttt{Blond\_Hair}, \texttt{Brown\_Hair}, \texttt{Bushy\_Eyebrows}, \texttt{Chubby}, \texttt{Double\_Chin}, \texttt{Eyeglasses}, \texttt{Goatee}, \texttt{Gray\_Hair}, \texttt{Heavy\_Makeup}, \texttt{Male}, \texttt{Mustache}, \texttt{Narrow\_Eyes}, \texttt{No\_Beard}, \texttt{Pale\_Skin}, \texttt{Pointy\_Nose}, \texttt{Receding\_Hairline}, \texttt{Sideburns}, \texttt{Straight\_Hair}, \texttt{Wavy\_Hair}, \texttt{Wearing\_Earrings}, \texttt{Wearing\_Hat}, \texttt{Wearing\_Lipstick}, \texttt{Wearing\_Necklace}, \texttt{Wearing\_Necktie}, and \texttt{Young}. The names of the 8 features excluded from $X$ because they could plausibly be caused by whether someone smiles are: \texttt{Arched\_Eyebrows}, \texttt{Attractive}, \texttt{Bags\_Under\_Eyes}, \texttt{Blurry}, \texttt{High\_Cheekbones}, \texttt{Mouth\_Slightly\_Open}, \texttt{Oval\_Face}, \texttt{Rosy\_Cheeks}.

\myparagraph{Nuisance estimation.}
For nuisance estimation, the inverse propensity was estimated using lightgbm \citep{ke2017lightgbm} with the Riesz regression loss \citep{chernozhukov2021automatic}, with hyperparameters tuned via Optuna \citep{akiba2019optuna}. The outcome model was estimated using $k$-nearest neighbors, with $k=200$ and the Euclidean distance used to compare feature vectors. A draw from this outcome model was obtained by sampling from $\mathrm{Unif}\{1,2,\ldots,200\}$ and then returning the corresponding neighboring image.

Sensitivity to nuisance misspecification was assessed by fitting each nuisance twice. In the first, well-specified, scenario, the nuisances were fit using all available data. In the second, misspecified, scenario, the outcome model was trained only using dark-haired instances and the inverse propensity weights for lighter-haired instances was scaled down down by a factor of 4. By misspecifying the nuisances in this way, IPW and plug-in estimation should overrepresent dark-haired individuals when their relevant nuisance is misspecified.

\myparagraph{Diffusion model architecture and training.} 
The diffusion models were trained using Diffusers and Pytorch \citep{von_platen_diffusers_2022,NEURIPS2019_9015}, with code adapted from \cite{huggingface2025} with the help of a Claude coding assistant. A U-Net score network was used \citep{ronneberger2015u}, with block output channels of (128, 384, 768), 4 layers per block, and attention head dimension 16. Training images were segmented to isolate faces \citep{kirillov2023segany,lugaresi2019mediapipe}, and the training loss upweighted face pixels by a factor of 8. AdamW was used to update network weights over 500 epochs, with minibatches of size 192 and an initial $10^{-4}$ learning rate decaying exponentially with a half-life of 250 epochs \citep{loshchilov2017decoupled}. To improve model stability, we kept an exponential moving average of the model weights during training (decay rate 0.999) and used these averaged weights to generate images at inference time \citep{polyak1992acceleration}.

\myparagraph{Performance metrics.} 
Performance was evaluated using several metrics. Fr\'{e}chet and kernel inception distances were computed \citep{heusel2017gans,binkowski2018demystifying}, based on an embedding tailored for human faces (ArcFace) \citep{deng2019arcface} and the classical Inception-v3 embedding \citep{szegedy2016rethinking}. Precision and recall were also computed \citep{kynkaanniemi2019improved}, using these same embeddings. All of the aforementioned metrics were evaluated using torch-fidelity with its default hyperpameter values \citep{obukhov2020torchfidelity}.

The metrics were evaluated by comparing synthetically generated images to those from the test set. This test set consists of 39,829 draws from $P$, rather than from the target counterfactual distribution $\mathbb{P}$. To account for this, those data were reweighted in a doubly robust fashion. Concretely, test data were used to estimate the propensity score and outcome model with the same lightgbm and $k$-nearest neighbors nuisance estimation strategies described earlier. This yielded doubly robust weights, such that the weighted empirical distribution of the test data approximates $\mathbb{P}$. Observations were then resampled with replacement according to these weights, yielding 10,000 test images that should be approximately distributed according to $\mathbb{P}$.


\subsubsection{Results}

\Cref{tab:facesMetricsInception} shows results for this diffusion model experiment, but using the Inception-v3 embedding rather than ArcFace embeddings. \Cref{fig:moreFaces} displays a random sample of 200 faces generated by the na\"{i}ve model and \nameref*{alg:double} when both nuisances were well specified.

\begin{table}[tb]\centering
\caption{Diffusion model performance as in \Cref{tab:facesMetricsArcFace}, but using Inception-v3 rather than ArcFace embeddings.}\label{tab:facesMetricsInception}
\begin{tabular}{ll@{\hskip 2.5em}llll}
     & & FID $\downarrow$ & KID $\downarrow$ & Prec. $\uparrow$ & Rec. $\uparrow$ \\\midrule\midrule
&  Na\"{i}ve                                &                 1.00 &                 1.00 &                 0.66 &                 0.49  \\\midrule
\multirow{3}{4em}{\textit{Both right}} &  Plug-in                                  & 0.90 & 0.78&                 0.64 &                 0.49  \\
&  IPW                                      & 0.88 & 0.76 &                 0.64 &                 0.49  \\
&  \nameref*{alg:double}                    & \textbf{\textcolor{CB5blue}{0.88}} & \textcolor{CB5blue}{0.79} &                 \textbf{0.65} & \textbf{\textcolor{CB5blue}{0.50}}   \\\midrule
\multirow{2}{4em}{\textit{Outcome wrong}} &  Plug-in                                  &                 1.76 &                 1.83 & 0.67 &                 0.25  \\
&  \nameref*{alg:double}                    & \textbf{\textcolor{CB5blue}{0.89}} & \textbf{\textcolor{CB5blue}{0.81}} &                 0.63 &                 \textbf{0.49}  \\\midrule
\multirow{2}{4em}{\textit{Propensity wrong}} &  IPW                                      & 0.92 & 0.76 &                 0.64 &                 0.47  \\
&  \nameref*{alg:double}                    & \textcolor{CB5blue}{0.99} & \textcolor{CB5blue}{0.92} &                 \textbf{0.64} &                 0.45  \\\midrule
\multirow{1}{5.5em}{\textit{Both wrong}} &  \nameref*{alg:double}                    &                 1.28 &                 1.27 &                 0.66 &                 0.40  \\\midrule
\end{tabular}
\vspace{-.75em}

\begin{flushleft}\footnotesize \textit{
$\downarrow$\,/\,$\uparrow$ denote whether smaller/larger values are preferred. \\
Fr\'{e}chet/Kernel inception distances (FID/KID) rescaled so  Na\"{i}ve takes value 1.00.}\\
Values where \nameref*{alg:double} is \textcolor{CB5blue}{better than Na\"{i}ve} are marked in blue, and---within a given nuisance misspecification category---those where it \textbf{performs at least as well as the best alternative method} are marked in bold.
\end{flushleft}
\end{table}

\begin{figure}[h!]
    \centering
    \includegraphics[width=\textwidth]{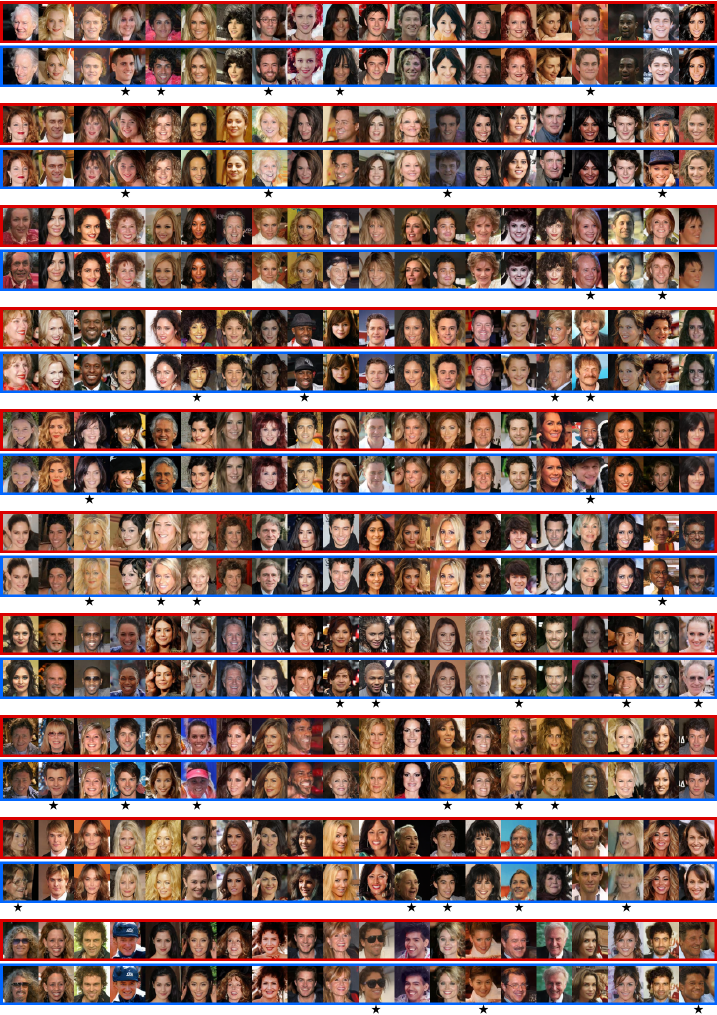}
    \caption{Batch of 200 random samples from a \textbf{\color{CB5red}traditional diffusion model} (left columns) and \textbf{\color{CB5blue}\nameref*{alg:double} diffusion model} (right columns), under the same setup as in \Cref{fig:faces}. The \raisebox{-0.1em}{\includegraphics[height=.85em]{faces/Five_Pointed_Star_Solid.pdf}}'s mark the 20\% of pairs whose ArcFace embeddings \citep{deng2019arcface} have the largest cosine dissimilarity.} \label{fig:moreFaces}
\end{figure}


\subsection{Generating counterfactual product reviews}\label{app:amazonReviews}

\subsubsection{Experimental setup}\label{app:amazonReviewsExperimentalSetup}

\myparagraph{Preprocessing.} Approximately 63.8 million instances with unknown product categories were removed from the dataset. Reviews were preprocessed by first prepending the rating (e.g., `5 stars: ' or `1 star: '), then tokenizing with the Llama 3 tokenizer, and finally truncating to at most 192 tokens.

\myparagraph{Features.} A total of 38 total features were considered. Of these, 33 were one-hot encoded product categories extracted from the dataset's metadata (books, electronics, etc.) The other 5 were the number of product images, number of product videos, and the length (in characters) of the product title, description, and details.

\myparagraph{Creating the semi-synthetic dataset.} To make the synthetic intervention depend on the features in a realistic way, we used 10\% of the data to estimate a propensity model $\pi(a\mymid x)$ based on an actual variable in the dataset, verified purchase status (see next paragraph for details). The remaining 90\% was restricted to verified reviews only. From this subset, 30\% was held out for evaluation, while 60\% was modified for training. We drew a synthetic intervention indicator for each instance conditionally on $X$ using $\pi$, and replaced review text for instances not following the intervention with that of randomly drawn unverified reviews. Letting $Q$ denote the distribution of the original observations, $P$ the synthetic training data, and $A=a^\star$ denote following the synthetic intervention, a draw from $P$ can be sampled by drawing $X\sim Q_{X\mymid A=a^\star}$, $A\mymid X$ from $\pi$, $Y$ given $A=a^\star$ and $X$ from $Q_{Y\mymid A=a^\star,X}$, and $Y$ given $A\not=a^\star$ and $X$ from $Q_{Y\mymid A\not=a^\star}$. Crucially, the G-computed counterfactual distribution under $P$, $\mathbb{P}$, equals $Q_{Y\mymid A=a^\star}$, which lets us evaluate the ground truth using our test set.

\myparagraph{Synthetic propensity $\pi$.} The propensity $\pi$ for the synthetic intervention $a^\star$ was designed to behave similarly to the propensity to have a purchase be verified, but with two key modifications. First, to make training computationally feasible on an academic budget, we reduced the proportion of instances receiving the synthetic intervention from approximately 90\% to 3.5\%. This reduction also demonstrates our method's performance in settings where receiving the target intervention is rare. Second, to ensure sufficient confounding for distinguishing between causal and non-causal methods in this illustrative experiment, we amplified how strongly the propensity depended on baseline features. Specifically, we used 10\% of the data to obtain an estimate $\widehat{\pi}$ via linear-logistic regression, then defined the synthetic propensity on the remaining data as $\expit[-8.5 + 2\logit \widehat{\pi}(x)]$. This transformation maintains the relative ordering of propensities while increasing variance and shifting the overall values they take downward.

Large absolute coefficients in the logistic regression fit $\widehat{\pi}$ include dummies for the following categories: books ($-0.32$), Kindle Store ($-0.31$), movies and TV ($-0.15$), automotive ($0.15$), CDs and vinyl ($-0.14$), and clothing, shoes, and jewelry ($0.13$). Video count also had a coefficient of $0.07$ and title length had a coefficient of $0.03$.

\myparagraph{Nuisance estimation.} Nuisances were estimated similarly to as in the smiling faces experiment, so here we only highlight the differences. For the propensity, we used the SynapseML implementation of lightgbm, since this can be run on a dataset with hundreds of millions of observations \citep{hamilton2018mmlspark}. Because that implementation does not allow custom loss functions, we used a cross-entropy loss instead of a Riesz regression loss. We also used standard default tuning parameter values, rather than selecting them via Optuna, and truncated inverse propensities at 1000. For the outcome model, we stratified the search for the 200 nearest neighbors by product category, using the other 5 features for matching within each category.

As for the smiling faces experiment, sensitivity to nuisance misspecification was assessed by fitting each nuisance model twice. The well-specified nuisances used all baseline features. The poorly specified propensity omitted the product category features. In contrast, the poorly specified outcome model ignored the 5 non-category features. A draw from this model was obtained by equiprobability sampling of a training instance that received the intervention within the category defined by $x$.

\myparagraph{Training.} We used low-rank adaptation (LoRA) to finetune the 1 billion parameter Llama 3.2 base model \citep{hu2022lora,dubey2024llama}. The LoRA layers had rank 8, resulting in about 5.5M trainable parameters. AdamW was used to update the weights, with a learning rate decaying from $2\times 10^{-4}$ according to a cosine decay over 1 epoch. Training was implemented using the Transformers and PEFT libraries, with a PyTorch backend \citep{wolf2019huggingface, peft, NEURIPS2019_9015}. A Claude coding assistant was used to help write, edit, and comment the code.

\myparagraph{Performance metrics.} Performance was evaluated using four metrics. Mean perplexity was computed using 20k randomly selected test samples, as computed in the HuggingFace evaluate library \citep{jelinek1977perplexity,von2022evaluate}. MAUVE and frontier integrals were used to assess divergences between 10k synthetic generated samples and 20k randomly selected test samples, using the default settings from the software in \citep{pillutla2021mauve}. Finally, the 1-Wasserstein distance was computed between the distributions of the star ratings appearing at the beginning of the generated samples (e.g., a review that begins ``3 stars: $\ldots$'' is converted to a numeric rating of 3) and the star ratings of the roughly 150M test samples.

\subsection{Results}

\begin{figure}[tb]
    \includegraphics[width=\textwidth]{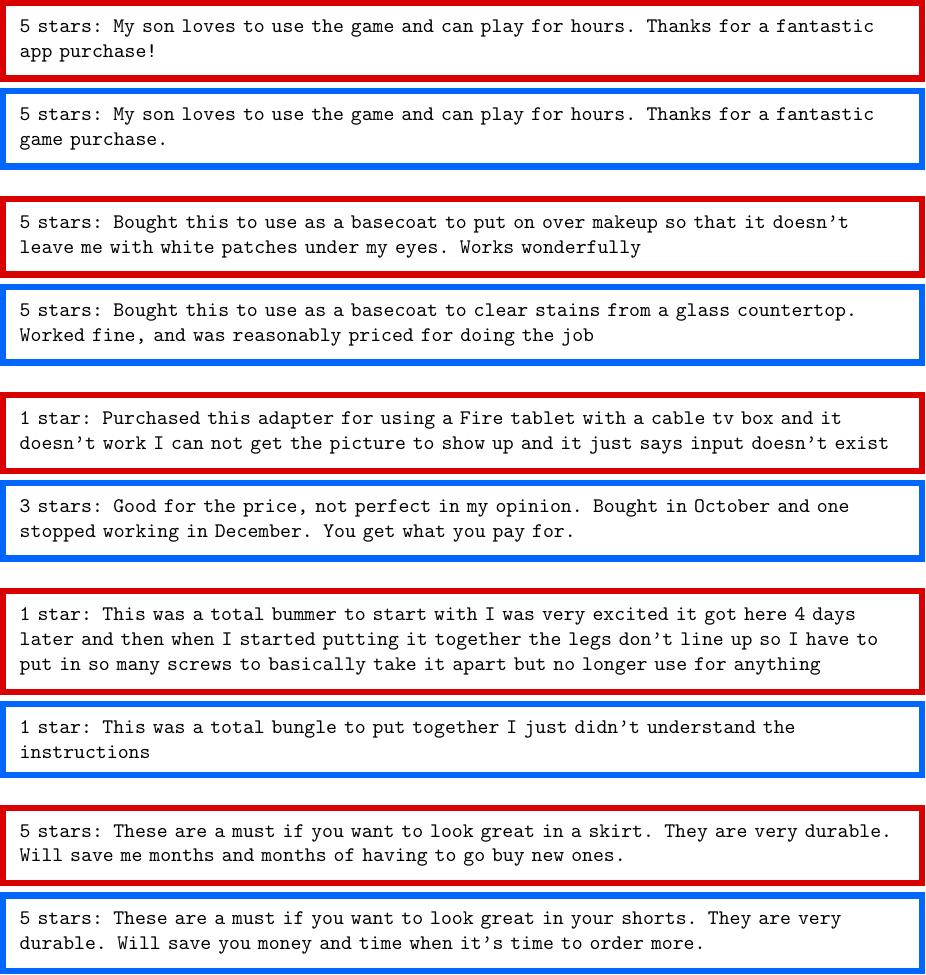}
    \caption{Reviews generated by a \textbf{\color{CB5red}traditional language model} trained on only instances that follow the target intervention (top) and \textbf{\color{CB5blue}\nameref*{alg:double} language model} (bottom). Samples are coupled, with the random seed set to the same value for both models before generation. While the models often produce similar reviews, certain product categories are over- or underrepresented by the traditional language model---see \Cref{fig:reviewsBooks} for an example with book reviews.}\label{fig:reviews}
\end{figure}

\begin{figure}[tb]
    \includegraphics[width=\textwidth]{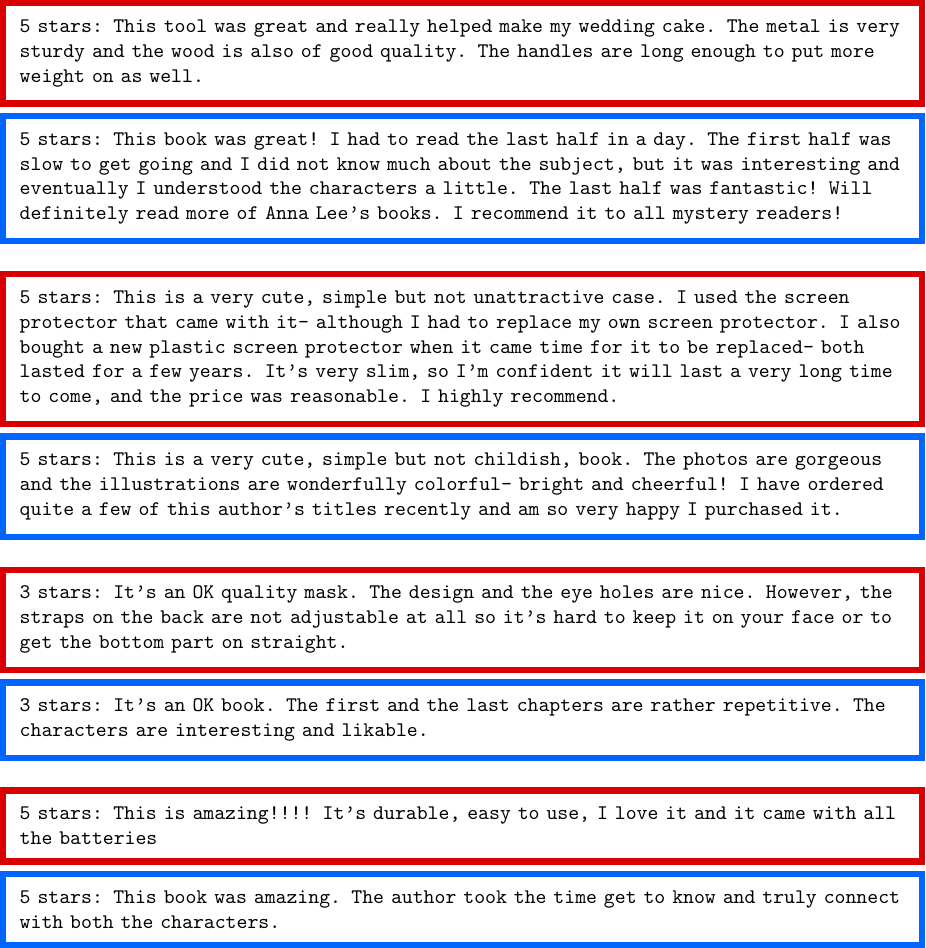}
    \caption{Coupled reviews from the same models as in \Cref{fig:reviews}, filtered to show only pairs where at least one model's review contains the word `book' or `books'. \textbf{\color{CB5blue}\nameref*{alg:double}} generates reviews with this word at roughly the same frequency as the test set (4.4\% vs. 4.2\%),  while the \textbf{\color{CB5red}traditional model} severely underrepresents this content (0.24\%).}\label{fig:reviewsBooks}
\end{figure}

\Cref{fig:reviews} displays reviews generated by \nameref*{alg:double} and the na\"{i}ve approach. Most of these reviews are similar to one another. This suggests there was relatively little confounding in this experiment. However, there are certain types of reviews that are underrepresented by the na\"{i}ve approach. This includes those that contain the word `book' or `books' (\Cref{fig:reviewsBooks}). This is not surprising given that the propensity to receive the synthetic intervention is lower for reviews written for items in the Amazon's Books product category.

\end{document}